\documentclass{article}

\PassOptionsToPackage{numbers,square}{natbib}
\usepackage[preprint]{neurips_2026}

\usepackage[utf8]{inputenc} 
\usepackage[T1]{fontenc}    
\usepackage{hyperref}       
\usepackage{url}            
\usepackage{booktabs}       
\usepackage{amsmath,amssymb,amsfonts,amsthm} 
\usepackage{nicefrac}       
\usepackage{microtype}      
\usepackage{xcolor}         
\usepackage{tikz}           
\usepackage{algorithm}       
\usepackage{algpseudocode}   
\makeatletter
\AtBeginDocument{\def\theHALG@line{\arabic{algorithm}.\arabic{ALG@line}}}
\makeatother
\usetikzlibrary{positioning, calc, fit, arrows.meta, patterns,
                decorations.pathreplacing}

\definecolor{bitsixteen}{HTML}{0C447C}
\definecolor{biteight}{HTML}{185FA5}
\definecolor{bitfour}{HTML}{378ADD}
\definecolor{bittwo}{HTML}{85B7EB}
\definecolor{hatchbg}{HTML}{D3D1C7}
\definecolor{hatchstroke}{HTML}{444441}
\definecolor{stairgray}{HTML}{5F5E5A}
\definecolor{prefilllight}{HTML}{B4B2A9}
\definecolor{prefillpale}{HTML}{D3D1C7}
\definecolor{prefillbg}{HTML}{F1EFE8}

\usepackage{wrapfig}        
\usepackage{graphicx}       
\usepackage{subcaption}     
\usepackage{rotating}       
\usepackage{multirow}       
\usepackage[nameinlink,noabbrev]{cleveref} 

\theoremstyle{plain}
\newtheorem{theorem}{Theorem}[section]
\newtheorem{lemma}[theorem]{Lemma}

\newtheorem{proposition}[theorem]{Proposition}
\theoremstyle{definition}

\theoremstyle{remark}
\newtheorem{remark}[theorem]{Remark}
\crefname{section}{Sec.}{Secs.}
\Crefname{section}{Section}{Sections}
\crefname{equation}{Eq.}{Eqs.}
\Crefname{equation}{Equation}{Equations}
\crefname{theorem}{Thm.}{Thms.}
\Crefname{theorem}{Theorem}{Theorems}
\crefname{lemma}{Lem.}{Lems.}
\Crefname{lemma}{Lemma}{Lemmas}
\crefname{proposition}{Prop.}{Props.}
\Crefname{proposition}{Proposition}{Propositions}
\Crefname{table}{Table}{Tables}
\crefname{table}{Tab.}{Tabs.}
\Crefname{figure}{Figure}{Figures}
\crefname{figure}{Fig.}{Figs.}
\Crefname{algorithm}{Algorithm}{Algorithms}
\crefname{algorithm}{Alg.}{Algs.}

\DeclareMathOperator*{\argmin}{arg\,min}

\title{RDKV: Rate-Distortion Bit Allocation for Joint Eviction and Quantization of the KV Cache}

\author{%
  Junkai Zhang$^1$ \quad Hang Guo$^2$ \quad Luca Benini$^1$ \quad Yawei Li$^1$ \\[0.5em]
  $^1$ ETH Zurich \quad $^2$ Tsinghua University
}

\begin{document}

\maketitle

\begin{abstract}
Large language models (LLMs) have shown strong performance across diverse tasks, but their inference with long input contexts is bottlenecked by memory size and bandwidth. The Key-Value (KV) cache size grows linearly with sequence length and needs to be re-read from off-chip high-bandwidth memory (HBM) to on-chip memory at every decoding step, resulting in memory-bound inference. Existing methods reduce the cache by either eviction or quantization, but typically treat the two in isolation. In this paper, we cast KV cache compression as a rate-distortion problem, under which eviction and quantization are two end-points of the same bit allocation scheme. This exposes the need to optimize them jointly, motivating our method, \textbf{RDKV} (\textbf{R}ate-\textbf{D}istortion \textbf{KV} cache compression). RDKV derives the weight of each token or channel from the distortion that compression induces on the attention computation. Based on these weights, it assigns each token or channel a bit-width ranging from full precision down to zero bits guided by reverse water-filling, applied once after the prefilling stage.  Experiments on LongBench, RULER, and InfiniteBench show that RDKV outperforms the best evaluated baseline by 9.1\% on average. On LongBench it recovers 97.81\% of full-cache accuracy with only 2.48\% cache retention. Compared with full-cache FlashAttention-2 decoding, it achieves \textbf{4.5$\times$} decode speedup and \textbf{1.9$\times$} peak memory reduction with 128K context length, while maintaining comparable performance.
\end{abstract}

\section{Introduction}
\label{sec:introduction}

Large language models (LLMs) \citep{brown2020language,achiam2023gpt,liu2024deepseek} have demonstrated strong performance from open-ended generation to complex multi-step reasoning. As these capabilities move into production, the input contexts they require have grown from thousands to hundreds of thousands of tokens. Retrieval-augmented generation~\citep{lewis2020retrieval}, multi-document reasoning, and agentic workflows all routinely operate at this scale. At these lengths, inference becomes memory-bound. Transformer decoding stores a Key-Value (KV) pair for every token in every layer. The cache therefore grows linearly with sequence length and reaches tens of gigabytes per request~\citep{kwon2023efficient}. Because every generated token re-reads this entire cache through the attention computation~\citep{dao2022flashattention}, memory traffic rather than arithmetic sets the per-token decode latency~\citep{pope2023efficiently}. Reducing the cache footprint without retraining, while preserving generation quality, has thus become a central problem for efficient long-context inference \citep{li2024survey,liu2025kv}.

A growing number of works address this bottleneck by compressing the KV cache after prefill. These methods fall along two axes: eviction and quantization. Eviction methods score each token or channel and keep the top-ranked units at full precision, dropping the rest~\citep{h2o,li2024snapkv,xu2024think}. Quantization methods retain every token and instead reduce the bit-width~\citep{kivi,he2024zipcache}. Yet cache units differ widely in importance: a few are critical, most are not. A binary keep-or-evict action leaves no precision tier for the broad middle, while mixed-precision quantization retains every token, including the many that contribute negligibly to the output. Recent analysis shows that combining eviction with quantization outperforms either action alone~\citep{qpruningkv}. Several hybrids route each token among full-precision, low-bit, and evicted states~\citep{arkv,anonymous2026hqekv}. But their routing scores remain heuristic and the budget ratios are preset or searched before the final assignment. A natural question is whether both the score and the allocation can be derived from a single objective function.

\begin{figure}[!tb]
  \centering
  \includegraphics[width=\textwidth]{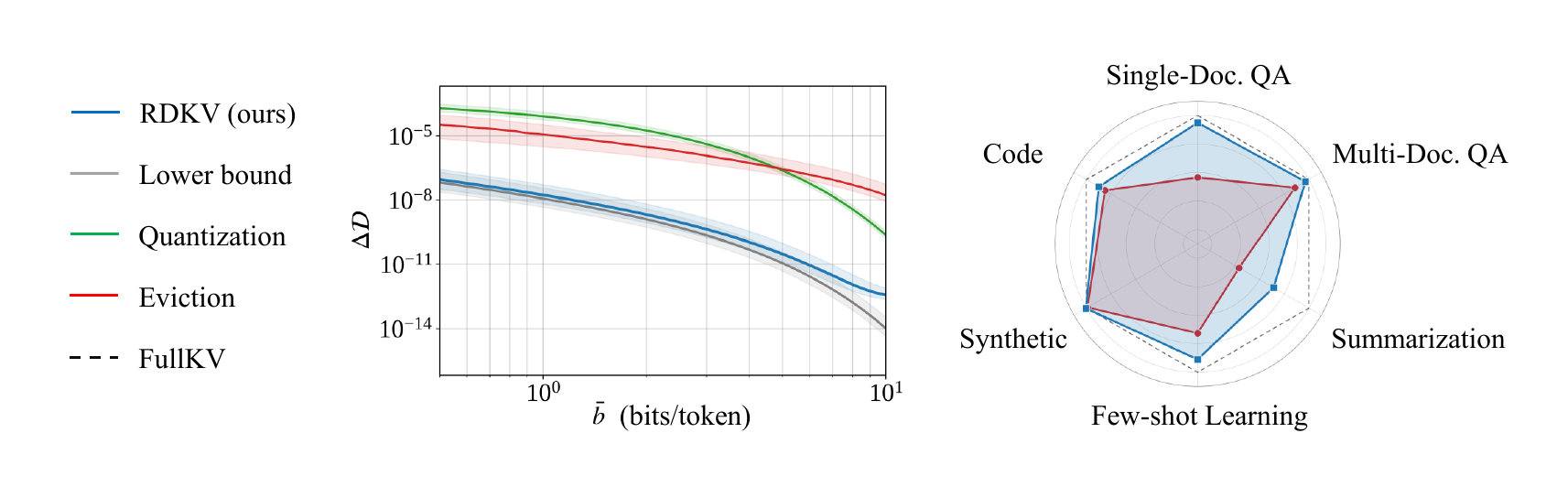}
  \caption{\textbf{Left:} Per-sequence weighted distortion 
$\Delta\mathcal{D}=\sum_u w_u\,\varepsilon_u(b_u)$ (~\cref{eq:mckp}) 
versus average bit-width $\bar{b}$. 
Lines: median across sequences; shaded: IQR. 
\emph{Lower bound}: continuous relaxation (Prop.~\ref{prop:weak-duality}).
\textbf{Right:} LongBench score by task category at a per-layer cache budget of 128 FP16-equivalent tokens 
($B_\text{total}=128L$), normalized by FullKV; per-task scores in~\cref{tab:longbench-llama31}. 
\emph{Eviction} refers to Ada-SnapKV~\citep{feng2024adakv} in the right panel. 
Model: LLaMA-3.1-8B-Instruct~\citep{llama3}.}
  \label{fig:intro}
\end{figure}

In this paper, we present RDKV, a rate–distortion framework for KV cache compression. RDKV poses a bit-allocation problem: given a fixed bit budget, assign each token in the V cache and each channel in the K cache a bit-width to minimize a weighted distortion. The weight of each token or channel is defined as the deviation in the attention distribution or the attention logit when it is evicted. Reverse water-filling converts these weights into bit-widths, from full precision for critical units down to zero bits (eviction) for negligible ones. Eviction and quantization are thus at the two ends of one allocation curve, and are explored jointly rather than in a staged fashion.

We make three \textbf{contributions}. (1) We formulate KV cache compression as a rate–distortion problem. The continuous optimum inherently mixes zero-rate (eviction) and finite-rate (quantization) units. Therefore, approximating this bound requires both actions in the same solver, while restricting to either alone leaves a clear gap, as \cref{fig:intro} (left) shows on calibration data. (2) We instantiate the allocation as a discrete knapsack over hardware-supported bit-widths and solve it via reverse water-filling with Lagrangian relaxation. (3) We realize the resulting mixed-bit cache with TriZone, a packed-decode layout that fuses dequantization into the attention kernel, turning the mixed-bit allocation into actual memory savings. Taken together, \cref{fig:intro} (right) shows the empirical picture: RDKV approaches FullKV's performance across all six categories while the strong baseline lags consistently.

\section{Related Work}
\label{sec:related}

Existing KV cache compression methods can be organized by the action
they apply: eviction, quantization, or both. Meanwhile, they can also be organized by how the compression
target is specified. We adopt the latter axis, distinguishing methods that target
a prescribed compression ratio or tier schedule from methods that allocate a
given active-cache budget.

\textbf{Compression-Ratio-Driven Methods.}
These methods reduce the stored representation of cache entries to reach a
target compression ratio. Quantization methods such as KIVI~\citep{kivi},
KVQuant~\citep{kvquant}, ZipCache~\citep{he2024zipcache}, and
KVTuner~\citep{kvtuner} exploit asymmetric K/V quantization, outlier handling,
salient-token preservation, or searched precision configurations to lower the
bit-width of stored keys and values. QAQ~\citep{qaq} and MiKV~\citep{mikv}
adapt precision by token or KV-pair importance, while AQUA-KV~\citep{aquakv}
exploits K/V dependencies and quantizes prediction residuals. Several recent hybrids add eviction as the lowest precision tier. 
HqeKV~\citep{anonymous2026hqekv} uses a Tree-structured parzen estimator search
(200 trials per setting) to determine the fraction of tokens at each bit-width;
ARKV~\citep{arkv} estimates per-layer Original--Quantization ratios from
prefill-time attention statistics, with scoring thresholds searched offline.
Both methods rely on offline optimization before deployment, with HqeKV searching
tier ratios and ARKV searching scoring thresholds.

\textbf{Budget-Driven Methods.}
These methods take an active-cache budget as input and decide which units remain
available during decoding. KV cache eviction methods score each cache unit and
retain the top-ranked subset under a fixed token count.
StreamingLLM~\citep{streamingllm} keeps sink tokens and a recent window.
H2O~\citep{h2o}, SnapKV~\citep{li2024snapkv}, Scissorhands~\citep{scissorhands}, and
NaCl~\citep{nacl} use attention-derived token scores.
AdaKV~\citep{feng2024adakv}, PyramidKV~\citep{cai2024pyramidkv}, and
CAKE~\citep{cake} further distribute the budget across heads or layers.
Expected Attention~\citep{kvpress} estimates future-query scores rather than
relying only on the observed window. Recent work sharpens the scoring signal
beyond attention weights: CriticalKV~\citep{criticalkv} optimizes worst-case
attention-output perturbation incorporating value states;
OBCache~\citep{obcache} applies the Optimal Brain
framework~\citep{lecun1989optimal,guo2025optimal} to estimate each pair's impact on the attention
output; CAOTE~\citep{goel2025caote} scores via the output error induced by
removal; AnDPro~\citep{andpro} projects value states onto anchor directions;
ReST-KV~\citep{restkv} stabilizes scores through layer-wise reconstruction and
spatial-temporal smoothing. ThinK~\citep{xu2024think} extends eviction
from tokens to key channels.

However, existing budget-driven methods restrict the
action to a binary keep-or-drop decision. This restriction leaves no intermediate precision tier
for tokens or channels between the two extremes. The framework in~\cref{sec:methodology}
generalizes this budget-driven paradigm: given a total bit budget, it derives
per-token or per-channel distortion weights from the attention computation and allocates
bit-widths via reverse water-filling. Unlike the ratio-driven hybrids above,  eviction and quantization are decided within one allocation problem without offline ratio or threshold optimization.


\section{Methodology}
\label{sec:methodology}
Consider a decoder-only transformer~\citep{vaswani2017attention,pope2023efficiently} with
$L$ layers, $H_q$ query heads, $H_\text{kv}$ KV heads (group size
$g = H_q/H_\text{kv}$)~\citep{ainslie2023gqa}, and per-head dimension~$d$.
For a prefilled context of $T$ tokens the KV cache per layer comprises $\mathbf{K}, \mathbf{V} \in \mathbb{R}^{H_{kv} \times T \times d}$, with attention weights $a_{\tau,t} = \operatorname{softmax}_t(q_\tau^\top k_t/\sqrt{d})$ and output $o_\tau = \sum_t a_{\tau,t}\,v_t$. 
We assign each cache unit~$u$ (a token or a channel) a bit-width $b_u \in \mathcal{B} = \{0,2,4,8,16\}$, from FP16 retention to outright removal, under a total bit budget $B$.
Our goal is to find the allocation $\{b_u^\star\}$ that minimizes the distortion introduced to the attention computation under this budget constraint.

In the following, we first quantify the compression cost of each token or channel, and formulate the bit allocation as a rate-distortion problem (\cref{sec:rd_formulation}). We then describe the bit allocation pipeline (\cref{sec:vk_pipeline}) and the packed-decode layout that realizes the mixed-bit cache in hardware (\cref{sec:efficient_decode}).
Due to space constraints, formal proofs in this section are deferred to \cref{app:proofs}.


\subsection{Rate-Distortion Formulation}
\label{sec:rd_formulation}

To quantify the compression cost of each unit, we analyze the effect of removing a single token or channel from the uncompressed cache.
In the V cache, each token enters the output as a single weighted term $a_{\tau,t}\,v_t$.
Removing it perturbs the attention distribution~$a_\tau$.
In the K cache, persistent channel-wise outliers~\citep{kivi} make the channel the natural compression unit.
Zeroing a channel perturbs the logit matrix~$Z$.
Because both perturbations are linear, they admit closed-form weights.

\textbf{Distortion weight in V cache.}
Consider evicting token~$t$ from the V cache of a single head.
Because $a_\tau$ is produced by softmax, the remaining tokens
automatically absorb the freed probability mass.
The post-eviction distribution is
\begin{equation}\label{eq:v-renorm}
\hat{a}_{\tau,t} = 0, \qquad
\hat{a}_{\tau,t'} = \frac{a_{\tau,t'}}{1 - a_{\tau,t}}
\quad \text{for } t' \neq t.
\end{equation}
To measure how much this changes the attention pattern, we use the
total variation (TV) distance between $a_\tau$
and~$\hat{a}_\tau$.
It yields the per-token distortion weight:

\begin{proposition}[Token weight in V cache]
\label[proposition]{thm:v-score}
Evicting token~$t$ yields
$\|a_\tau - \hat{a}_\tau\|_{\mathrm{TV}} = a_{\tau,t}$
for each query~$\tau$.
The per-token distortion weight, obtained by summing over all
queries, is
\begin{equation}\label{eq:v-weight}
w_t \;:=\; \sum_\tau a_{\tau,t}.
\end{equation}
\end{proposition}

\noindent
The weight $w_t$ is the total attention that token~$t$ receives
across all query positions, recovering the cumulative attention
widely used as a token-level eviction
criterion~\citep{h2o,li2024snapkv,feng2024adakv}.
In prior work this quantity serves as a ranking score for binary
keep-or-evict decisions. Here it enters the objective
as a multiplicative coefficient. Its magnitude therefore affects the allocated bit-width, not merely its rank.

\textbf{Distortion weight in K cache.}
An analogous construction applies to the K cache, but with a
different partition.
Channel~$c$ contains the $T$ scalars
$\{k_{t,c}\}_{t=1}^T$ across all token positions.
Consider removing channel~$c$, i.e., setting $K[:,c]$ to zero.
Since each logit entry is the inner product
$q_\tau^\top k_t / \sqrt{d}$, zeroing one channel removes the
contribution of that coordinate from every entry simultaneously.
The resulting deviation in the logit matrix is
\begin{equation}\label{eq:k-deviation}
\delta Z \;=\; -\frac{1}{\sqrt{d}}\,Q[:,c]\,K[:,c]^\top,
\end{equation}
and it yields the per-channel distortion weight:

\begin{proposition}[Channel weight in K cache]
\label[proposition]{thm:k-score}
Removing channel~$c$ from head~$h$ produces a rank-one logit
deviation; its spectral norm gives the per-channel distortion weight
\begin{equation}\label{eq:k-weight}
w_c \;:=\; \frac{1}{\sqrt{d}}\,\|Q[:,c]\|_2\,\|K[:,c]\|_2.
\end{equation}
\end{proposition}

\noindent
The weight $w_c$ is the product of the $c$-th column norms of the
query and key matrices: channels whose Q and K columns are both large
contribute strongly to the logit and are expensive to remove.
This weight coincides with ThinK~\citep{xu2024think},
which arrives at the same expression via Frobenius minimization.
As with $w_t$, the magnitude of $w_c$ shapes its allocated
bit-width, not its rank.

Notably, $w_t$ (\cref{eq:v-weight}) and $w_c$ (\cref{eq:k-weight})
are both derived from eviction, i.e., the complete removal of a token
or channel.
However, they extend naturally to quantization because the attention
computation is linear in both cache types.
Specifically, quantizing a token in the V cache produces the output
perturbation $\delta o_\tau = a_{\tau,t}\,\delta v_t$, so the cost
scales with~$w_t$.
Similarly, quantizing a channel in the K cache produces a logit
perturbation that scales with~$w_c$.
The weights $w_t$ and $w_c$ therefore serve as the cost coefficients in the objective function of bit allocation.

\textbf{Rate-Distortion Aware Bit Allocation.}
Given these coefficients, the remaining question is how to choose bit-widths~$b_u$ under a
total budget, where $u$ indexes tokens ($n_u = d$ scalars) or
channels ($n_u = T$ scalars).
Under Bennett's high-rate approximation~\citep{bennett1948spectra},
a uniform scalar quantizer with dynamic range~$R_u$ achieves
per-coordinate root-mean-square error $\sigma_u\,2^{-b}$ with
$\sigma_u := R_u/(2\sqrt{3})$.
Combining the per-unit weights with this per-coordinate error gives
the rate-distortion objective
\begin{equation}\label{eq:rd-objective}
\Delta\mathcal{D}(\{b_u\})
  \;=\;
  \sum_{u} w_u\,\sigma_u\,2^{-b_u},
  \qquad
  \text{s.t.}\;\;\sum_u b_u \le B.
\end{equation}
The coefficient $w_u\,\sigma_u$ combines the distortion
weight~($w_u$) with quantization hardness~($\sigma_u$).
The Lagrangian solution is a reverse water-filling familiar from
channel coding~\citep{shannon1959coding,cover2006elements}:

\begin{theorem}[Optimal bit allocation]
\label{thm:water-filling}
The minimizer of \cref{eq:rd-objective} subject to $b_u \ge 0$ is
\[
b_u^\star
  \;=\;
  \left[\log_2
    \frac{\ln 2\cdot w_u\,\sigma_u}{\lambda}
  \right]_{\!+}\!,
\]
with $\lambda > 0$ chosen so that the budget binds.
The water level $\lambda/\ln 2$ induces a phase transition:
units with $w_u\,\sigma_u < \lambda/\ln 2$ receive
$b_u^\star = 0$, while units above receive finite bit-widths.
\end{theorem}

\noindent
\Cref{thm:water-filling} reveals that eviction and quantization are
not separate actions but two regimes of one allocation curve,
selected jointly by the budget~$B$ rather than composed under
pre-allocated ratios~\citep{arkv}.
Tightening the budget~$B$ raises~$\lambda$ and pushes more units past
the zero-rate boundary into eviction.
Loosening it promotes previously evicted units to low-bit retention.
The allocator thus navigates eviction and quantization on a single
curve, with the budget as the only control knob.

Hardware restricts bit-widths to the discrete set
$\mathcal{B} = \{0,2,4,8,16\}$.
We replace Bennett's curve with an empirical per-coordinate distortion
$\varepsilon_u(b)$ calibrated at each
$b \in \mathcal{B}$ (see \cref{app:epsilon-calibration}), including $b=0$ for removal, and solve the
resulting multiple-choice knapsack problem (MCKP):
\begin{equation}\label{eq:mckp}
\{b_u^\star\}
  \;=\;
  \argmin_{b_u \in \mathcal{B}}\;
  \sum_u w_u\,\varepsilon_u(b_u)
  \qquad
  \text{s.t.}\;\;\sum_u b_u \le B.
\end{equation}
Lagrangian relaxation decouples this into independent per-coordinate table lookups.
A one-dimensional bisection over~$\lambda$ recovers a feasible primal
allocation in $O(U|\mathcal{B}|)$ time, adding negligible latency
compared to prefill (\cref{app:ablation-overhead}).
Weak-duality bounds certify near-optimality (\cref{app:proofs}).


\subsection{Per-Head Allocation Pipeline}
\label{sec:vk_pipeline}

\begin{figure}[!tb]
\centering
\resizebox{\textwidth}{!}{%
\begin{tikzpicture}[
  font=\footnotesize,
  >={Latex[length=1.3mm, width=0.9mm]},
]

\def\barw{0.28}
\def\barstep{0.36}
\def\maxh{1.54}
\pgfmathsetmacro{\chartw}{8*\barstep-0.08}  

\def\vbx{3.10}   \def\vby{3.10}
\def\kbx{3.10}   \def\kby{0.50}

\pgfmathsetmacro{\arrS}{\vbx+\chartw+0.30}
\pgfmathsetmacro{\arrE}{\arrS+0.55}
\pgfmathsetmacro{\cachex}{\arrE+0.30}        

\def\vccw{0.28}  \def\vcch{0.275}
\def\kccw{0.21}  \def\kcch{0.275}

\def\vtA{1.317}  \def\vtB{0.852}  \def\vtC{0.493}  \def\vtD{0.282}
\def\ktA{1.240}  \def\ktB{0.852}  \def\ktC{0.542}  \def\ktD{0.310}

\pgfmathsetmacro{\vzl}{\vbx+\chartw+0.06}
\pgfmathsetmacro{\kzl}{\kbx+\chartw+0.06}
\pgfmathsetmacro{\kcy}{\kby+(\maxh-6*\kcch)/2}
\pgfmathsetmacro{\vbrkx}{\cachex+6*\vccw+0.08}
\pgfmathsetmacro{\vbrktxt}{\vbrkx+0.50}
\pgfmathsetmacro{\connx}{\cachex+3*\vccw}

\def\qkw{0.70}  \def\qkh{0.85}  \def\qkgap{0.20}
\pgfmathsetmacro{\qleft}{1.20 - \qkw - \qkgap/2}
\pgfmathsetmacro{\qright}{\qleft + \qkw}
\pgfmathsetmacro{\kleft}{1.20 + \qkgap/2}
\pgfmathsetmacro{\kright}{\kleft + \qkw}
\def\qkbot{1.40}
\pgfmathsetmacro{\qktop}{\qkbot + \qkh}
\pgfmathsetmacro{\qcstep}{\qkw/5}

\begin{scope}[shift={(0.24, 0.56)}, scale=0.8, transform shape]

\fill[prefillbg!40, rounded corners=2.5pt]
     (-0.10, 0.20) rectangle (2.50, 5.40);
\draw[stairgray!25, rounded corners=2.5pt, line width=0.3pt]
     (-0.10, 0.20) rectangle (2.50, 5.40);

\node[font=\scriptsize\bfseries, text=stairgray]
     at (1.20, 5.18) {\strut Stage 1: Weighting};

\node[font=\tiny\itshape, text=stairgray] at (1.30, 4.64) {$a_{\tau,t}$};
\def\cs{0.23}
\def\hx{0.30}
\def\hy{3.60}
\foreach \r/\vals in {%
  3/{90,36,10,60,24,5,42,68},%
  2/{78,30, 8,55,20,6,35,62},%
  1/{65,28, 9,68,18,4,30,76},%
  0/{52,22, 5,47,22,5,28,48}%
}{%
  \foreach \val [count=\ci from 0] in \vals {%
    \pgfmathsetmacro{\op}{\val/100}%
    \fill[bitsixteen, opacity=\op]
      ({\hx+\ci*\cs}, {\hy+\r*\cs}) rectangle +(\cs,\cs);
  }%
}
\draw[stairgray, line width=0.12pt]
     (\hx,\hy) grid[step=\cs] ({\hx+8*\cs},{\hy+4*\cs});
\draw[stairgray, line width=0.30pt]
     (\hx,\hy) rectangle ({\hx+8*\cs},{\hy+4*\cs});
\node[font=\tiny, text=stairgray, rotate=90, anchor=south]
     at ({\hx-0.04},{\hy+2*\cs}) {query};

\node[font=\tiny, text=stairgray] at (1.20, 3.28)
     {$w_t \!=\! \sum_\tau a_{\tau,t}$};

\draw[stairgray!18, line width=0.2pt] (0.10, 2.85) -- (2.40, 2.85);

\fill[biteight!18] (\qleft,\qkbot) rectangle (\qright,\qktop);
\draw[stairgray, line width=0.25pt] (\qleft,\qkbot) rectangle (\qright,\qktop);
\foreach \j in {1,...,4}{
  \pgfmathsetmacro{\cx}{\qleft+\j*\qcstep}
  \draw[white, line width=0.15pt] (\cx,\qkbot) -- (\cx,\qktop);
}
\node[font=\scriptsize\bfseries, text=biteight!70]
     at ({(\qleft+\qright)/2},{(\qkbot+\qktop)/2}) {$Q$};
\draw[stairgray!50, decorate, decoration={brace, amplitude=1.5pt, mirror}]
  (\qleft,{\qkbot-0.04}) -- (\qright,{\qkbot-0.04})
  node[midway, below, font=\tiny, text=stairgray, yshift=-1.5pt] {$c$};

\fill[bitfour!18] (\kleft,\qkbot) rectangle (\kright,\qktop);
\draw[stairgray, line width=0.25pt] (\kleft,\qkbot) rectangle (\kright,\qktop);
\foreach \j in {1,...,4}{
  \pgfmathsetmacro{\cx}{\kleft+\j*\qcstep}
  \draw[white, line width=0.15pt] (\cx,\qkbot) -- (\cx,\qktop);
}
\node[font=\scriptsize\bfseries, text=bitfour!70]
     at ({(\kleft+\kright)/2},{(\qkbot+\qktop)/2}) {$K$};
\draw[stairgray!50, decorate, decoration={brace, amplitude=1.5pt, mirror}]
  (\kleft,{\qkbot-0.04}) -- (\kright,{\qkbot-0.04})
  node[midway, below, font=\tiny, text=stairgray, yshift=-1.5pt] {$c$};

\node[font=\tiny, text=stairgray, text width=2.4cm, align=center]
     at (1.20, 0.60)
     {$w_c \!=\! \dfrac{\|Q_{:,c}\|\;\|K_{:,c}\|}{\sqrt{d}}$};

\end{scope}

\draw[->, stairgray, line width=0.40pt]
  (2.55, 3.8) -- node[above,font=\tiny,inner sep=1pt]{$w_t$} (3.10, 3.8);
\draw[->, stairgray, line width=0.40pt]
  (2.55, 1.6) -- node[above,font=\tiny,inner sep=1pt]{$w_c$} (3.10, 1.6);

\begin{scope}[shift={(0.91, 0.80)}, scale=0.8, transform shape]

\node[font=\scriptsize\bfseries, text=stairgray, anchor=west]
     at (3.10, 5.05) {\strut Stage 2: Token Allocation};

\fill[bitsixteen!10] (\vbx,{\vby+\vtA}) rectangle ({\vbx+\chartw},{\vby+\maxh+0.06});
\fill[biteight!10]   (\vbx,{\vby+\vtB}) rectangle ({\vbx+\chartw},{\vby+\vtA});
\fill[bitfour!10]    (\vbx,{\vby+\vtC}) rectangle ({\vbx+\chartw},{\vby+\vtB});
\fill[bittwo!10]     (\vbx,{\vby+\vtD}) rectangle ({\vbx+\chartw},{\vby+\vtC});
\fill[hatchbg!12]    (\vbx,\vby)         rectangle ({\vbx+\chartw},{\vby+\vtD});

\foreach \tl in {\vtA,\vtB,\vtC,\vtD}
  \draw[stairgray!45, line width=0.25pt, dashed]
    (\vbx,{\vby+\tl}) -- ({\vbx+\chartw},{\vby+\tl});

\node[font=\tiny,text=bitsixteen,anchor=west] at (\vzl,{\vby+(\vtA+\maxh+0.06)/2}) {16};
\node[font=\tiny,text=biteight,anchor=west]   at (\vzl,{\vby+(\vtB+\vtA)/2})       {8};
\node[font=\tiny,text=bitfour,anchor=west]    at (\vzl,{\vby+(\vtC+\vtB)/2})       {4};
\node[font=\tiny,text=bittwo,anchor=west]     at (\vzl,{\vby+(\vtD+\vtC)/2})       {2};
\node[font=\tiny,text=hatchstroke,anchor=west] at (\vzl,{\vby+\vtD/2})              {0};

\foreach \i/\h/\clr in {%
  0/1.426/bitsixteen, 1/0.589/bitfour, 3/1.162/biteight,%
  4/0.434/bittwo,     6/0.697/bitfour, 7/1.271/biteight%
}{%
  \fill[\clr]  ({\vbx+\i*\barstep},\vby) rectangle ({\vbx+\i*\barstep+\barw},{\vby+\h});
  \draw[stairgray!50, line width=0.12pt] ({\vbx+\i*\barstep},\vby) rectangle ({\vbx+\i*\barstep+\barw},{\vby+\h});
}
\foreach \i/\h in {2/0.186, 5/0.124}{%
  \fill[hatchbg] ({\vbx+\i*\barstep},\vby) rectangle ({\vbx+\i*\barstep+\barw},{\vby+\h});
  \fill[pattern=north east lines, pattern color=hatchstroke, opacity=0.50]
    ({\vbx+\i*\barstep},\vby) rectangle ({\vbx+\i*\barstep+\barw},{\vby+\h});
  \draw[stairgray!50, line width=0.12pt] ({\vbx+\i*\barstep},\vby) rectangle ({\vbx+\i*\barstep+\barw},{\vby+\h});
}
\foreach \i in {0,...,7}{%
  \node[font=\tiny,text=stairgray] at ({\vbx+\i*\barstep+\barw/2},{\vby-0.12})
       {\pgfmathparse{int(\i+1)}\pgfmathresult};
}
\draw[stairgray, line width=0.30pt] (\vbx,\vby) -- ({\vbx+\chartw},\vby);

\end{scope}

\draw[->, stairgray, line width=0.40pt]
  (6.10, 3.8) -- node[above,font=\tiny\itshape,inner sep=1pt]{MCKP} (6.95, 3.8);

\begin{scope}[shift={(1.63, 0.45)}, scale=0.8, transform shape]

\foreach \row/\clr in {%
  7/bitsixteen, 6/bitfour, 4/biteight,%
  3/bittwo,     1/bitfour, 0/biteight%
}{%
  \fill[\clr] (\cachex,{\vby+\row*\vcch}) rectangle ({\cachex+6*\vccw},{\vby+(\row+1)*\vcch});
}
\foreach \row in {5,2}{%
  \fill[hatchbg] (\cachex,{\vby+\row*\vcch}) rectangle ({\cachex+6*\vccw},{\vby+(\row+1)*\vcch});
  \fill[pattern=north east lines, pattern color=hatchstroke, opacity=0.40]
    (\cachex,{\vby+\row*\vcch}) rectangle ({\cachex+6*\vccw},{\vby+(\row+1)*\vcch});
}
\foreach \c in {1,...,5}
  \draw[white, line width=0.12pt] ({\cachex+\c*\vccw},\vby) -- ({\cachex+\c*\vccw},{\vby+8*\vcch});
\foreach \r in {1,...,7}
  \draw[white, line width=0.12pt] (\cachex,{\vby+\r*\vcch}) -- ({\cachex+6*\vccw},{\vby+\r*\vcch});
\draw[stairgray, line width=0.30pt] (\cachex,\vby) rectangle ({\cachex+6*\vccw},{\vby+8*\vcch});

\foreach \i in {1,...,8}
  \node[font=\tiny,text=stairgray,anchor=east] at ({\cachex-0.06},{\vby+(8.5-\i)*\vcch}) {\i};

\node[font=\tiny\bfseries,text=stairgray,anchor=south]
     at ({\cachex+3*\vccw},{\vby+8*\vcch+0.06}) {V cache};

\node[font=\tiny,text=red!55!black] at ({\vbrkx+0.06},{\vby+5.5*\vcch}) {$\times$};
\node[font=\tiny,text=red!55!black] at ({\vbrkx+0.06},{\vby+2.5*\vcch}) {$\times$};
\draw[stairgray!45, decorate, decoration={brace, amplitude=2.5pt, mirror}]
  ({\vbrkx+0.22},\vby) -- ({\vbrkx+0.22},{\vby+8*\vcch});
\node[font=\tiny\itshape,text=stairgray,rotate=-90]
     at (\vbrktxt,{\vby+4*\vcch})
     {$\mathcal{T}_{\text{kept}}\,/\,\mathcal{T}_{\text{evict}}$};

\end{scope}

\begin{scope}[shift={(0.91, 0.7)}, scale=0.8, transform shape]

\node[font=\scriptsize\bfseries, text=stairgray, anchor=west]
     at (3.10, 2.45) {\strut Stage 3: Channel Allocation};

\fill[bitsixteen!10] (\kbx,{\kby+\ktA}) rectangle ({\kbx+\chartw},{\kby+\maxh+0.06});
\fill[biteight!10]   (\kbx,{\kby+\ktB}) rectangle ({\kbx+\chartw},{\kby+\ktA});
\fill[bitfour!10]    (\kbx,{\kby+\ktC}) rectangle ({\kbx+\chartw},{\kby+\ktB});
\fill[bittwo!10]     (\kbx,{\kby+\ktD}) rectangle ({\kbx+\chartw},{\kby+\ktC});
\fill[hatchbg!12]    (\kbx,\kby)         rectangle ({\kbx+\chartw},{\kby+\ktD});

\foreach \tl in {\ktA,\ktB,\ktC,\ktD}
  \draw[stairgray!45, line width=0.25pt, dashed]
    (\kbx,{\kby+\tl}) -- ({\kbx+\chartw},{\kby+\tl});

\node[font=\tiny,text=bitsixteen,anchor=west] at (\kzl,{\kby+(\ktA+\maxh+0.06)/2}) {16};
\node[font=\tiny,text=biteight,anchor=west]   at (\kzl,{\kby+(\ktB+\ktA)/2})       {8};
\node[font=\tiny,text=bitfour,anchor=west]    at (\kzl,{\kby+(\ktC+\ktB)/2})       {4};
\node[font=\tiny,text=bittwo,anchor=west]     at (\kzl,{\kby+(\ktD+\ktC)/2})       {2};
\node[font=\tiny,text=hatchstroke,anchor=west] at (\kzl,{\kby+\ktD/2})              {0};

\foreach \i/\h/\clr in {%
  0/1.364/bitsixteen, 1/0.651/bitfour, 3/1.054/biteight,%
  4/0.465/bittwo,     5/0.806/bitfour, 7/0.744/bitfour%
}{%
  \fill[\clr]  ({\kbx+\i*\barstep},\kby) rectangle ({\kbx+\i*\barstep+\barw},{\kby+\h});
  \draw[stairgray!50, line width=0.12pt] ({\kbx+\i*\barstep},\kby) rectangle ({\kbx+\i*\barstep+\barw},{\kby+\h});
}
\foreach \i/\h in {2/0.232, 6/0.155}{%
  \fill[hatchbg] ({\kbx+\i*\barstep},\kby) rectangle ({\kbx+\i*\barstep+\barw},{\kby+\h});
  \fill[pattern=north east lines, pattern color=hatchstroke, opacity=0.50]
    ({\kbx+\i*\barstep},\kby) rectangle ({\kbx+\i*\barstep+\barw},{\kby+\h});
  \draw[stairgray!50, line width=0.12pt] ({\kbx+\i*\barstep},\kby) rectangle ({\kbx+\i*\barstep+\barw},{\kby+\h});
}
\foreach \i in {0,...,7}{%
  \node[font=\tiny,text=stairgray] at ({\kbx+\i*\barstep+\barw/2},{\kby-0.12})
       {\pgfmathparse{int(\i+1)}\pgfmathresult};
}
\draw[stairgray, line width=0.30pt] (\kbx,\kby) -- ({\kbx+\chartw},\kby);

\end{scope}

\draw[->, stairgray, line width=0.40pt]
  (6.10, 1.6) -- node[above,font=\tiny\itshape,inner sep=1pt]{MCKP} (6.95, 1.6);

\begin{scope}[shift={(1.63, 0.7)}, scale=0.8, transform shape]

\foreach \col/\clr in {%
  0/bitsixteen, 1/bitfour, 3/biteight,%
  4/bittwo,     5/bitfour, 7/bitfour%
}{%
  \fill[\clr] ({\cachex+\col*\kccw},\kcy) rectangle ({\cachex+(\col+1)*\kccw},{\kcy+6*\kcch});
}
\foreach \col in {2,6}{%
  \fill[hatchbg] ({\cachex+\col*\kccw},\kcy) rectangle ({\cachex+(\col+1)*\kccw},{\kcy+6*\kcch});
  \fill[pattern=north east lines, pattern color=hatchstroke, opacity=0.40]
    ({\cachex+\col*\kccw},\kcy) rectangle ({\cachex+(\col+1)*\kccw},{\kcy+6*\kcch});
}
\foreach \c in {1,...,7}
  \draw[white, line width=0.12pt] ({\cachex+\c*\kccw},\kcy) -- ({\cachex+\c*\kccw},{\kcy+6*\kcch});
\foreach \r in {1,...,5}
  \draw[white, line width=0.12pt] (\cachex,{\kcy+\r*\kcch}) -- ({\cachex+8*\kccw},{\kcy+\r*\kcch});
\draw[stairgray, line width=0.30pt] (\cachex,\kcy) rectangle ({\cachex+8*\kccw},{\kcy+6*\kcch});

\foreach \ri/\tok in {5/1, 4/2, 3/4, 2/5, 1/7, 0/8}
  \node[font=\tiny,text=stairgray,anchor=east] at ({\cachex-0.06},{\kcy+(\ri+0.5)*\kcch}) {\tok};

\node[font=\tiny\bfseries,text=stairgray,anchor=south]
     at ({\cachex+4*\kccw},{\kcy+6*\kcch+0.06}) {K cache};

\node[font=\tiny\itshape,text=stairgray!70, anchor=north west]
     at (\cachex,{\kcy-0.10})
     {on $\mathcal{T}_{\text{kept}}$ only};

\end{scope}

\draw[->,stairgray!50, line width=0.30pt, dashed]
  (8.65, 2.93)
  to[out=330, in=30]
  node[right, font=\tiny\itshape, text=stairgray,
       inner sep=2pt, pos=0.5]
  {$\mathcal{T}_{\text{kept}}$}
  (8.65, 2.38);

\begin{scope}[shift={(1.55, 0.7)}, scale=0.8, transform shape]
\def\ly{0.00}
\node[font=\tiny\bfseries,text=stairgray] at (2.30,\ly) {bit:};

\fill[bitsixteen] (2.65,{\ly-0.12}) rectangle (2.93,{\ly+0.12});
\draw[stairgray, line width=0.10pt] (2.65,{\ly-0.12}) rectangle (2.93,{\ly+0.12});
\node[font=\tiny,text=stairgray,anchor=west] at (2.85,\ly) {16};

\fill[biteight] (3.30,{\ly-0.12}) rectangle (3.58,{\ly+0.12});
\draw[stairgray, line width=0.10pt] (3.30,{\ly-0.12}) rectangle (3.58,{\ly+0.12});
\node[font=\tiny,text=stairgray,anchor=west] at (3.50,\ly) {8};

\fill[bitfour] (3.85,{\ly-0.12}) rectangle (4.13,{\ly+0.12});
\draw[stairgray, line width=0.10pt] (3.85,{\ly-0.12}) rectangle (4.13,{\ly+0.12});
\node[font=\tiny,text=stairgray,anchor=west] at (4.05,\ly) {4};

\fill[bittwo] (4.40,{\ly-0.12}) rectangle (4.68,{\ly+0.12});
\draw[stairgray, line width=0.10pt] (4.40,{\ly-0.12}) rectangle (4.68,{\ly+0.12});
\node[font=\tiny,text=stairgray,anchor=west] at (4.60,\ly) {2};

\fill[hatchbg] (4.95,{\ly-0.12}) rectangle (5.23,{\ly+0.12});
\fill[pattern=north east lines, pattern color=hatchstroke, opacity=0.50]
     (4.95,{\ly-0.12}) rectangle (5.23,{\ly+0.12});
\draw[stairgray, line width=0.10pt] (4.95,{\ly-0.12}) rectangle (5.23,{\ly+0.12});
\node[font=\tiny,text=stairgray,anchor=west] at (5.15,\ly) {0\,(evict)};
\end{scope}

\end{tikzpicture}%
}
\caption{RDKV per-head bit-allocation pipeline (illustrated for 8 tokens and 8 channels).
\textbf{Stage 1.}~Token weights~$w_t=\sum_\tau a_{\tau,t}$ are derived from the
attention matrix, and channel weights~$w_c$ from Q/K column norms.
\textbf{Stage 2.}~Reverse water-filling on~$w_t$: four thresholds partition
scores into five bit-width zones $\{16,8,4,2,0\}$.
In this example tokens~3 and~6 fall below the lowest threshold and are
evicted (hatched rows in V cache).
\textbf{Stage 3.}~The same allocation on~$w_c$ assigns per-channel bit-widths
to the K cache, which retains only the six kept tokens
($\mathcal{T}_{\text{kept}}=\{1,2,4,5,7,8\}$).}
\label{fig:methodology-pipeline}
\end{figure}
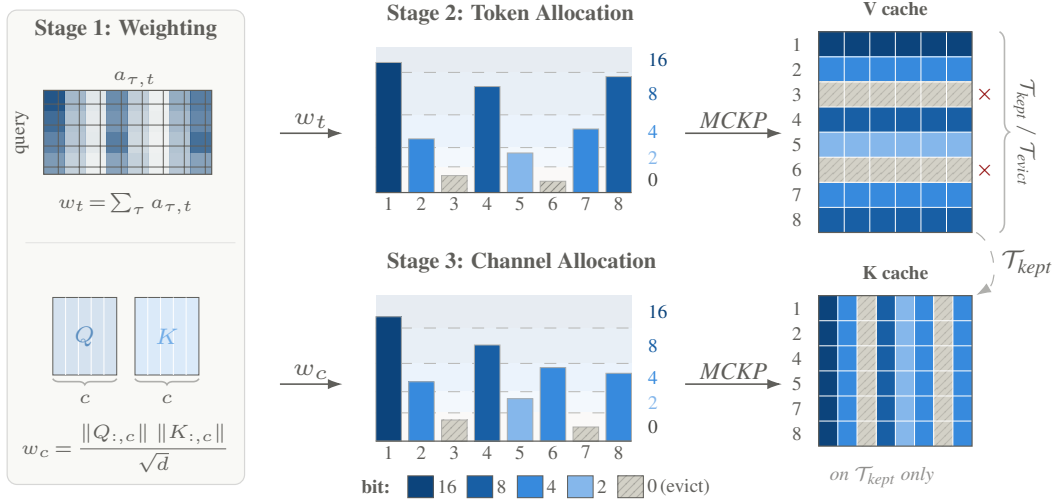

We instantiate the weights and allocation of
\cref{sec:rd_formulation} as a three-stage pipeline
(\cref{fig:methodology-pipeline}) that runs once after prefill,
separately for each layer-head pair $(\ell, h)$.
Under FP16 reference, a per-head budget of $B_\text{tok}$ tokens
corresponds to $B_\text{head} = 2\,B_\text{tok}\,d\cdot 16$ total
bits, split equally between V and K:
$B_V = B_K = \frac{1}{2}B_\text{head}$.

\textbf{Stage~1: Distortion Weight Computation.}
From the prefill forward pass we compute the token weights~$w_t$
(\cref{thm:v-score}) and channel weights~$w_c$
(\cref{thm:k-score}), as illustrated in
\cref{fig:methodology-pipeline}\,(a).
Both are computed once from the full uncompressed cache:
$w_t$ requires only the attention matrix;
$w_c$ requires only Q and K column norms, with no value vectors or
output residuals needed.

\textbf{Stage~2: Token Allocation in V Cache.}
Since each token contains $d$ scalars, the total V budget $B_V$
translates to $\bar{B}_V := B_V / d$ in units of summed bit-widths.
Given the token weights $w_t$ and a calibrated distortion table
$\varepsilon_V(b)$ (\cref{app:impl-details}),
Stage~2 solves the MCKP of \cref{eq:mckp}, as shown in
\cref{fig:methodology-pipeline}\,(b):
\[
\{b_t^V\}
  = \argmin_{b_t^V \in \mathcal{B}}\;
  \sum_t w_t\,\varepsilon_V(b_t^V)
  \qquad
  \text{s.t.}\;\;\sum_t b_t^V \le \bar{B}_V.
\]
The output determines
$\mathcal{T}_\text{kept} := \{t : b_t^V > 0\}$
and its complement $\mathcal{T}_\text{evict}$.
Retained tokens receive bit-widths in $\{2, 4, 8, 16\}$.
Evicted tokens are removed entirely.

\textbf{Stage~3: Channel Allocation in K Cache.}
As shown in \cref{fig:methodology-pipeline}\,(c),
Stage~3 distributes the K budget $B_K$ across only the retained
tokens $\mathcal{T}_\text{kept}$.
Each channel contains $|\mathcal{T}_\text{kept}|$ scalars, so the
budget translates to
$\bar{B}_K := B_K / |\mathcal{T}_\text{kept}|$:
\[
\{b_c^K\}
  = \argmin_{b_c^K \in \mathcal{B}}\;
  \sum_c w_c\,\varepsilon_K(b_c^K)
  \qquad
  \text{s.t.}\;\;\sum_c b_c^K \le \bar{B}_K.
\]
A channel assigned $b_c^K = 0$ is removed under the same zero-rate
interpretation as token eviction.
Across retained tokens, K bit-widths are assigned per-channel
independently of the per-token~$b_t^V$.
A single token's K and V can therefore carry different bit-widths, the
prerequisite for the TriZone layout of \cref{sec:efficient_decode}.
Visualization of the resulting bit-width allocations of tokens and channels
is in \cref{app:bit-vis} .

The V cache is allocated before K because the eviction decision
determines which tokens remain in the K cache.
Reversing the order would waste K budget on soon-to-be-evicted tokens.
The channel weights~$w_c$ are reused from Stage~1 without
recomputation after V-side eviction, avoiding a second forward pass.
The complete pseudocode is given in \cref{app:algorithm}.





\subsection{Efficient Packed Decode}
\label{sec:efficient_decode}

\begin{wrapfigure}{r}{0.52\textwidth}
\centering
\vspace{-10pt}
\begin{tikzpicture}[
    zone/.style={draw, minimum width=2.8cm, minimum height=0.7cm, align=center, font=\footnotesize},
    rowlbl/.style={font=\footnotesize\bfseries, anchor=east},
    collbl/.style={font=\footnotesize\bfseries, anchor=south},
    >=latex
]

\node[collbl] at (0, 0.6)   {\textbf{K}};
\node[collbl] at (3.0, 0.6) {\textbf{V}};

\node[rowlbl] at (-1.6, 0)    {Zone A};
\node[rowlbl] at (-1.6, -1.0) {Zone B};
\node[rowlbl] at (-1.6, -2.0) {Zone C};

\node[zone, fill=blue!8] (AK) at (0, 0)
    {all retained, packed};
\node[zone, fill=blue!8] (AV) at (3.0, 0)
    {$b^V\!\in\!\{2,4,8\}$, packed};

\node[zone, fill=gray!10, font=\footnotesize\itshape, text=black!55] (BK) at (0, -1.0)
    {(stored in Zone A)};
\node[zone, fill=orange!12] (BV) at (3.0, -1.0)
    {$b^V\!=\!16$, FP16};

\node[zone, fill=green!8] (CK) at (0, -2.0)
    {new K, FP16};
\node[zone, fill=green!8] (CV) at (3.0, -2.0)
    {new V, FP16};

\end{tikzpicture}
\caption{TriZone cache layout for one $(\ell,h)$ pair. Each zone admits a uniform dequantization path; cell labels show the stored format.}
\label{fig:trizone-layout}
\vspace{-10pt}
\end{wrapfigure}
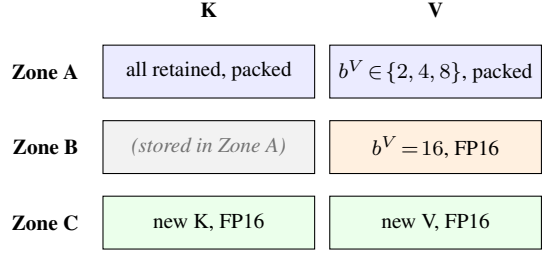

The pipeline of \cref{sec:vk_pipeline} produces a mixed-bit cache per
$(\ell,h)$ pair.
A mixed-bit allocation alone does not reduce decode cost: if the
quantized entries are unpacked to FP16 before the attention kernel
reads them, peak HBM usage stays unchanged and the extra
dequantization pass adds latency.
To translate bit-width savings into actual speedup and memory
reduction, the cache must stay packed in HBM while dequantization is
fused into the attention computation.
The challenge is that tokens within one head now carry different
bit-widths, yet GPU kernels require uniform-precision input segments
to avoid per-element branching.

We address this by organizing each $(\ell,h)$ pair into three storage zones that separate
packed quantized entries, full-precision retained entries, and newly
generated decode tokens. Each zone admits a uniform
dequantization path during the attention kernel
(\cref{fig:trizone-layout}; we suppress $(\ell,h)$ in what follows).

Zone~A is the packed old-cache stream: all retained K rows from
$\mathcal{T}_{\text{kept}}$, together with the V rows whose
$b_t^V \in \{2,4,8\}$. Within Zone~A, V rows are grouped into
uniform-bit segments by $b_t^V$ and K channels by $b_c^K$
(\cref{app:impl-details}). Each segment uses a single dequantization
path, avoiding per-element branching. Zone~B holds the FP16 V rows for
$\mathcal{T}_{V16} := \{t \in \mathcal{T}_{\text{kept}} : b_t^V = 16\}$;
their K rows remain in Zone~A, since K bit-widths follow the
per-channel allocation of \cref{sec:vk_pipeline} rather than
$b^V$. A single token's K and V can therefore reside in different
zones. Zone~C appends FP16 K and V for newly generated tokens
$\mathcal{T}_{\text{new}}$, growing by one entry per decode step.

At each decode step, the query $q_\tau$ attends over retained prefill
tokens and newly appended decode tokens. Logits combine the quantized
old K with the FP16 new K,
\[
z_\tau
  \;=\;
  q_\tau
  \bigl[\,\hat{K}_{\mathcal{T}_{\text{kept}}}^\top,\;
         K_{\mathcal{T}_{\text{new}}}^\top\bigr]
  \big/\sqrt{d},
  \qquad
  a_\tau = \operatorname{softmax}(z_\tau),
\]
and the output decomposes into three value sources:
\begin{equation}\label{eq:trizone-output}
o_\tau
  \;=\;
  \underbrace{\sum_{t \in \mathcal{T}_{\text{kept}} \setminus \mathcal{T}_{V16}}
    a_{\tau,t}\,\hat{v}_t}_{\text{quantized V}}
  \;+\;
  \underbrace{\sum_{t \in \mathcal{T}_{V16}}
    a_{\tau,t}\,v_t}_{\text{FP16 V}\;(b^V\!=\!16)}
  \;+\;
  \underbrace{\sum_{t \in \mathcal{T}_{\text{new}}}
    a_{\tau,t}\,v_t}_{\text{FP16 V (new)}}.
\end{equation}
The three terms correspond to the three zones: the first reads packed
quantized V from Zone~A, the second reads FP16~V from Zone~B, and
the third reads FP16~V from Zone~C.
\Cref{eq:trizone-output} is the numerical reference implemented by the
packed decode path: dense attention on the compressed cache
$(\hat{K}, \hat{V})$, with dequantization fused into the attention
computation. The latency measurements in \cref{sec:experiments} use this
implementation. \Cref{app:impl-details} gives layout details including bit-width
grouping, padding, and masking.

\section{Experiments}
\label{sec:experiments}

\textbf{Models.} We primarily evaluate on
LLaMA-3.1-8B-Instruct~\citep{llama3}. Cross-architecture results
(Mistral-7B-Instruct-v0.3~\citep{mistral},
Qwen3-4B~\citep{qwen3}) and cross-scale results
(LLaMA2-13B-Chat~\citep{touvron2023llama}, Qwen2.5-72B-Instruct~\citep{qwen25}) are listed in
\cref{app:more-models}.

\textbf{Baselines.}
We compare against four budget-driven methods.
SnapKV~\citep{li2024snapkv} scores tokens by attention within a
recent observation window and retains the top-$k$ per head.
Ada-SnapKV~\citep{feng2024adakv} (hereafter AdaKV) extends this with adaptive per-head budget
allocation.
ThinK~\citep{xu2024think} applies pruning at the channel level of the
K cache rather than the token level (paired with SnapKV).
SnapKV+ZipCache~\citep{li2024snapkv,he2024zipcache} combines token
eviction with post-hoc quantization of the retained cache.

\textbf{Benchmarks.}
We evaluate on LongBench~\citep{bai2023longbench} (16 tasks covering 6 different task categories),
Needle-in-a-Haystack~\citep{kamradt2023niah} (single-fact retrieval at varying depths),
RULER~\citep{hsieh2024ruler} (11 retrieval and reasoning tasks from 4K to 128K),
and InfiniteBench~\citep{zhang2024infbench} (10 tasks with context lengths up to $2$M tokens).

\textbf{Evaluation Protocol.}
All methods are evaluated under matched per-layer cache budgets
$B_\text{total} = nL$ with
$n \in \{64, 128, 256, 512, 1024\}$.
For RDKV, $n$ is the FP16-equivalent token count reallocated across
tokens (V) and channels (K) at bit-widths
$\mathcal{B} = \{0, 2, 4, 8, 16\}$. 
All score-based methods share the same size of observation window
and pooling kernel.
Further details in \cref{app:impl-details}.

\subsection{Main Results}
\label{sec:exp-longbench}

\textbf{LongBench.}
LongBench~\citep{bai2023longbench} evaluates long-context understanding across 16 tasks
spanning single/multi-document QA, summarization, few-shot learning, synthetic, and code.
\Cref{tab:longbench-llama31} reports per-task scores for
LLaMA-3.1-8B-Instruct~\citep{llama3} across five cache budgets.
RDKV achieves the highest average at every budget:
at $B_\text{total}=1024L$ it reaches $49.47$, within $0.35$ points of
FullKV; at $64L$ the lead over the strongest baseline widens to $2.7$
points.
This reflects a structural advantage of joint allocation.
Binary baselines discard every token below the top-$k$ threshold.
RDKV instead chooses from a larger action space via reverse
water-filling, assigning borderline tokens a low bit-width rather
than evicting them.
SnapKV+ZipCache~\citep{li2024snapkv,he2024zipcache} also quantizes,
but treats eviction and quantization as separate stages rather than
two ends of one allocation curve.
The advantage spans all six categories; the largest per-task gain
appears on GovRep ($3$--$5$ points), where attention dispersed across
lengthy documents favors mixed-precision over binary eviction.
Cross-model results (Mistral-7B~\citep{mistral}, Qwen3-4B~\citep{qwen3},
LLaMA-2-13B~\citep{touvron2023llama}, Qwen2.5-72B~\citep{qwen25}) are
in \cref{app:longbench-detailed}.

\begin{table}[!tb]
\centering
\caption{Performance on 16 LongBench~\citep{bai2023longbench} datasets for LLaMA-3.1-8B-Instruct across cache
budgets $B_\text{total} \in \{64L, 128L, 256L, 512L, 1024L\}$.
Snap+Zip denotes SnapKV~\citep{li2024snapkv}+ZipCache~\citep{he2024zipcache}.
The best result in each row is in \textbf{bold}; the second-best is \underline{underlined}.}
\label{tab:longbench-llama31}
\setlength{\tabcolsep}{0.6pt}
\fontsize{6.4}{7.6}\selectfont
\begin{tabular}{lccccccccccccccccc}
\toprule
\multirow{2}{*}{Method}
 & \multicolumn{3}{c}{Single-Doc QA}
 & \multicolumn{3}{c}{Multi-Doc QA}
 & \multicolumn{3}{c}{Summarization}
 & \multicolumn{3}{c}{Few-shot Learning}
 & \multicolumn{2}{c}{Synthetic}
 & \multicolumn{2}{c}{Code}
 & \multirow{2}{*}{Avg.} \\
\cmidrule(lr){2-4}\cmidrule(lr){5-7}\cmidrule(lr){8-10}\cmidrule(lr){11-13}\cmidrule(lr){14-15}\cmidrule(lr){16-17}
 & \rotatebox{35}{NrtvQA} & \rotatebox{35}{Qasper} & \rotatebox{35}{MF-en}
 & \rotatebox{35}{HotpotQA} & \rotatebox{35}{2WikiMQA} & \rotatebox{35}{Musique}
 & \rotatebox{35}{GovRep} & \rotatebox{35}{QMSum} & \rotatebox{35}{MultiNews}
 & \rotatebox{35}{TREC} & \rotatebox{35}{TriviaQA} & \rotatebox{35}{SAMSum}
 & \rotatebox{35}{PCount} & \rotatebox{35}{PRe}
 & \rotatebox{35}{Lcc} & \rotatebox{35}{RB-P}
 & \\
\midrule
\multicolumn{18}{c}{\textit{LLaMA-3.1-8B-Instruct,} $B_\text{total} = \text{Full}$} \\
\midrule
FullKV & 29.38 & 44.82 & 55.46 & 57.62 & 49.22 & 32.80 & 34.58 & 25.34 & 26.90 & 73.00 & 91.51 & 43.64 & 9.04 & 99.50 & 65.24 & 59.08 & 49.82 \\
\midrule
\multicolumn{18}{c}{\textit{LLaMA-3.1-8B-Instruct,} $B_\text{total} = 64L$} \\
\midrule
SnapKV\citep{li2024snapkv} & 25.20 & 21.05 & 39.35 & 51.69 & 45.25 & 26.57 & 15.06 & 21.24 & 14.60 & 39.00 & 81.43 & 36.62 & \underline{9.00} & 95.00 & 54.74 & 47.41 & 38.95 \\
AdaKV\citep{feng2024adakv} & 24.62 & 19.62 & 39.41 & 50.93 & 42.13 & 26.10 & 16.08 & 22.17 & 15.71 & 39.50 & 85.62 & 37.83 & \underline{9.00} & 97.50 & 56.57 & 50.70 & 39.59 \\
ThinK\citep{xu2024think} & 24.93 & 23.54 & 45.55 & 53.54 & 44.17 & 27.79 & 18.38 & 22.77 & 18.10 & 42.00 & 84.35 & 38.50 & \underline{9.00} & \underline{98.50} & 55.81 & 49.22 & 41.01 \\
Snap+Zip & \textbf{29.85} & \underline{32.53} & \underline{52.16} & \underline{55.24} & \underline{46.72} & \textbf{30.60} & \underline{19.96} & \underline{22.88} & \underline{21.26} & \underline{52.50} & \underline{90.50} & \underline{40.03} & \textbf{9.01} & 75.50 & \underline{61.01} & \underline{51.97} & \underline{43.23} \\
\textbf{RDKV}   & \underline{28.00} & \textbf{34.47} & \textbf{54.16} & \textbf{55.90} & \textbf{47.45} & \underline{29.65} & \textbf{23.10} & \textbf{23.78} & \textbf{22.54} & \textbf{59.50} & \textbf{91.76} & \textbf{41.52} & 8.70 & \textbf{99.17} & \textbf{61.70} & \textbf{54.10} & \textbf{45.97} \\
\midrule
\multicolumn{18}{c}{\textit{LLaMA-3.1-8B-Instruct,} $B_\text{total} = 128L$} \\
\midrule
SnapKV\citep{li2024snapkv} & 24.41 & 24.19 & 48.91 & 55.98 & 45.51 & 28.41 & 18.64 & 22.76 & 18.72 & 47.00 & 88.91 & 39.41 & \textbf{9.00} & 98.00 & 59.14 & 52.18 & 42.57 \\
AdaKV\citep{feng2024adakv} & 26.35 & 23.50 & 51.59 & 57.11 & 46.38 & 28.36 & 19.68 & 23.09 & 19.58 & 49.50 & 89.92 & 40.25 & \textbf{9.00} & \underline{99.00} & 60.42 & \underline{54.59} & 43.64 \\
ThinK\citep{xu2024think} & 25.40 & 30.67 & 49.96 & 55.71 & 46.76 & 28.47 & 21.44 & 23.35 & 20.83 & 50.00 & 89.87 & 40.97 & \textbf{9.00} & \underline{99.00} & 60.18 & 53.22 & 44.05 \\
Snap+Zip & \underline{29.40} & \underline{36.64} & \underline{53.26} & \underline{57.27} & \underline{47.23} & \textbf{31.09} & \underline{22.59} & \underline{23.66} & \underline{23.97} & \underline{61.00} & \textbf{92.00} & \textbf{42.17} & 8.22 & 95.50 & \underline{62.24} & 53.57 & \underline{46.24} \\
\textbf{RDKV}   & \textbf{29.45} & \textbf{41.34} & \textbf{55.55} & \textbf{57.39} & \textbf{49.38} & \underline{30.89} & \textbf{25.55} & \textbf{24.59} & \textbf{24.29} & \textbf{65.00} & \underline{91.92} & \underline{41.80} & \underline{8.75} & \textbf{100.00} & \textbf{62.68} & \textbf{55.41} & \textbf{47.75} \\
\midrule
\multicolumn{18}{c}{\textit{LLaMA-3.1-8B-Instruct,} $B_\text{total} = 256L$} \\
\midrule
SnapKV\citep{li2024snapkv} & 26.90 & 31.27 & 53.45 & 56.80 & 47.97 & 29.52 & 21.99 & \underline{24.34} & 21.31 & 55.00 & 91.41 & 40.75 & 8.61 & 99.00 & 62.20 & 55.15 & 45.35 \\
AdaKV\citep{feng2024adakv} & 25.22 & 32.02 & 52.51 & 56.82 & 47.65 & 30.71 & 22.49 & 23.86 & 21.78 & 61.00 & 91.27 & 40.92 & \underline{8.63} & \textbf{99.75} & \textbf{63.42} & \textbf{56.58} & 45.91 \\
ThinK\citep{xu2024think} & 26.83 & 34.14 & 52.82 & 56.77 & \underline{49.14} & 29.08 & 23.63 & 23.47 & 22.87 & 61.50 & 91.51 & 41.35 & 8.60 & \underline{99.50} & 62.82 & 55.44 & 46.22 \\
Snap+Zip & \underline{28.02} & \underline{41.07} & \underline{54.05} & \underline{57.74} & 48.02 & \textbf{31.98} & \underline{24.27} & 23.80 & \underline{24.82} & \underline{66.00} & \underline{91.67} & \textbf{42.07} & 7.90 & \underline{99.50} & 61.72 & 54.18 & \underline{47.30} \\
\textbf{RDKV}   & \textbf{29.63} & \textbf{44.85} & \textbf{56.33} & \textbf{57.85} & \textbf{49.64} & \underline{31.73} & \textbf{28.01} & \textbf{24.66} & \textbf{25.66} & \textbf{70.00} & \textbf{91.87} & \underline{41.79} & \textbf{8.83} & \underline{99.50} & \underline{63.27} & \underline{56.01} & \textbf{48.73} \\
\midrule
\multicolumn{18}{c}{\textit{LLaMA-3.1-8B-Instruct,} $B_\text{total} = 512L$} \\
\midrule
SnapKV\citep{li2024snapkv} & \underline{29.68} & 39.51 & \textbf{55.94} & 56.94 & 48.87 & 30.56 & 24.31 & 24.07 & 23.60 & 65.50 & 91.46 & 41.44 & 8.57 & \textbf{100.00} & \underline{64.50} & 56.58 & 47.60 \\
AdaKV\citep{feng2024adakv} & 27.93 & 40.26 & 54.76 & \underline{56.96} & 48.68 & 31.45 & 25.27 & 24.09 & 23.71 & \underline{69.50} & \underline{91.96} & 41.77 & \underline{8.65} & \underline{99.50} & \textbf{64.89} & \textbf{57.66} & 47.94 \\
ThinK\citep{xu2024think} & 27.96 & 40.16 & 55.25 & 56.85 & \textbf{49.60} & 31.16 & \underline{26.00} & 24.36 & 24.28 & 68.00 & \textbf{92.00} & 41.79 & 8.63 & \textbf{100.00} & 64.08 & \underline{57.30} & \underline{47.96} \\
Snap+Zip & 29.32 & \underline{41.46} & 53.15 & 55.47 & 46.77 & \textbf{33.26} & 23.64 & \underline{24.88} & \underline{24.89} & 68.00 & 91.75 & \underline{42.32} & 8.22 & \underline{99.50} & 61.75 & 55.03 & 47.46 \\
\textbf{RDKV}   & \textbf{30.24} & \textbf{45.67} & \underline{55.58} & \textbf{57.06} & \underline{48.95} & \underline{31.67} & \textbf{30.68} & \textbf{25.06} & \textbf{26.46} & \textbf{72.00} & 91.83 & \textbf{43.15} & \textbf{8.83} & \textbf{100.00} & 63.36 & 57.02 & \textbf{49.22} \\
\midrule
\multicolumn{18}{c}{\textit{LLaMA-3.1-8B-Instruct,} $B_\text{total} = 1024L$} \\
\midrule
SnapKV\citep{li2024snapkv} & 29.51 & 43.17 & \textbf{56.26} & 57.43 & \underline{49.18} & 32.07 & 27.21 & 24.63 & 25.26 & 69.50 & 91.70 & 42.28 & 8.15 & \textbf{100.00} & \textbf{64.71} & \textbf{58.52} & 48.72 \\
AdaKV\citep{feng2024adakv} & 29.76 & \underline{43.39} & \underline{55.82} & \underline{57.62} & 48.31 & \underline{32.35} & 27.22 & 24.75 & 25.18 & \textbf{72.00} & 91.78 & 42.12 & 8.15 & \textbf{100.00} & \underline{64.56} & \underline{58.37} & \underline{48.84} \\
ThinK\citep{xu2024think} & 27.34 & 42.48 & 55.59 & \textbf{57.78} & 49.17 & \textbf{32.36} & \underline{28.23} & \underline{24.82} & \underline{25.85} & \underline{71.50} & \textbf{92.50} & 42.25 & 8.27 & \textbf{100.00} & 64.44 & \textbf{58.52} & 48.82 \\
Snap+Zip & \underline{30.23} & 40.65 & 54.13 & 56.06 & 46.23 & 31.45 & 22.84 & 24.50 & 24.28 & 69.00 & 91.76 & \textbf{44.15} & \underline{8.61} & \underline{99.00} & 60.90 & 56.38 & 47.51 \\
\textbf{RDKV}   & \textbf{30.41} & \textbf{44.54} & \underline{55.82} & 57.14 & \textbf{49.25} & 31.77 & \textbf{33.14} & \textbf{25.19} & \textbf{26.92} & \textbf{72.00} & \underline{91.84} & \underline{43.90} & \textbf{9.08} & \textbf{100.00} & 63.56 & 57.01 & \textbf{49.47} \\
\bottomrule
\end{tabular}
\end{table}

\textbf{Needle-In-A-Haystack.}\label{sec:exp-niah}
Needle-in-a-Haystack~\citep{kamradt2023niah} tests single-fact
retrieval by inserting a target fact at varying depths.
\Cref{fig:niah} compares methods under $B_\text{total}=64L$ at $32$k
context length.
SnapKV~\citep{li2024snapkv} and AdaKV~\citep{feng2024adakv} develop
failure bands at intermediate depths where the needle falls below the
top-$k$ threshold and is evicted entirely.
RDKV assigns the same token a low bit-width instead of discarding it. A 2- or 4-bit copy suffices to recover the answer.
RDKV (avg.\ $0.99$) thus maintains near-uniform retrieval close to
FullKV ($1.00$). Further results are in \cref{app:niah}.

\begin{figure}[!tb]
  \centering
  \begin{subfigure}[t]{0.49\linewidth}
    \includegraphics[width=\linewidth]{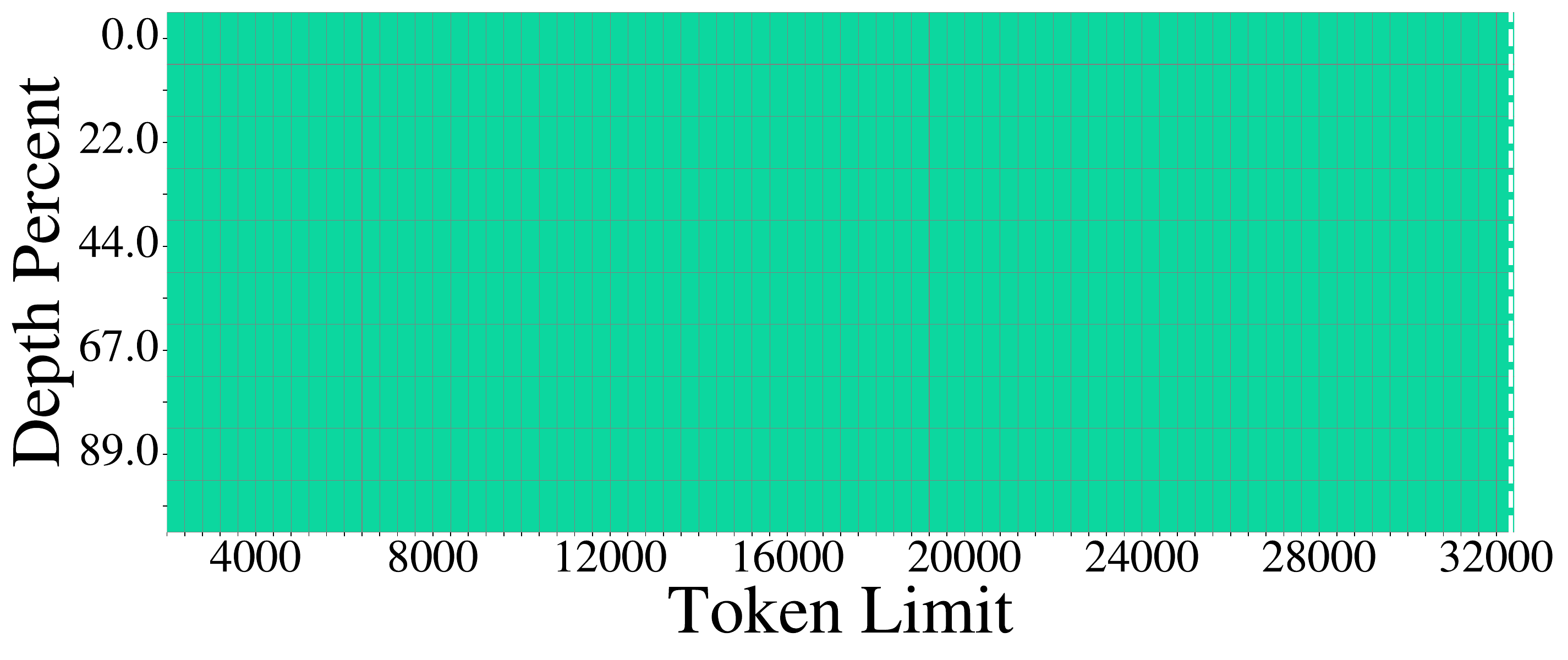}
    \caption{FullKV (avg.\ $1.00$).}
  \end{subfigure}
  \begin{subfigure}[t]{0.49\linewidth}
    \includegraphics[width=\linewidth]{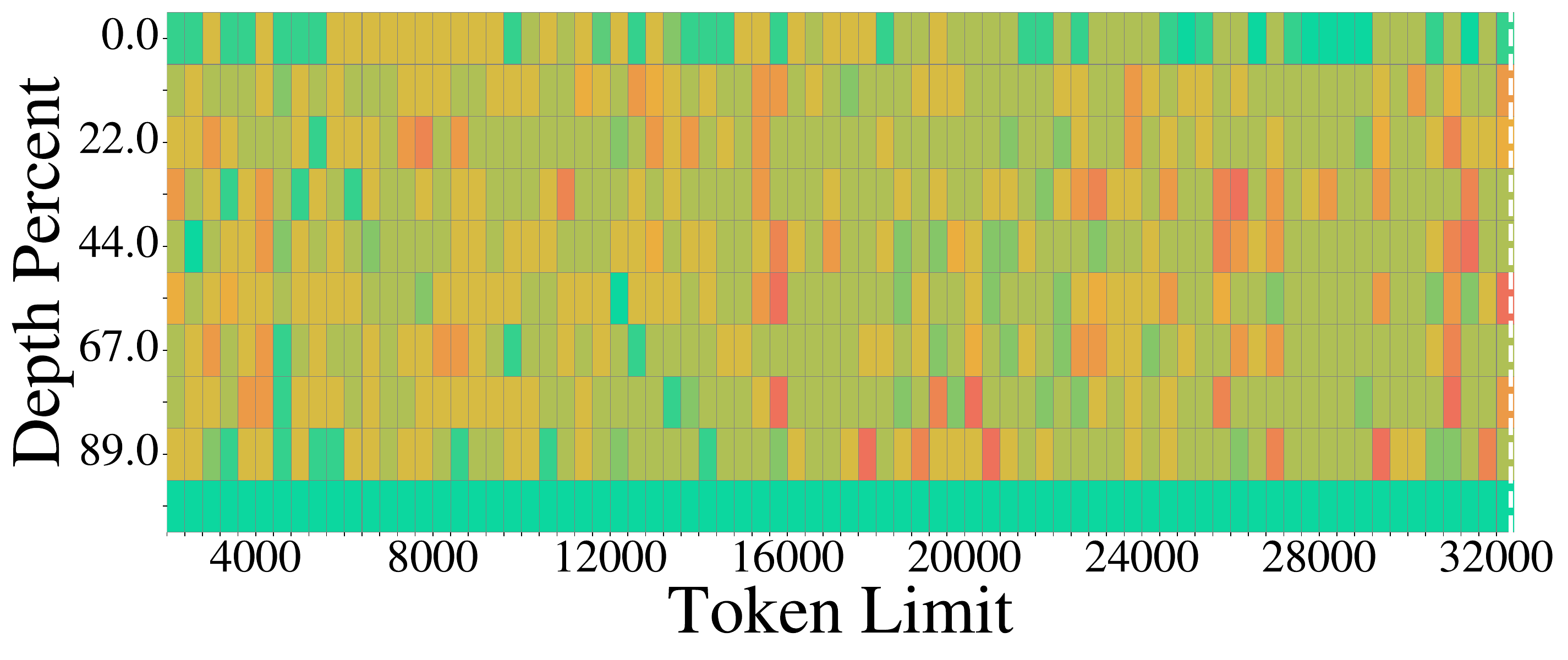}
    \caption{SnapKV~\citep{li2024snapkv} (avg.\ $0.64$).}
  \end{subfigure}\\[2pt]
  \begin{subfigure}[t]{0.49\linewidth}
    \includegraphics[width=\linewidth]{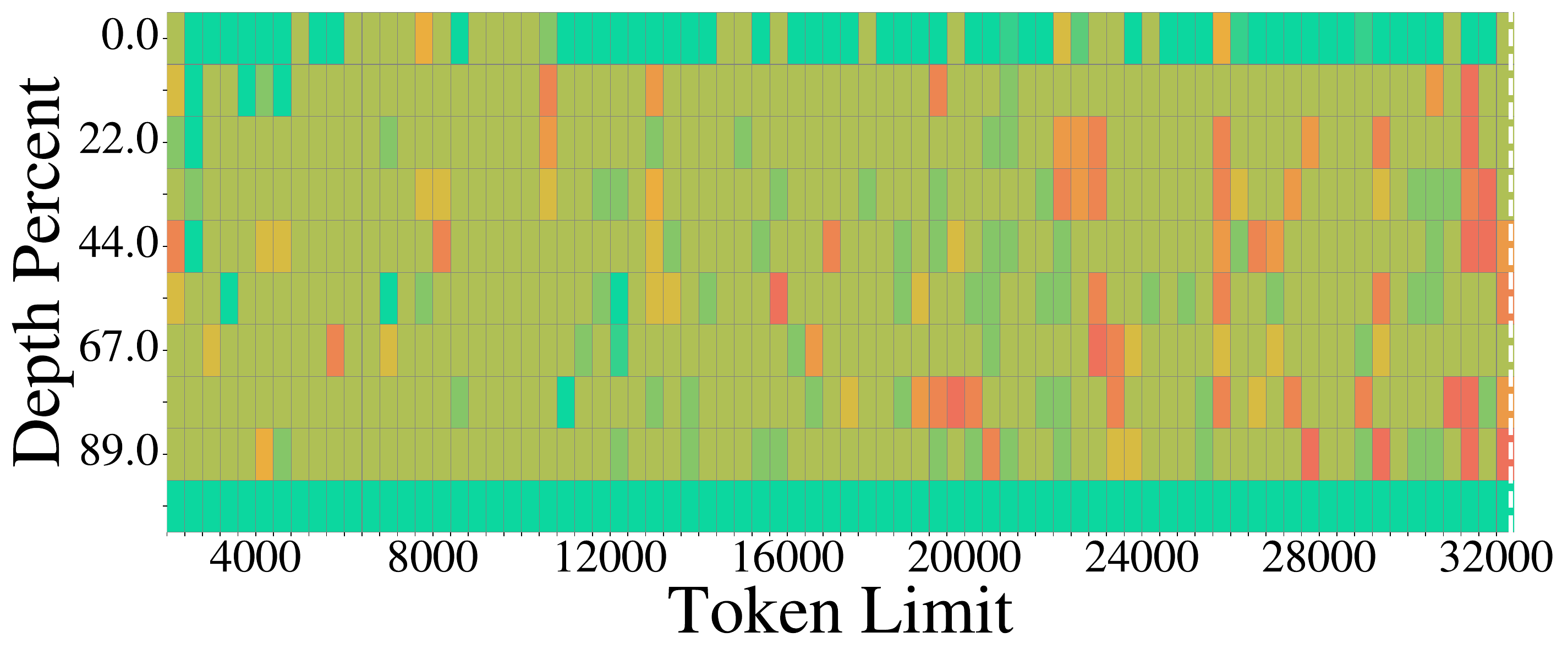}
    \caption{AdaKV~\citep{feng2024adakv} (avg.\ $0.68$).}
  \end{subfigure}
  \begin{subfigure}[t]{0.49\linewidth}
    \includegraphics[width=\linewidth]{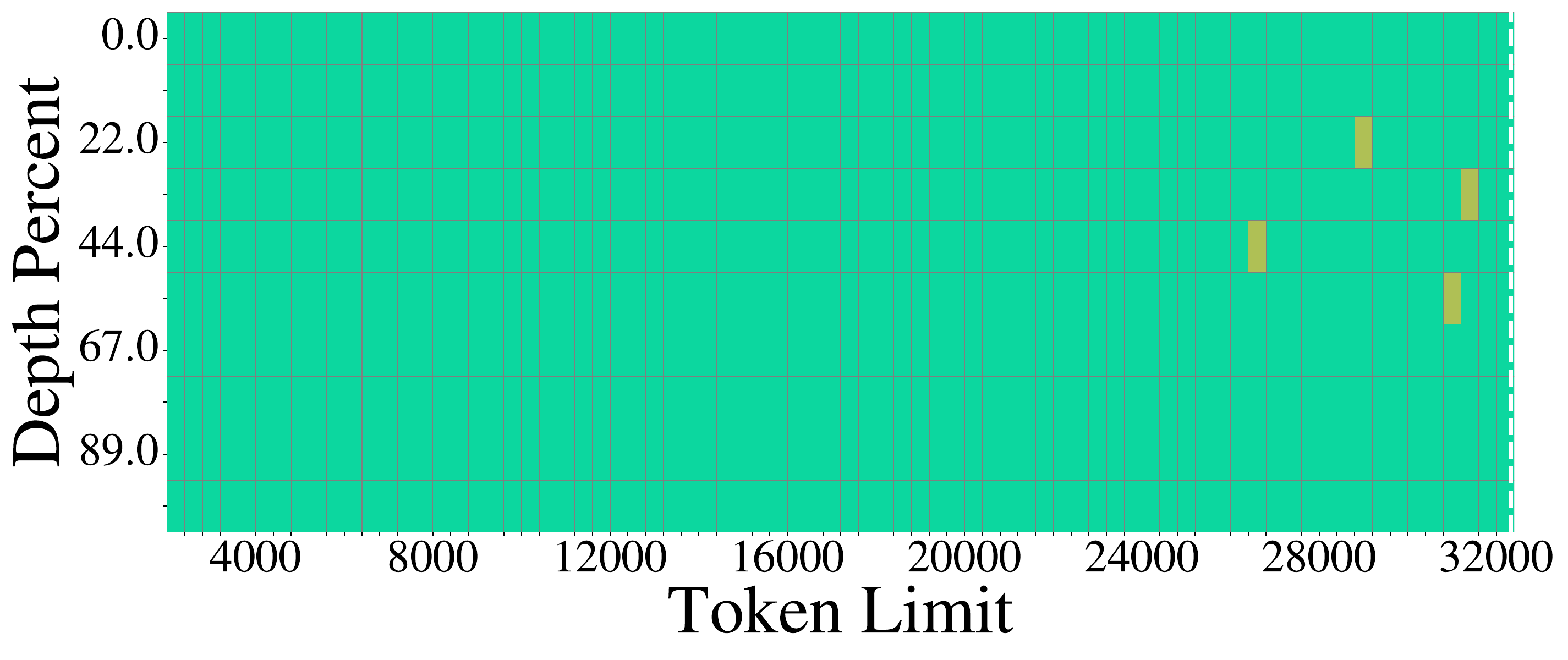}
    \caption{RDKV (avg.\ $0.99$). }
  \end{subfigure}
  \caption{Needle-in-a-Haystack~\citep{kamradt2023niah} on LLaMA-3.1-8B-Instruct at $B_\text{total} = 64L$. RDKV
  preserves a near-uniform retrieval pattern; SnapKV and AdaKV drop accuracy in mid-depth bands.}
  \label{fig:niah}
\end{figure}

\textbf{RULER.}
RULER~\citep{hsieh2024ruler} evaluates retrieval and reasoning at
extreme context lengths, averaging 11 subtasks at sequence lengths
from $4$k to $128$k.
\Cref{tab:ruler} reports results on LLaMA-3.1-8B-Instruct under
$B_\text{total} = 1024L$.
Eviction-only baselines degrade substantially as context grows,
falling over $6$ points below FullKV even at $4$k and $17$ points
at $128$k.
RDKV stays within $1$ point of FullKV up to $8$k and leads the
strongest baseline by $4.1$--$14.2$ points across all lengths,
because tokens that a top-$k$ policy would discard are instead
retained at low bit-width.
Per-task breakdowns are in \cref{app:ruler-detailed}.

\textbf{InfiniteBench.}
InfiniteBench~\citep{zhang2024infbench} spans 10 tasks with sequences
exceeding $100$k tokens (average ${\sim}200$k; the longest task,
Zh.QA, reaches ${\sim}2$M tokens).
\Cref{tab:infbench} reports results on LLaMA-3.1-8B-Instruct under
$B_\text{total} = 1024L$.
RDKV ranks first among compression methods, ahead of AdaKV by
$1.40$ and SnapKV by $1.55$ points, confirming that the allocation
advantage persists at ultra-long contexts well beyond the lengths
covered by LongBench and RULER.
Per-task results are in \cref{app:infbench-detailed}.

\begin{table}[!tb]
  \centering
  \begin{minipage}[t]{0.55\linewidth}
    \centering
    \small
    \caption{RULER~\citep{hsieh2024ruler} 11-task average across sequence lengths (LLaMA-3.1-8B-Instruct).}
    \label{tab:ruler}
    \setlength{\tabcolsep}{3pt}
    \begin{tabular}{lcccccc}
      \toprule
      Method     & 4K    & 8K    & 16K   & 32K   & 64K   & 128K \\
      \midrule
      FullKV     & 98.58 & 98.88 & 96.98 & 90.33 & 88.49 & 79.91 \\
      \midrule
      SnapKV\citep{li2024snapkv} & 89.48 & 84.75 & 80.43 & 75.82 & 72.77 & \underline{62.85} \\
      AdaKV\citep{feng2024adakv} & \underline{92.03} & 85.49 & 80.70 & 75.67 & 72.47 & 58.43 \\
     ThinK\citep{xu2024think} & 91.76 & \underline{86.44} & \underline{81.45} & \underline{75.94} & 72.44 & \underline{62.85} \\
      Snap+Zip   & 76.49 & 77.84 & 76.26 & 75.05 & \underline{74.59} & 62.37 \\
      \textbf{RDKV} & \textbf{98.62} & \textbf{98.61} & \textbf{95.65} & \textbf{88.28} & \textbf{80.07} & \textbf{66.95} \\
      \bottomrule
    \end{tabular}
  \end{minipage}\hfill
  \begin{minipage}[t]{0.42\linewidth}
    \centering
    \small
    \caption{InfiniteBench~\citep{zhang2024infbench} 10-task average (LLaMA-3.1-8B-Instruct).}
    \label{tab:infbench}
    \setlength{\tabcolsep}{6pt}
    \begin{tabular}{lc}
      \toprule
      Method     & Avg.\ \\
      \midrule
      FullKV     & 45.38 \\
      \midrule
      SnapKV\citep{li2024snapkv} & 37.91 \\
      AdaKV\citep{feng2024adakv} & \underline{38.06} \\
      ThinK\citep{xu2024think} & 36.73 \\
      Snap+Zip   & 37.76 \\
      \textbf{RDKV} & \textbf{39.46} \\
      \bottomrule
    \end{tabular}
  \end{minipage}
\end{table}

\subsection{Memory and Latency}
\label{sec:exp-efficiency}

\begin{figure}[!tb]
\centering
\begin{subfigure}[b]{0.32\textwidth}
  \includegraphics[width=\linewidth]{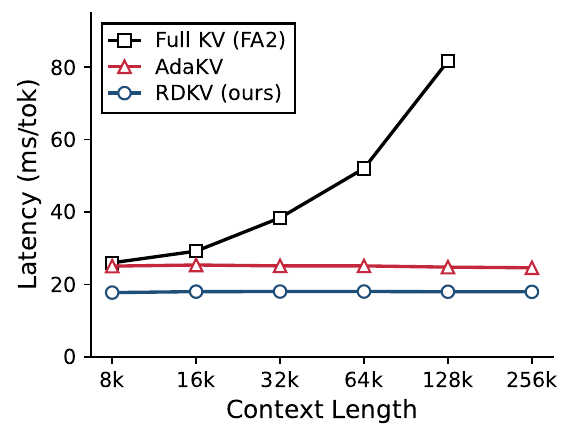}
  \caption{Decode latency.}
  \label{fig:eff-latency}
\end{subfigure}
\hfill
\begin{subfigure}[b]{0.32\textwidth}
  \includegraphics[width=\linewidth]{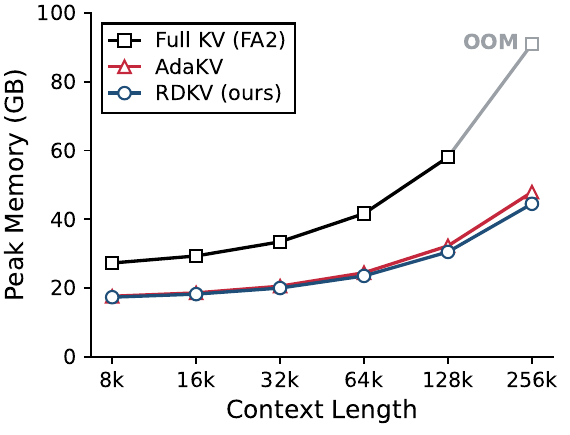}
  \caption{Peak memory.}
  \label{fig:eff-memory}
\end{subfigure}
\hfill
\begin{subfigure}[b]{0.32\textwidth}
  \includegraphics[width=\linewidth]{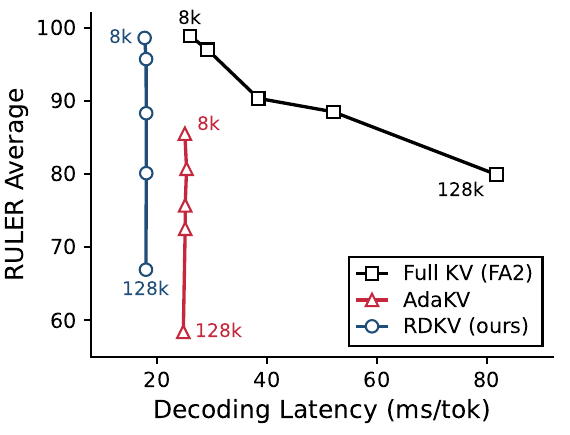}
  \caption{Latency vs.\ Accuracy.}
  \label{fig:eff-tradeoff}
\end{subfigure}

\caption{Decode latency, peak memory, and latency-accuracy trade-off for LLaMA-3.1-8B-Instruct from 8K to 256K context length on a single A100 64\,GB. Both FullKV and AdaKV use FA2.}
\label{fig:efficiency}
\end{figure}

\textbf{Latency.} \cref{fig:efficiency} reports decode latency, peak memory, and the
latency-accuracy trade-off on LLaMA-3.1-8B-Instruct from $8$K to
$256$K on a single A100 64\,GB.
Both FullKV and AdaKV~\citep{feng2024adakv} use FA2~\citep{dao2023flashattention2};
reporting FullKV under SDPA would inflate RDKV's relative speedup to roughly
$13{\times}$. FullKV per-token decode latency grows from $26$\,ms at
$8$K to $82$\,ms at $128$K, while RDKV stays flat at ${\sim}18$\,ms, a
$4.5{\times}$ speedup at $128$K (\cref{fig:eff-latency}), since
TriZone's packed cache size is fixed by $B_\text{total}$ rather than
by $T$. AdaKV is also at ${\sim}25$\,ms under the same token
budget, but RDKV is $1.4{\times}$ faster.

\textbf{Memory.} FullKV peak memory grows linearly to $58.1$\,GB at
$128$K (\cref{fig:eff-memory}) and OOMs at $256$K.
RDKV uses $30.5$\,GB at $128$K ($1.9{\times}$ reduction vs.\
FullKV) and $44.5$\,GB at $256$K, running on the same device where
FullKV cannot.
 
\textbf{Latency-Accuracy Trade-off.} \cref{fig:eff-tradeoff} plots latency against RULER accuracy at each
context length.
RDKV and AdaKV operate at comparable latency, but RDKV is more
accurate at every context length: at $128$K the gap is $8.5$ points
($67.0$ vs.\ $58.4$).
RDKV thus sits strictly above AdaKV on the latency--accuracy Pareto
front.

\subsection{Ablation Studies}
\label{sec:ablation}

\begin{wraptable}{r}{0.3\textwidth}
\centering
\small
\vspace{-12pt}
\caption{Action-space ablation on LongBench average.}
\label{tab:ablation-action}
\vspace{2pt}
\setlength{\tabcolsep}{3pt}
\begin{tabular}{lcc}
\toprule
Config.\ & $128L$\ & $512L$ \\
\midrule
Quant-only    & 39.06 & 40.59 \\
Evict-only       & 	42.53 & 47.18 \\
Tri-state      & \underline{46.91} & \underline{49.05} \\
Joint (RDKV) & \textbf{47.43} & \textbf{49.22} \\
\bottomrule
\end{tabular}
\vspace{-8pt}
\end{wraptable}

\paragraph{Action Space.}
We compare four bit-width sets under the same distortion weights and
water-filling allocator (\cref{tab:ablation-action},
LLaMA-3.1-8B-Instruct, $B_\text{total} = 128L$ and $512L$).
Removing eviction (Quant-only) costs $8.4$ / $8.6$ points: the allocator
must spend bits on every token, diluting precision for the critical
few.
Removing quantization (Evict-only) costs $4.9$ / $2$ points:
moderate-importance tokens are discarded entirely when a 2- or 4-bit
copy would suffice.
Tri-state ($\{0,4,16\}$) closes most of the
gap ($0.5$ / $0.2$), confirming that \emph{mixing} eviction with
quantization matters more than bit-width granularity.

\begin{wraptable}{r}{0.32\textwidth}
\centering
\small
\vspace{-12pt}
\caption{TriZone ablation: decode latency (ms/tok).}
\label{tab:trizone-ablation}
\vspace{2pt}
\setlength{\tabcolsep}{3pt}
\begin{tabular}{lcc}
\toprule
Config.\ & $16$K & $64$K \\
\midrule
RDKV w/ TriZone  & \textbf{17.99} & \textbf{18.02} \\
RDKV w/o TriZone & 36.05 & 36.47 \\
\bottomrule
\end{tabular}
\vspace{-9pt}
\end{wraptable}

\paragraph{TriZone Packed Decode.}
Without TriZone (\cref{sec:efficient_decode}), the mixed-bit cache
must be unpacked to FP16 before FA2, reading the same HBM as
full-precision.
\Cref{tab:trizone-ablation} measured on LLaMA-3.1-8B-Instruct shows
the effect: without TriZone, per-token latency is ${\sim}36$\,ms at
both $16$K and $64$K; with TriZone, it drops to ${\sim}18$\,ms, a
$2{\times}$ reduction with bit-identical output.
Combined with the smaller cache footprint from mixed-bit allocation,
this gives the $4.5{\times}$ decode speedup over FullKV reported in
\cref{fig:eff-latency}.

Additional ablation experiments are in~\cref{app:ablation-extension}.

\section{Conclusion}
\label{sec:conclusion}
We formulate KV cache compression as a rate-distortion problem.
The reverse water-filling solution allocates each cache unit a
bit-width from $\mathcal{B} = \{0, 2, 4, 8, 16\}$.
Eviction emerges as the zero-rate boundary of the same curve as
quantization, so sparsification and quantization are selected from a
single optimization rather than composed under pre-allocated ratios.
A packed-decode layout fuses dequantization into attention, turning
the mixed-bit allocation into actual HBM savings. Across LongBench,
RULER, and InfiniteBench on five open-source LLMs spanning different
architectures and scales, RDKV consistently
outperforms the evaluated baselines in accuracy, decoding speed, and
peak memory at long context, and reaches $256$K on a single A100
$64$\,GB where Full KV does not.
Our results show that treating eviction and quantization jointly can
improve accuracy at no cost to efficiency.
A natural extension is streaming
re-budgeting: re-running the allocation during decoding so that the
bit assignment adapts to the growing cache instead of being frozen
after prefill.

\begin{ack}
This work was supported by a grant from the Swiss National Supercomputing Centre (CSCS) under project ID lp160 on Alps.
\end{ack}






\bibliographystyle{unsrtnat}
\bibliography{references}

\begin{thebibliography}{51}
\providecommand{\natexlab}[1]{#1}
\providecommand{\url}[1]{\texttt{#1}}
\expandafter\ifx\csname urlstyle\endcsname\relax
  \providecommand{\doi}[1]{doi: #1}\else
  \providecommand{\doi}{doi: \begingroup \urlstyle{rm}\Url}\fi

\bibitem[Brown et~al.(2020)Brown, Mann, Ryder, Subbiah, Kaplan, Dhariwal,
  Neelakantan, Shyam, Sastry, Askell, et~al.]{brown2020language}
Tom Brown, Benjamin Mann, Nick Ryder, Melanie Subbiah, Jared~D Kaplan, Prafulla
  Dhariwal, Arvind Neelakantan, Pranav Shyam, Girish Sastry, Amanda Askell,
  et~al.
\newblock Language models are few-shot learners.
\newblock \emph{Advances in Neural Information Processing Systems},
  33:\penalty0 1877--1901, 2020.

\bibitem[Achiam et~al.(2023)Achiam, Adler, Agarwal, Ahmad, Akkaya, Aleman,
  Almeida, Altenschmidt, Altman, Anadkat, et~al.]{achiam2023gpt}
Josh Achiam, Steven Adler, Sandhini Agarwal, Lama Ahmad, Ilge Akkaya,
  Florencia~Leoni Aleman, Diogo Almeida, Janko Altenschmidt, Sam Altman,
  Shyamal Anadkat, et~al.
\newblock {GPT}-4 technical report.
\newblock \emph{arXiv preprint arXiv:2303.08774}, 2023.

\bibitem[Liu et~al.(2024{\natexlab{a}})Liu, Feng, Xue, Wang, Wu, Lu, Zhao,
  Deng, Zhang, Ruan, et~al.]{liu2024deepseek}
Aixin Liu, Bei Feng, Bing Xue, Bingxuan Wang, Bochao Wu, Chengda Lu, Chenggang
  Zhao, Chengqi Deng, Chenyu Zhang, Chong Ruan, et~al.
\newblock {DeepSeek}-v3 technical report.
\newblock \emph{arXiv preprint arXiv:2412.19437}, 2024{\natexlab{a}}.

\bibitem[Lewis et~al.(2020)Lewis, Perez, Piktus, Petroni, Karpukhin, Goyal,
  K{\"u}ttler, Lewis, Yih, Rockt{\"a}schel, et~al.]{lewis2020retrieval}
Patrick Lewis, Ethan Perez, Aleksandra Piktus, Fabio Petroni, Vladimir
  Karpukhin, Naman Goyal, Heinrich K{\"u}ttler, Mike Lewis, Wen-tau Yih, Tim
  Rockt{\"a}schel, et~al.
\newblock Retrieval-augmented generation for knowledge-intensive {NLP} tasks.
\newblock \emph{Advances in Neural Information Processing Systems},
  33:\penalty0 9459--9474, 2020.

\bibitem[Kwon et~al.(2023)Kwon, Li, Zhuang, Sheng, Zheng, Yu, Gonzalez, Zhang,
  and Stoica]{kwon2023efficient}
Woosuk Kwon, Zhuohan Li, Siyuan Zhuang, Ying Sheng, Lianmin Zheng, Cody~Hao Yu,
  Joseph Gonzalez, Hao Zhang, and Ion Stoica.
\newblock Efficient memory management for large language model serving with
  {PagedAttention}.
\newblock In \emph{Proceedings of the 29th Symposium on Operating Systems
  Principles}, pages 611--626, 2023.

\bibitem[Dao et~al.(2022)Dao, Fu, Ermon, Rudra, and
  R{\'e}]{dao2022flashattention}
Tri Dao, Dan Fu, Stefano Ermon, Atri Rudra, and Christopher R{\'e}.
\newblock {FlashAttention}: Fast and memory-efficient exact attention with
  {IO}-awareness.
\newblock \emph{Advances in Neural Information Processing Systems},
  35:\penalty0 16344--16359, 2022.

\bibitem[Pope et~al.(2023)Pope, Douglas, Chowdhery, Devlin, Bradbury, Heek,
  Xiao, Agrawal, and Dean]{pope2023efficiently}
Reiner Pope, Sholto Douglas, Aakanksha Chowdhery, Jacob Devlin, James Bradbury,
  Jonathan Heek, Kefan Xiao, Shivani Agrawal, and Jeff Dean.
\newblock Efficiently scaling transformer inference.
\newblock \emph{Proceedings of Machine Learning and Systems}, 5:\penalty0
  606--624, 2023.

\bibitem[Li et~al.(2024{\natexlab{a}})Li, Li, Tian, Tang, Xu, Chen, Hu, Dong,
  Li, and Chen]{li2024survey}
Haoyang Li, Yiming Li, Anxin Tian, Tianhao Tang, Zhanchao Xu, Xuejia Chen,
  Nicole Hu, Wei Dong, Qing Li, and Lei Chen.
\newblock A survey on large language model acceleration based on {KV} cache
  management.
\newblock \emph{arXiv preprint arXiv:2412.19442}, 2024{\natexlab{a}}.

\bibitem[Liu et~al.(2025)Liu, Fu, Liu, Zou, Zhang, and Zhou]{liu2025kv}
Yanyu Liu, Jingying Fu, Sixiang Liu, Yitian Zou, Shouhua Zhang, and Jiehan
  Zhou.
\newblock {KV} cache compression for inference efficiency in {LLMs}: A review.
\newblock In \emph{Proceedings of the 4th International Conference on
  Artificial Intelligence and Intelligent Information Processing}, pages
  207--212, 2025.

\bibitem[Zhang et~al.(2023)Zhang, Sheng, Zhou, Chen, Zheng, Cai, Song, Tian,
  R{\'e}, Barrett, et~al.]{h2o}
Zhenyu Zhang, Ying Sheng, Tianyi Zhou, Tianlong Chen, Lianmin Zheng, Ruisi Cai,
  Zhao Song, Yuandong Tian, Christopher R{\'e}, Clark Barrett, et~al.
\newblock {H2O}: Heavy-hitter oracle for efficient generative inference of
  large language models.
\newblock \emph{Advances in Neural Information Processing Systems},
  36:\penalty0 34661--34710, 2023.

\bibitem[Li et~al.(2024{\natexlab{b}})Li, Huang, Yang, Venkitesh, Locatelli,
  Ye, Cai, Lewis, and Chen]{li2024snapkv}
Yuhong Li, Yingbing Huang, Bowen Yang, Bharat Venkitesh, Acyr Locatelli,
  Hanchen Ye, Tianle Cai, Patrick Lewis, and Deming Chen.
\newblock {SnapKV}: {LLM} knows what you are looking for before generation.
\newblock \emph{Advances in Neural Information Processing Systems},
  37:\penalty0 22947--22970, 2024{\natexlab{b}}.

\bibitem[Xu et~al.(2024)Xu, Jie, Dong, Wang, Lu, Zhou, Saha, Xiong, and
  Sahoo]{xu2024think}
Yuhui Xu, Zhanming Jie, Hanze Dong, Lei Wang, Xudong Lu, Aojun Zhou, Amrita
  Saha, Caiming Xiong, and Doyen Sahoo.
\newblock {ThinK}: Thinner key cache by query-driven pruning.
\newblock \emph{arXiv preprint arXiv:2407.21018}, 2024.

\bibitem[Liu et~al.(2024{\natexlab{b}})Liu, Yuan, Jin, Zhong, Xu, Braverman,
  Chen, and Hu]{kivi}
Zirui Liu, Jiayi Yuan, Hongye Jin, Shaochen Zhong, Zhaozhuo Xu, Vladimir
  Braverman, Beidi Chen, and Xia Hu.
\newblock {KIVI}: A tuning-free asymmetric 2bit quantization for {KV} cache.
\newblock \emph{arXiv preprint arXiv:2402.02750}, 2024{\natexlab{b}}.

\bibitem[He et~al.(2024)He, Zhang, Wu, Liu, Zhou, and Zhuang]{he2024zipcache}
Yefei He, Luoming Zhang, Weijia Wu, Jing Liu, Hong Zhou, and Bohan Zhuang.
\newblock {ZipCache}: Accurate and efficient {KV} cache quantization with
  salient token identification.
\newblock \emph{Advances in Neural Information Processing Systems},
  37:\penalty0 68287--68307, 2024.

\bibitem[Zhang et~al.(2024{\natexlab{a}})Zhang, Zhu, Song, Wu, Kuang, Li,
  Shang, Liu, and Li]{qpruningkv}
Jiebin Zhang, Dawei Zhu, Yifan Song, Wenhao Wu, Chuqiao Kuang, Xiaoguang Li,
  Lifeng Shang, Qun Liu, and Sujian Li.
\newblock More tokens, lower precision: Towards the optimal token-precision
  trade-off in {KV} cache compression.
\newblock \emph{arXiv preprint arXiv:2412.12706}, 2024{\natexlab{a}}.

\bibitem[Lei and Ilager(2026)]{arkv}
Jianlong Lei and Shashikant Ilager.
\newblock {ARKV}: Adaptive and resource-efficient {KV} cache management under
  limited memory budget for long-context inference in {LLMs}.
\newblock \emph{arXiv preprint arXiv:2603.08727}, 2026.

\bibitem[Anonymous(2026)]{anonymous2026hqekv}
Anonymous.
\newblock Hqe{KV}: Towards hybrid quantization and eviction for {KV} cache in
  long-context {LLM} inference.
\newblock In \emph{Submitted to ACL Rolling Review - January 2026}, 2026.
\newblock under review.

\bibitem[Hooper et~al.(2024)Hooper, Kim, Mohammadzadeh, Mahoney, Shao, Keutzer,
  and Gholami]{kvquant}
Coleman Hooper, Sehoon Kim, Hiva Mohammadzadeh, Michael~W Mahoney, Yakun~S
  Shao, Kurt Keutzer, and Amir Gholami.
\newblock {KVQuant}: Towards 10 million context length {LLM} inference with
  {KV} cache quantization.
\newblock \emph{Advances in Neural Information Processing Systems},
  37:\penalty0 1270--1303, 2024.

\bibitem[Li et~al.(2025)Li, Xing, Li, Qu, Zhen, Liu, Yao, Pan, and
  Yuan]{kvtuner}
Xing Li, Zeyu Xing, Yiming Li, Linping Qu, Hui-Ling Zhen, Wulong Liu, Yiwu Yao,
  Sinno~Jialin Pan, and Mingxuan Yuan.
\newblock {KVTuner}: Sensitivity-aware layer-wise mixed-precision {KV} cache
  quantization for efficient and nearly lossless {LLM} inference.
\newblock \emph{arXiv preprint arXiv:2502.04420}, 2025.

\bibitem[Dong et~al.(2024)Dong, Cheng, Qin, and Wang]{qaq}
Shichen Dong, Wen Cheng, Jiayu Qin, and Wei Wang.
\newblock {QAQ}: Quality adaptive quantization for {LLM} {KV} cache.
\newblock \emph{arXiv preprint arXiv:2403.04643}, 2024.

\bibitem[Yang et~al.(2024)Yang, Kim, Bae, Kwon, Park, Yang, Kwon, and
  Lee]{mikv}
June~Yong Yang, Byeongwook Kim, Jeongin Bae, Beomseok Kwon, Gunho Park, Eunho
  Yang, Se~Jung Kwon, and Dongsoo Lee.
\newblock No token left behind: Reliable kv cache compression via
  importance-aware mixed precision quantization.
\newblock \emph{arXiv preprint arXiv:2402.18096}, 2024.

\bibitem[Shutova et~al.(2025)Shutova, Malinovskii, Egiazarian, Kuznedelev,
  Mazur, Surkov, Ermakov, and Alistarh]{aquakv}
Alina Shutova, Vladimir Malinovskii, Vage Egiazarian, Denis Kuznedelev, Denis
  Mazur, Nikita Surkov, Ivan Ermakov, and Dan Alistarh.
\newblock Cache me if you must: Adaptive key-value quantization for large
  language models.
\newblock \emph{arXiv preprint arXiv:2501.19392}, 2025.

\bibitem[Xiao et~al.(2023)Xiao, Tian, Chen, Han, and Lewis]{streamingllm}
Guangxuan Xiao, Yuandong Tian, Beidi Chen, Song Han, and Mike Lewis.
\newblock Efficient streaming language models with attention sinks.
\newblock \emph{arXiv preprint arXiv:2309.17453}, 2023.

\bibitem[Liu et~al.(2023)Liu, Desai, Liao, Wang, Xie, Xu, Kyrillidis, and
  Shrivastava]{scissorhands}
Zichang Liu, Aditya Desai, Fangshuo Liao, Weitao Wang, Victor Xie, Zhaozhuo Xu,
  Anastasios Kyrillidis, and Anshumali Shrivastava.
\newblock Scissorhands: Exploiting the persistence of importance hypothesis for
  llm kv cache compression at test time.
\newblock \emph{Advances in Neural Information Processing Systems},
  36:\penalty0 52342--52364, 2023.

\bibitem[Chen et~al.(2024)Chen, Wang, Shang, Cui, Zhang, Liu, Wang, Sun, Yu,
  and Wu]{nacl}
Yilong Chen, Guoxia Wang, Junyuan Shang, Shiyao Cui, Zhenyu Zhang, Tingwen Liu,
  Shuohuan Wang, Yu~Sun, Dianhai Yu, and Hua Wu.
\newblock Nacl: A general and effective kv cache eviction framework for llm at
  inference time.
\newblock In \emph{Proceedings of the 62nd Annual Meeting of the Association
  for Computational Linguistics (Volume 1: Long Papers)}, pages 7913--7926,
  2024.

\bibitem[Feng et~al.(2024)Feng, Lv, Cao, Xie, and Zhou]{feng2024adakv}
Yuan Feng, Junlin Lv, Yukun Cao, Xike Xie, and S~Kevin Zhou.
\newblock {Ada-KV}: Optimizing {KV} cache eviction by adaptive budget
  allocation for efficient {LLM} inference.
\newblock \emph{arXiv preprint arXiv:2407.11550}, 2024.

\bibitem[Cai et~al.(2024)Cai, Zhang, Gao, Liu, Li, Liu, Lu, Xiong, Dong, Hu,
  et~al.]{cai2024pyramidkv}
Zefan Cai, Yichi Zhang, Bofei Gao, Yuliang Liu, Yucheng Li, Tianyu Liu, Keming
  Lu, Wayne Xiong, Yue Dong, Junjie Hu, et~al.
\newblock {PyramidKV}: Dynamic {KV} cache compression based on pyramidal
  information funneling.
\newblock \emph{arXiv preprint arXiv:2406.02069}, 2024.

\bibitem[Qin et~al.(2025)Qin, Cao, Lin, Hu, Fan, Cheng, Lin, and Li]{cake}
Ziran Qin, Yuchen Cao, Mingbao Lin, Wen Hu, Shixuan Fan, Ke~Cheng, Weiyao Lin,
  and Jianguo Li.
\newblock Cake: Cascading and adaptive kv cache eviction with layer
  preferences.
\newblock \emph{arXiv preprint arXiv:2503.12491}, 2025.

\bibitem[Devoto et~al.(2025)Devoto, Jeblick, and J{\'e}gou]{kvpress}
Alessio Devoto, Maximilian Jeblick, and Simon J{\'e}gou.
\newblock Expected attention: {KV} cache compression by estimating attention
  from future queries distribution.
\newblock \emph{arXiv preprint arXiv:2510.00636}, 2025.

\bibitem[Feng et~al.(2025)Feng, Lv, Cao, Xie, and Zhou]{criticalkv}
Yuan Feng, Junlin Lv, Yukun Cao, Xike Xie, and S~Kevin Zhou.
\newblock Identify critical {KV} cache in {LLM} inference from an output
  perturbation perspective.
\newblock \emph{arXiv preprint arXiv:2502.03805}, 2025.

\bibitem[Gu et~al.(2025)Gu, Liang, Zhao, and Diao]{obcache}
Yuzhe Gu, Xiyu Liang, Jiaojiao Zhao, and Enmao Diao.
\newblock {OBCache}: Optimal brain {KV} cache pruning for efficient
  long-context {LLM} inference.
\newblock \emph{arXiv preprint arXiv:2510.07651}, 2025.

\bibitem[LeCun et~al.(1989)LeCun, Denker, and Solla]{lecun1989optimal}
Yann LeCun, John Denker, and Sara Solla.
\newblock Optimal brain damage.
\newblock \emph{Advances in Neural Information Processing Systems}, 2, 1989.

\bibitem[Guo et~al.(2025)Guo, Li, and Benini]{guo2025optimal}
Hang Guo, Yawei Li, and Luca Benini.
\newblock Optimal brain restoration for joint quantization and sparsification
  of llms.
\newblock \emph{arXiv preprint arXiv:2509.11177}, 2025.

\bibitem[Goel et~al.(2025)Goel, Park, Gagrani, Jones, Morse, Langston, Lee, and
  Lott]{goel2025caote}
Raghavv Goel, Junyoung Park, Mukul Gagrani, Dalton Jones, Matthew Morse, Harper
  Langston, Mingu Lee, and Chris Lott.
\newblock Caote: Kv cache selection for {LLMs} via attention output error-based
  token eviction.
\newblock \emph{arXiv preprint arXiv:2504.14051}, 2025.

\bibitem[Geng et~al.(2025)Geng, Wang, Liu, Ju, Li, Li, Yuan, Hao, Lian, Chen,
  et~al.]{andpro}
Zijie Geng, Jie Wang, Ziqi Liu, Feng Ju, Yiming Li, Xing Li, Mingxuan Yuan,
  Jianye Hao, Defu Lian, Enhong Chen, et~al.
\newblock Accurate kv cache eviction via anchor direction projection for
  efficient llm inference.
\newblock In \emph{The Thirty-ninth Annual Conference on Neural Information
  Processing Systems}, 2025.

\bibitem[An et~al.(2026)An, Lu, Zhu, Yu, Zhao, Wu, Tang, and Wang]{restkv}
Yongqi An, Chang Lu, Kuan Zhu, Tao Yu, Chaoyang Zhao, Hong Wu, Ming Tang, and
  Jinqiao Wang.
\newblock {ReST-KV}: Robust {KV} cache eviction with layer-wise output
  reconstruction and spatial-temporal smoothing.
\newblock In \emph{The Fourteenth International Conference on Learning
  Representations}, 2026.

\bibitem[Vaswani et~al.(2017)Vaswani, Shazeer, Parmar, Uszkoreit, Jones, Gomez,
  Kaiser, and Polosukhin]{vaswani2017attention}
Ashish Vaswani, Noam Shazeer, Niki Parmar, Jakob Uszkoreit, Llion Jones,
  Aidan~N Gomez, {\L}ukasz Kaiser, and Illia Polosukhin.
\newblock Attention is all you need.
\newblock \emph{Advances in Neural Information Processing Systems}, 30, 2017.

\bibitem[Ainslie et~al.(2023)Ainslie, Lee-Thorp, De~Jong, Zemlyanskiy,
  Lebr{\'o}n, and Sanghai]{ainslie2023gqa}
Joshua Ainslie, James Lee-Thorp, Michiel De~Jong, Yury Zemlyanskiy, Federico
  Lebr{\'o}n, and Sumit Sanghai.
\newblock {GQA}: Training generalized multi-query transformer models from
  multi-head checkpoints.
\newblock In \emph{Proceedings of the 2023 Conference on Empirical Methods in
  Natural Language Processing}, pages 4895--4901, 2023.

\bibitem[Bennett(1948)]{bennett1948spectra}
William~Ralph Bennett.
\newblock Spectra of quantized signals.
\newblock \emph{The Bell System Technical Journal}, 27\penalty0 (3):\penalty0
  446--472, 1948.

\bibitem[Shannon et~al.(1959)]{shannon1959coding}
Claude~E Shannon et~al.
\newblock Coding theorems for a discrete source with a fidelity criterion.
\newblock \emph{IRE Nat. Conv. Rec}, 4\penalty0 (142-163):\penalty0 1, 1959.

\bibitem[Cover and Thomas(2006)]{cover2006elements}
Thomas~M. Cover and Joy~A. Thomas.
\newblock \emph{Elements of Information Theory 2nd Edition (Wiley Series in
  Telecommunications and Signal Processing)}.
\newblock Wiley-Interscience, July 2006.
\newblock ISBN 0471241954.

\bibitem[Grattafiori et~al.(2024)Grattafiori, Dubey, Jauhri, Pandey, Kadian,
  Al-Dahle, Letman, Mathur, Schelten, Vaughan, et~al.]{llama3}
Aaron Grattafiori, Abhimanyu Dubey, Abhinav Jauhri, Abhinav Pandey, Abhishek
  Kadian, Ahmad Al-Dahle, Aiesha Letman, Akhil Mathur, Alan Schelten, Alex
  Vaughan, et~al.
\newblock The {Llama} 3 herd of models.
\newblock \emph{arXiv preprint arXiv:2407.21783}, 2024.

\bibitem[Jiang et~al.(2023)Jiang, Sablayrolles, Mensch, Bamford, Chaplot,
  de~las Casas, Bressand, Lengyel, Lample, Saulnier, Lavaud, Lachaux, Stock,
  Scao, Lavril, Wang, Lacroix, and Sayed]{mistral}
Albert~Q. Jiang, Alexandre Sablayrolles, Arthur Mensch, Chris Bamford,
  Devendra~Singh Chaplot, Diego de~las Casas, Florian Bressand, Gianna Lengyel,
  Guillaume Lample, Lucile Saulnier, Lélio~Renard Lavaud, Marie-Anne Lachaux,
  Pierre Stock, Teven~Le Scao, Thibaut Lavril, Thomas Wang, Timothée Lacroix,
  and William~El Sayed.
\newblock Mistral 7{B}, 2023.

\bibitem[Yang et~al.(2025)Yang, Li, Yang, Zhang, Hui, Zheng, Yu, Gao, Huang,
  Lv, et~al.]{qwen3}
An~Yang, Anfeng Li, Baosong Yang, Beichen Zhang, Binyuan Hui, Bo~Zheng, Bowen
  Yu, Chang Gao, Chengen Huang, Chenxu Lv, et~al.
\newblock Qwen3 technical report.
\newblock \emph{arXiv preprint arXiv:2505.09388}, 2025.

\bibitem[Touvron et~al.(2023)Touvron, Martin, Stone, Albert, Almahairi, Babaei,
  Bashlykov, Batra, Bhargava, Bhosale, et~al.]{touvron2023llama}
Hugo Touvron, Louis Martin, Kevin Stone, Peter Albert, Amjad Almahairi, Yasmine
  Babaei, Nikolay Bashlykov, Soumya Batra, Prajjwal Bhargava, Shruti Bhosale,
  et~al.
\newblock {Llama} 2: Open foundation and fine-tuned chat models.
\newblock \emph{arXiv preprint arXiv:2307.09288}, 2023.

\bibitem[Qwen et~al.(2025)Qwen, :, Yang, Yang, Zhang, Hui, Zheng, Yu, Li, Liu,
  Huang, Wei, Lin, Yang, Tu, Zhang, Yang, Yang, Zhou, Lin, Dang, Lu, Bao, Yang,
  Yu, Li, Xue, Zhang, Zhu, Men, Lin, Li, Tang, Xia, Ren, Ren, Fan, Su, Zhang,
  Wan, Liu, Cui, Zhang, and Qiu]{qwen25}
Qwen, :, An~Yang, Baosong Yang, Beichen Zhang, Binyuan Hui, Bo~Zheng, Bowen Yu,
  Chengyuan Li, Dayiheng Liu, Fei Huang, Haoran Wei, Huan Lin, Jian Yang,
  Jianhong Tu, Jianwei Zhang, Jianxin Yang, Jiaxi Yang, Jingren Zhou, Junyang
  Lin, Kai Dang, Keming Lu, Keqin Bao, Kexin Yang, Le~Yu, Mei Li, Mingfeng Xue,
  Pei Zhang, Qin Zhu, Rui Men, Runji Lin, Tianhao Li, Tianyi Tang, Tingyu Xia,
  Xingzhang Ren, Xuancheng Ren, Yang Fan, Yang Su, Yichang Zhang, Yu~Wan,
  Yuqiong Liu, Zeyu Cui, Zhenru Zhang, and Zihan Qiu.
\newblock Qwen2.5 technical report, 2025.

\bibitem[Bai et~al.(2024)Bai, Lv, Zhang, Lyu, Tang, Huang, Du, Liu, Zeng, Hou,
  et~al.]{bai2023longbench}
Yushi Bai, Xin Lv, Jiajie Zhang, Hongchang Lyu, Jiankai Tang, Zhidian Huang,
  Zhengxiao Du, Xiao Liu, Aohan Zeng, Lei Hou, et~al.
\newblock {LongBench}: A bilingual, multitask benchmark for long context
  understanding.
\newblock In \emph{Proceedings of the 62nd Annual Meeting of the Association
  for Computational Linguistics (Volume 1: Long Papers)}, pages 3119--3137,
  2024.

\bibitem[Kamradt(2023)]{kamradt2023niah}
Gregory Kamradt.
\newblock Needle in a haystack - pressure testing {LLMs}, 2023.

\bibitem[Hsieh et~al.(2024)Hsieh, Sun, Kriman, Acharya, Rekesh, Jia, Zhang, and
  Ginsburg]{hsieh2024ruler}
Cheng-Ping Hsieh, Simeng Sun, Samuel Kriman, Shantanu Acharya, Dima Rekesh, Fei
  Jia, Yang Zhang, and Boris Ginsburg.
\newblock {RULER}: What's the real context size of your long-context language
  models?
\newblock \emph{arXiv preprint arXiv:2404.06654}, 2024.

\bibitem[Zhang et~al.(2024{\natexlab{b}})Zhang, Chen, Hu, Xu, Chen, Hao, Han,
  Thai, Wang, Liu, et~al.]{zhang2024infbench}
Xinrong Zhang, Yingfa Chen, Shengding Hu, Zihang Xu, Junhao Chen, Moo Hao,
  Xu~Han, Zhen Thai, Shuo Wang, Zhiyuan Liu, et~al.
\newblock {$\infty$} bench: Extending long context evaluation beyond {100K}
  tokens.
\newblock In \emph{Proceedings of the 62nd Annual Meeting of the Association
  for Computational Linguistics (Volume 1: Long Papers)}, pages 15262--15277,
  2024{\natexlab{b}}.

\bibitem[Dao(2023)]{dao2023flashattention2}
Tri Dao.
\newblock {FlashAttention}-2: Faster attention with better parallelism and work
  partitioning.
\newblock \emph{arXiv preprint arXiv:2307.08691}, 2023.

\end{thebibliography}

\newpage
\appendix
\section{Proofs}\label{app:proofs}

This appendix collects the proofs for the results in
\cref{sec:rd_formulation}.
All quantities below are local to a single layer and KV head
$(\ell, h)$, with indices suppressed.

\textbf{Token weight in V cache.}

\begin{proof}[Proof of \cref{thm:v-score}]
By definition,
$\|a_\tau - \hat{a}_\tau\|_{\mathrm{TV}}
  = \frac{1}{2}\sum_{t'} |a_{\tau,t'} - \hat{a}_{\tau,t'}|$.
The evicted entry contributes $|a_{\tau,t} - 0| = a_{\tau,t}$.
Each surviving entry $t' \neq t$ shifts by
$|\hat{a}_{\tau,t'} - a_{\tau,t'}|
  = a_{\tau,t'} \cdot a_{\tau,t}/(1-a_{\tau,t})$.
Summing the surviving terms,
\[
\sum_{t' \neq t}
  a_{\tau,t'} \cdot \frac{a_{\tau,t}}{1-a_{\tau,t}}
  = (1-a_{\tau,t}) \cdot \frac{a_{\tau,t}}{1-a_{\tau,t}}
  = a_{\tau,t}.
\]
Therefore
$\|a_\tau - \hat{a}_\tau\|_{\mathrm{TV}}
  = \frac{1}{2}(a_{\tau,t} + a_{\tau,t}) = a_{\tau,t}$.
Summing over queries,
$w_t = \sum_\tau a_{\tau,t}$.
\end{proof}

\textbf{Channel weight in K cache.}

\begin{proof}[Proof of \cref{thm:k-score}]
Removing channel~$c$ sets $K[:,c] = 0$.
Since each logit entry is
$Z_{\tau,t} = q_\tau^\top k_t / \sqrt{d}
  = \sum_{c'} q_{\tau,c'}\,k_{t,c'}/\sqrt{d}$,
the logit deviation is
$\delta Z = -(1/\sqrt{d})\,Q[:,c]\,K[:,c]^\top$.
This is a rank-one matrix $uv^\top$ with
$u = Q[:,c]/\sqrt{d}$ and $v = K[:,c]$.
A rank-one matrix has a single nonzero singular value
$\|u\|_2\|v\|_2$, so
$\|\delta Z\|_2 = \|Q[:,c]\|_2\,\|K[:,c]\|_2 / \sqrt{d} = w_c$.
\end{proof}

\begin{remark}
The same expression for $w_c$ obtains under any unitarily invariant
norm of $\delta Z$ (spectral, Frobenius, nuclear, Schatten-$p$):
a rank-one matrix has a single nonzero singular value, so all such
norms reduce to $\|u\|_2\|v\|_2$.
In particular, ThinK's Frobenius-based derivation and our
spectral-norm derivation yield the same weight.
\end{remark}

\textbf{Bennett high-rate distortion.}

\begin{lemma}[Bennett high-rate distortion]\label[lemma]{lem:bennett}
For a uniform scalar quantizer with per-unit dynamic range $R_u$ at
bit-width $b$ (high-rate regime), the per-coordinate root-mean-square
error is $\sigma_u\,2^{-b} + o(2^{-b})$ with
$\sigma_u := R_u/(2\sqrt{3})$.
\end{lemma}

\begin{proof}
With per-unit scale, all $n_u$ coordinates within a unit share a
single uniform quantizer of cell width $\Delta = R_u/2^b$.
Bennett's high-rate uniform-cell
approximation~\citep{bennett1948spectra} gives the per-coordinate
mean-squared error as
$\Delta^2/12 + o(\Delta^2) = \sigma_u^2\,2^{-2b} + o(2^{-2b})$.
Taking the square root yields the per-coordinate RMSE
$\sigma_u\,2^{-b} + o(2^{-b})$.
\end{proof}

\textbf{Optimal allocation.}

\begin{proof}[Proof of \cref{thm:water-filling}]
The objective $\sum_u w_u\sigma_u\,2^{-b_u}$ is strictly decreasing
in each $b_u$, so the budget is active at any optimum:
$\sum_u b_u^\star = B$.
Each summand $w_u\sigma_u\,2^{-b_u}$ is convex in $b_u$ (second
derivative $(\ln 2)^2 w_u\sigma_u\,2^{-b_u} > 0$), so the
objective is convex and KKT is sufficient for global optimality.
The Lagrangian, with $\lambda \geq 0$ for the budget and
$\mu_u \geq 0$ for nonnegativity, is
\[
\mathcal{L}(\{b_u\}, \lambda, \{\mu_u\})
  = \sum_u w_u\sigma_u\,2^{-b_u}
  + \lambda\!\Big(\sum_u b_u - B\Big)
  - \sum_u \mu_u\,b_u.
\]
Stationarity in $b_u$ gives
$-(\ln 2)\,w_u\sigma_u\,2^{-b_u} + \lambda - \mu_u = 0$.
For an interior solution ($\mu_u = 0$ by complementary slackness),
\[
2^{-b_u}
  = \frac{\lambda}{\ln 2 \cdot w_u\sigma_u},
  \qquad \text{i.e.,} \qquad
  b_u = \log_2\!\left(
    \frac{\ln 2 \cdot w_u\sigma_u}{\lambda}\right).
\]
When this interior value is negative, the boundary
$b_u^\star = 0$ activates instead, which occurs iff
$w_u\sigma_u < \lambda/\ln 2$.
The two cases combine into the $[\cdot]_+$ projection.
Since $\sum_u b_u^\star(\lambda)$ is monotone non-increasing in
$\lambda$, the budget equation
$\sum_u b_u^\star(\lambda) = B$ is solvable by 1D search.
\end{proof}

\textbf{Weak duality for discrete allocation.}

\begin{proposition}[Weak duality]\label[proposition]{prop:weak-duality}
For
$\mathrm{OPT} := \min_{b_u \in \mathcal{B}}
  \sum_u w_u\,\varepsilon_u(b_u)$
s.t.\ $\sum_u b_u \leq B$
with
$\mathcal{B} = \{0, 2, 4, 8, 16\}$
and
$\varepsilon_u(b)$
the calibration-measured per-coordinate distortion at bit-width~$b$,
let
$g(\lambda) := \sum_u \min_{b \in \mathcal{B}}
  [w_u\,\varepsilon_u(b) + \lambda b] - \lambda B$.
Then $g(\lambda) \leq \mathrm{OPT}$ for every $\lambda \geq 0$;
when $\{b_u^\star(\lambda)\}$ is primal feasible,
$\sum_u w_u\varepsilon_u(b_u^\star(\lambda))$ is an upper bound.
\end{proposition}

\begin{proof}
For any feasible $\{b_u\}$ and $\lambda \geq 0$,
\[
\sum_u w_u\varepsilon_u(b_u)
  \;\geq\;
  \sum_u w_u\varepsilon_u(b_u)
    + \lambda\!\Big(\sum_u b_u - B\Big)
  \;\geq\;
  \sum_u \min_{b \in \mathcal{B}}\!
    [w_u\varepsilon_u(b) + \lambda b] - \lambda B
  = g(\lambda),
\]
using $\sum_u b_u - B \leq 0$ in the first inequality and per-unit
minimization in the second.
Taking minimum over feasible primals yields
$\mathrm{OPT} \geq g(\lambda)$.
Conversely, any feasible $\{b_u^\star(\lambda)\}$ satisfies
$\sum_u w_u\varepsilon_u(b_u^\star(\lambda)) \geq \mathrm{OPT}$
by definition.
\end{proof}


\section{Detailed Implementation}
\label{app:impl-details}

This appendix expands on the implementation details abbreviated in \cref{sec:experiments}.

\textbf{Prefill stage.}
Probe-based score computation, rate-distortion bit allocation, and TriZone packing are
performed once at the end of prefill, on the same hidden states already produced by the
forward pass. The probe shares the SnapKV observation window of size $S_w = 32$ and
pooling kernel $w = 5$ across all eviction-style baselines (SnapKV, AdaKV, ThinK,
SnapKV+ZipCache) to ensure that score-based comparisons isolate the indicator and the
allocation, not the probe geometry. The bit-width set is fixed to
$\mathcal{B} = \{0, 2, 4, 8, 16\}$. The K/V budget is split equally,
$B^V = B^K = \tfrac{1}{2}B_\text{head}$, as a symmetric default validated by the
ablation in \cref{app:ablation-kv-split}, which shows that $r_K = 0.5$
consistently achieves the highest average score across budget levels.

\textbf{Calibration.}
The empirical per-coordinate distortion table $\varepsilon_u(b)$ is estimated on
calibration sequences drawn from the prefill prefixes of the LongBench tasks. We use
$32$ sequences (truncated to $4$k tokens each) to estimate $\varepsilon_u(b)$ for each
$b \in \{2, 4, 8\}$ at the per-token (V) and per-channel (K) granularities. The estimate
at $b = 0$ is computed from the same sample but refers to outright removal rather than
quantization, as discussed in \cref{sec:methodology}. Calibration is performed once per
$(\ell, h)$ slice and cached to disk.

\textbf{Lagrangian search.}
The MCKP (Prop. A.3) is solved by one-dimensional bisection on $\lambda$. At each step we evaluate the per-element minimizer $b_u^\star(\lambda) = \arg\min_{b\in\mathcal B} w_u\varepsilon_u(b) + \lambda b$ via table lookup, and bisect on $\lambda$ until the average bit budget $\tfrac{1}{|U|}\sum_u b_u^\star(\lambda)$ matches the target $B$ within a relative tolerance of $10^{-2}$, or after at most $64$ bisection steps. The resulting allocation is treated as the final ${b_u^\star}$.

\textbf{Hardware.}
All models up to 8B parameters (LLaMA-3.1-8B-Instruct, Mistral-7B-Instruct-v0.3,
Qwen3-4B) run on a single NVIDIA A100 64\,GB GPU.
LLaMA-2-13B-Chat and Qwen2.5-72B-Instruct run on NVIDIA GH200 96\,GB GPUs; Qwen2.5-72B uses 2-GPU tensor parallelism.
Decoding latency and peak memory measurements (\cref{sec:experiments},
\cref{fig:efficiency}) are collected on the A100 64\,GB.

\textbf{Baseline implementations.}
SnapKV and AdaKV use the official \texttt{kvpress} implementations.
ThinK~\citep{xu2024think} and ZipCache~\citep{he2024zipcache} use their respective official open-source code.
SnapKV+ZipCache combines the \texttt{kvpress} SnapKV with the official ZipCache mixed-precision quantizer, with no code modifications.
All baselines use the default hyperparameters recommended by their respective authors. 
All methods receive the same prefill probe ($S_w = 32$, $w = 5$) so that any difference in the resulting weights is attributable to the indicator, not the calibration window.

\textbf{Baseline selection rationale.}
For the eviction+quantization slot we compose SnapKV~\citep{li2024snapkv} with ZipCache~\citep{he2024zipcache} rather than adopting QPruningKV~\citep{qpruningkv} directly.
QPruningKV focuses on discussing the combination of existing eviction and quantization methods rather than proposing a new joint algorithm.
In addition, the quantization methods it adopts assign precision only at layer level~\citep{kivi}, a coarse granularity consistently outperformed by token-level alternatives such as ZipCache~\citep{he2024zipcache}.

\textbf{TriZone packed layout.}
Within Zone~A (\cref{sec:efficient_decode}), V rows are sorted by~$b_t^V$
into three contiguous sub-segments, each with a distinct byte-packing
convention over the head dimension~$d$: 2-bit uses \emph{quarter-split}
packing (four channels per byte at bit offsets $0,2,4,6$; $d/4$ packed
bytes per row), 4-bit uses \emph{half-split} packing (two channels per
byte at offsets $0,4$; $d/2$ bytes per row), and 8-bit stores one channel
per byte directly. Each sub-segment admits a single shift-and-mask
dequantization path with no per-element branching.
K channels within each retained row follow the per-channel
allocation~$b_c^K$: channels are sorted by bit-width into three analogous
2/4/8-bit segments along the channel axis, and the query~$q_\tau$ is
permuted to the same sorted order once at the end of prefill.

\textbf{K-cache dequantization fusion.}
Per-channel quantization encodes channel~$c$ of the retained K cache as
$\hat{k}_{t,c} = s_c\,(\tilde{k}_{t,c} - z_c)$, where $s_c, z_c$ are
the per-channel scale and zero-point and $\tilde{k}_{t,c} \in
\{0,\ldots,2^{b_c^K}\!-\!1\}$ is the raw integer code. The kernel avoids
materializing a dequantized FP16 tile by rewriting the dot product as
\[
  q_\tau^\top \hat{k}_t
  \;=\;
  \textstyle\sum_c (s_c\,q_{\tau,c})\;\tilde{k}_{t,c}
  \;-\;\sum_c s_c\,z_c\,q_{\tau,c}\,.
\]
The scaled query $\tilde{q}_c = s_c\,q_{\tau,c}$ is computed once per
decode step and cast to FP16 for the tensor-core matrix multiply.
Raw codes~$\tilde{k}_{t,c}$ are cast exactly to FP16 (all values
${\leq}\,255$) and serve as the right operand. The second sum is a
per-query-head bias independent of~$t$, subtracted once after the dot
product accumulates over all retained tokens. This avoids one write/read
pass of $|\mathcal{T}_{\text{kept}}| \times d$ FP16 values per decode
step.

\textbf{Padding and alignment.}
Bit-packing imposes alignment constraints on K-channel segment sizes: the
2-bit segment is zero-padded so that the channel count is a multiple
of~$4$ (four channels per packed byte), and the 4-bit segment is padded
to a multiple of~$2$. Padded channels receive $s_c = z_c = 0$ and thus
contribute zero to both the dot product and the bias term. The 8-bit
segment has no alignment constraint. V sub-segments are padded
analogously along~$d$.

\textbf{Masking.}
Two boolean masks guard out-of-range accesses in the Triton kernels: a
token-axis mask ($i < T_{\text{eff}}$) zeroes loads beyond the effective
token count, and a channel-axis mask ($j < N_{\text{ch}}$, K-side only)
zeroes loads beyond the true segment width. Since padded channels have
$s_c = z_c = 0$, masked positions dequantize to zero and contribute
nothing to the final result.

\section{Algorithm}
\label{app:algorithm}

\Cref{alg:rdkv} summarizes the full RDKV pipeline executed once at the
end of prefill. The algorithm takes the prefill KV cache and a per-head
bit budget as input, and produces the TriZone packed cache used during
decoding.

\begin{algorithm}[!tb]
\caption{RDKV: Rate-Distortion KV Cache Compression}
\label{alg:rdkv}
\begin{algorithmic}[1]
\Require Prefill KV cache $K^{(\ell)}, V^{(\ell)}$ for each layer $\ell$;
         per-head budget $B_\text{head}$;
         observation window size $S_w$; pooling kernel $w$;
         bit-width set $\mathcal{B} = \{0,2,4,8,16\}$;
         empirical distortion tables $\varepsilon^K(b), \varepsilon^V(b)$
\Ensure  TriZone packed cache for each $(\ell, h)$
\For{each layer $\ell = 1, \ldots, L$}
    \State \textbf{// Stage 1: Weight computation}
    \State $A \gets \operatorname{Softmax}\!\bigl(Q_{[\tau-S_w:\tau]}^{(\ell)}\, K^{(\ell)\top}\!/\sqrt{d}\bigr)$
    \For{each KV head $h = 1, \ldots, H_\text{kv}$}
        \State $w_t^{(h)} \gets \sum_{\tau,g} a_{\tau,g,t}^{(h)}$ for all $t$
        \Comment{token weight in V cache}
        \State $w_t^{(h)} \gets \operatorname{AvgPool1d}(w_t^{(h)},\; w)$
        \State $w_c^{(h)} \gets \tfrac{1}{\sqrt{d}}\,\|Q_{:,c}^{(h)}\|_2 \cdot \|K_{:,c}^{(h)}\|_2$ for all $c$
        \Comment{channel weight in K cache}
    \EndFor
    \Statex
    \State \textbf{// Stage 2: V-side token allocation (per head)}
    \State $B^V \gets \tfrac{1}{2} B_\text{head}$;\quad
           $\bar{B}^V \gets B^V / d$
    \For{each KV head $h$}
        \State $\{b_t^V\}_h \gets \textsc{MCKP}(w_t^{(h)},\; \varepsilon^V,\; \bar{B}^V / T)$
        \State $\mathcal{T}_\text{kept}^{(h)} \gets \{t : b_t^V > 0\}$
    \EndFor
    \Statex
    \State \textbf{// Stage 3: K-side channel allocation (per head)}
    \State $B^K \gets \tfrac{1}{2} B_\text{head}$
    \For{each KV head $h$}
        \State $k_\text{avg}^{(h)} \gets B^K \,/\, \bigl(|\mathcal{T}_\text{kept}^{(h)}| \cdot d\bigr)$
        \State $\{b_c^K\}_h \gets \textsc{MCKP}(w_c^{(h)},\; \varepsilon^K,\; k_\text{avg}^{(h)})$
    \EndFor
    \Statex
    \State \textbf{// Stage 4: TriZone packing}
    \For{each KV head $h$}
        \State Sort $\mathcal{T}_\text{kept}^{(h)}$ by $b_t^V$ into sub-segments $\mathcal{S}_2, \mathcal{S}_4, \mathcal{S}_8$
        \State Sort channels by $b_c^K$ into segments; permute $q$ to match
        \State Quantize and byte-pack V sub-segments $\to$ Zone~A (V)
        \State Quantize and byte-pack K rows of $\mathcal{T}_\text{kept}^{(h)}$ $\to$ Zone~A (K)
        \State Store $\{t : b_t^V = 16\}$ V rows in FP16 $\to$ Zone~B
    \EndFor
\EndFor
\end{algorithmic}
\end{algorithm}

\textsc{MCKP} (\cref{alg:mckp}) solves the multiple-choice knapsack via
Lagrangian bisection on~$\lambda$. For each candidate~$\lambda$, every
unit independently picks the bit-width minimizing its
cost~$w_u \varepsilon_u(b) + \lambda b$. The bisection terminates when
the average bit-width is within tolerance of the budget.

\begin{algorithm}[!tb]
\caption{MCKP: Lagrangian Bisection Knapsack Solver}
\label{alg:mckp}
\begin{algorithmic}[1]
\Require Weights $\{w_u\}$; distortion table $\varepsilon(b)$;
         target average bits $\bar{b}$;
         bit-width set $\mathcal{B}$; tolerance $\delta$; max iterations $I$
\Ensure  Bit-width assignment $\{b_u^\star\}$
\State $\lambda_\text{lo} \gets 0$;\quad $\lambda_\text{hi} \gets \max_u w_u$
\For{$i = 1, \ldots, I$}
    \State $\lambda \gets (\lambda_\text{lo} + \lambda_\text{hi}) / 2$
    \For{each unit $u$}
        \State $b_u^\star \gets \arg\min_{b \in \mathcal{B}}\; w_u \,\varepsilon(b) + \lambda\, b$
    \EndFor
    \State $\bar{b}_\text{cur} \gets \operatorname{mean}(\{b_u^\star\})$
    \If{$|\bar{b}_\text{cur} - \bar{b}| / \bar{b} < \delta$}
        \State \Return $\{b_u^\star\}$
    \ElsIf{$\bar{b}_\text{cur} > \bar{b}$}
        \State $\lambda_\text{lo} \gets \lambda$
    \Else
        \State $\lambda_\text{hi} \gets \lambda$
    \EndIf
\EndFor
\State \Return $\{b_u^\star\}$
\end{algorithmic}
\end{algorithm}


\section{Additional Experiments on LongBench}
\label{app:longbench-detailed}
\label{app:more-models}

In this section, we provide comprehensive experimental results on
LongBench~\citep{bai2023longbench}, a benchmark focused on long-context
understanding with 16 tasks spanning single-document QA, multi-document
QA, summarization, few-shot learning, synthetic retrieval, and code
completion. We perform detailed evaluations with cache budgets ranging
from $64L$ to $1024L$ on four additional models beyond the primary
LLaMA-3.1-8B-Instruct reported in \cref{sec:exp-longbench}:
Mistral-7B-Instruct-v0.3~\citep{mistral} and
Qwen3-4B~\citep{qwen3} to test cross-architecture
generality, and LLaMA-2-13B-Chat~\citep{touvron2023llama} and
Qwen2.5-72B-Instruct~\citep{qwen25} (2-GPU tensor parallelism) to test
cross-scale generality. All methods share the same probe configuration
($S_w = 32$, $w = 5$) so that any performance difference is
attributable to the indicator and allocation, not the probe geometry.

\Cref{tab:longbench-mistral,tab:longbench-qwen3,tab:longbench-13b-72b}
present the detailed per-task scores. Overall, RDKV achieves the highest
average at every (model, budget) combination---without
architecture-specific tuning.
The advantage is largest under aggressive compression: at
$B_\text{total} = 64L$, RDKV leads the strongest baseline by
$1.5$/$5.5$/$2.7$ points on Mistral-7B/Qwen3-4B/LLaMA-3.1-8B,
because water-filling retains moderate-importance tokens at low bit-width
rather than evicting them entirely. At $B_\text{total} = 1024L$ it
recovers $98.6$--$99.4\%$ of FullKV across models; on the two larger
models (LLaMA-2-13B, Qwen2.5-72B) the recovery reaches
$99.5$--$99.6\%$ already at $B_\text{total} = 512L$. The consistency
across architectures and scales suggests that the distortion weights
(\cref{sec:methodology}) capture a model-agnostic signal, enabling the
allocation to transfer without modification.

\begin{table}[!tb]
\centering
\caption{Performance on 16 LongBench datasets for Mistral-7B-Instruct-v0.3 across cache
budgets $B_\text{total} \in \{64L, 128L, 256L, 512L, 1024L\}$.
Snap+Zip denotes SnapKV+ZipCache.
The best result in each row is in \textbf{bold}; the second-best is \underline{underlined}.}
\label{tab:longbench-mistral}
\setlength{\tabcolsep}{0.6pt}
\fontsize{6.4}{7.6}\selectfont
\begin{tabular}{lccccccccccccccccc}
\toprule
\multirow{2}{*}{Method}
 & \multicolumn{3}{c}{Single-Doc QA}
 & \multicolumn{3}{c}{Multi-Doc QA}
 & \multicolumn{3}{c}{Summarization}
 & \multicolumn{3}{c}{Few-shot Learning}
 & \multicolumn{2}{c}{Synthetic}
 & \multicolumn{2}{c}{Code}
 & \multirow{2}{*}{Avg.} \\
\cmidrule(lr){2-4}\cmidrule(lr){5-7}\cmidrule(lr){8-10}\cmidrule(lr){11-13}\cmidrule(lr){14-15}\cmidrule(lr){16-17}
 & \rotatebox{35}{NrtvQA} & \rotatebox{35}{Qasper} & \rotatebox{35}{MF-en}
 & \rotatebox{35}{HotpotQA} & \rotatebox{35}{2WikiMQA} & \rotatebox{35}{Musique}
 & \rotatebox{35}{GovRep} & \rotatebox{35}{QMSum} & \rotatebox{35}{MultiNews}
 & \rotatebox{35}{TREC} & \rotatebox{35}{TriviaQA} & \rotatebox{35}{SAMSum}
 & \rotatebox{35}{PCount} & \rotatebox{35}{PRe}
 & \rotatebox{35}{Lcc} & \rotatebox{35}{RB-P}
 & \\
\midrule
\multicolumn{18}{c}{\textit{Mistral-7B-Instruct-v0.3,} $B_\text{total} = \text{Full}$} \\
\midrule
FullKV & 25.91 & 38.13 & 49.62 & 52.07 & 39.03 & 28.10 & 34.11 & 25.73 & 26.47 & 76.00 & 88.59 & 47.46 & 7.01 & 98.00 & 61.58 & 62.34 & 47.51 \\
\midrule
\multicolumn{18}{c}{\textit{Mistral-7B-Instruct-v0.3,} $B_\text{total} = 64L$} \\
\midrule
SnapKV          & 17.53 & 17.78 & 32.36 & 40.94 & 30.84 & 19.34 & 16.27 & 20.33 & 13.77 & 38.00 & 87.64 & 38.80 & \textbf{5.00} & 76.50 & 48.46 & 47.72 & 34.45 \\
AdaKV           & 17.33 & 18.43 & 34.33 & 42.13 & 31.16 & 21.08 & 16.79 & 20.64 & 14.23 & 38.00 & 88.18 & 39.61 & 3.50 & 83.00 & 51.28 & 48.93 & 35.54 \\
ThinK    & 20.27 & 24.90 & 40.42 & 45.00 & 33.54 & 21.44 & 19.83 & 21.26 & 18.45 & 39.50 & 87.94 & 40.66 & 3.00 & 83.00 & 53.19 & 51.51 & 37.74 \\
Snap+Zip & \underline{22.80} & \underline{25.83} & \textbf{47.55} & \textbf{47.38} & \textbf{35.23} & \textbf{24.03} & \underline{21.54} & \underline{21.59} & \underline{21.61} & \underline{55.00} & \underline{89.05} & \underline{44.84} & 4.00 & \underline{88.00} & \underline{56.59} & \underline{56.61} & \underline{41.35} \\
\textbf{RDKV}   & \textbf{24.70} & \textbf{29.46} & \underline{47.51} & \underline{46.48} & \underline{34.37} & \underline{22.73} & \textbf{23.40} & \textbf{23.23} & \textbf{22.73} & \textbf{60.00} & \textbf{89.07} & \textbf{44.99} & \underline{4.50} & \textbf{96.00} & \textbf{57.78} & \textbf{58.19} & \textbf{42.82} \\
\midrule
\multicolumn{18}{c}{\textit{Mistral-7B-Instruct-v0.3,} $B_\text{total} = 128L$} \\
\midrule
SnapKV          & 21.25 & 22.92 & 42.30 & 44.75 & 33.34 & 21.24 & 21.00 & 21.37 & 19.30 & 44.50 & \textbf{89.35} & 42.47 & 4.50 & 91.00 & 55.98 & 54.12 & 39.34 \\
AdaKV           & 21.97 & 23.31 & 46.28 & 46.67 & 34.52 & 21.47 & 20.92 & 21.85 & 19.71 & 49.00 & \underline{89.19} & 42.63 & \textbf{6.50} & 91.50 & 56.41 & 55.60 & 40.47 \\
ThinK    & 21.90 & 29.83 & 48.32 & 47.25 & 34.76 & 23.27 & 21.73 & \underline{22.37} & 21.18 & 47.50 & 88.36 & 43.10 & 5.00 & \underline{93.00} & 57.62 & 56.39 & 41.35 \\
Snap+Zip & \underline{23.62} & \underline{32.30} & \underline{48.55} & \textbf{50.23} & \underline{35.41} & \textbf{26.96} & \underline{24.22} & 22.23 & \underline{23.58} & \underline{67.50} & 88.63 & \underline{45.19} & \underline{5.50} & 92.00 & \underline{59.01} & \underline{58.30} & \underline{43.95} \\
\textbf{RDKV}   & \textbf{25.61} & \textbf{33.76} & \textbf{48.58} & \underline{48.65} & \textbf{36.19} & \underline{25.37} & \textbf{25.39} & \textbf{24.14} & \textbf{24.14} & \textbf{70.00} & 89.18 & \textbf{45.52} & 5.00 & \textbf{97.50} & \textbf{59.30} & \textbf{59.69} & \textbf{44.88} \\
\midrule
\multicolumn{18}{c}{\textit{Mistral-7B-Instruct-v0.3,} $B_\text{total} = 256L$} \\
\midrule
SnapKV          & 24.37 & 29.08 & 47.68 & 48.14 & 34.66 & 24.45 & 23.17 & 23.24 & 22.37 & 57.00 & 88.85 & 44.36 & \textbf{6.00} & \underline{96.50} & 59.32 & 57.98 & 42.95 \\
AdaKV           & 24.71 & 30.19 & 48.31 & 49.08 & 34.69 & \textbf{26.03} & 23.13 & 23.46 & 22.19 & 66.00 & \underline{88.93} & 44.57 & 4.50 & \underline{96.50} & 59.38 & 59.67 & 43.83 \\
ThinK    & 24.82 & 31.21 & \textbf{50.22} & \textbf{50.79} & 35.00 & 23.70 & 23.75 & 23.55 & 23.17 & 66.00 & \textbf{89.61} & 43.90 & \underline{5.50} & 94.50 & 59.37 & 59.04 & 44.01 \\
Snap+Zip & \underline{25.85} & \textbf{36.05} & 48.22 & \underline{49.31} & \textbf{36.59} & 24.75 & \underline{27.34} & \underline{23.59} & \textbf{25.26} & \underline{71.00} & 88.75 & \underline{45.37} & \underline{5.50} & 94.50 & \textbf{59.76} & \underline{59.88} & \underline{45.11} \\
\textbf{RDKV}   & \textbf{26.00} & \underline{35.68} & \underline{49.89} & 49.11 & \underline{35.50} & \underline{25.99} & \textbf{27.78} & \textbf{24.48} & \underline{25.22} & \textbf{73.50} & 88.58 & \textbf{46.40} & \textbf{6.00} & \textbf{97.00} & \underline{59.49} & \textbf{60.60} & \textbf{45.70} \\
\midrule
\multicolumn{18}{c}{\textit{Mistral-7B-Instruct-v0.3,} $B_\text{total} = 512L$} \\
\midrule
SnapKV          & 25.89 & 32.99 & 48.59 & 50.39 & \underline{36.72} & 26.24 & 24.86 & 23.81 & 24.17 & 69.00 & \textbf{89.28} & 45.39 & \textbf{6.00} & \textbf{97.50} & 60.21 & 61.03 & 45.13 \\
AdaKV           & 26.04 & 33.06 & \textbf{50.75} & 50.29 & 35.81 & \underline{26.58} & 25.03 & 23.39 & 23.65 & 71.50 & 88.86 & 46.03 & \textbf{6.00} & \textbf{97.50} & \textbf{60.74} & \textbf{62.25} & 45.47 \\
ThinK    & \underline{27.21} & \underline{36.41} & \underline{49.96} & \underline{50.54} & \textbf{37.33} & 25.04 & 25.82 & 23.87 & 24.70 & \underline{72.50} & 88.94 & 45.01 & \textbf{6.00} & \underline{96.00} & \underline{60.58} & \underline{62.08} & 45.75 \\
Snap+Zip & \textbf{27.57} & 35.17 & 48.23 & \textbf{51.43} & 36.70 & 25.40 & \underline{28.66} & \underline{24.37} & \underline{25.61} & \textbf{75.50} & 89.01 & \underline{46.09} & \textbf{6.00} & 95.50 & 60.34 & 60.70 & \underline{46.02} \\
\textbf{RDKV}   & 26.06 & \textbf{37.26} & 49.62 & 50.07 & 36.58 & \textbf{26.99} & \textbf{30.48} & \textbf{25.14} & \textbf{26.23} & \textbf{75.50} & \underline{89.08} & \textbf{46.77} & \underline{5.50} & \underline{96.50} & 59.14 & 60.58 & \textbf{46.34} \\
\midrule
\multicolumn{18}{c}{\textit{Mistral-7B-Instruct-v0.3,} $B_\text{total} = 1024L$} \\
\midrule
SnapKV          & 26.05 & 35.63 & \textbf{50.05} & 50.39 & 36.44 & \underline{27.40} & 27.24 & 23.97 & 25.22 & 73.00 & \textbf{89.19} & 45.84 & 6.00 & \textbf{98.50} & \underline{61.42} & 61.55 & 46.12 \\
AdaKV           & 26.31 & 36.49 & 49.95 & \textbf{51.76} & 37.37 & 27.07 & 27.11 & \underline{24.84} & 25.01 & 73.00 & \underline{89.03} & 45.83 & 5.50 & \textbf{98.50} & 61.03 & \textbf{62.35} & 46.32 \\
ThinK    & \textbf{27.97} & \textbf{38.68} & \underline{49.96} & 49.97 & \textbf{38.69} & 25.28 & 28.05 & 24.57 & 25.75 & \underline{74.00} & 88.72 & 45.83 & \textbf{6.50} & 95.00 & \textbf{61.67} & \underline{61.85} & \underline{46.41} \\
Snap+Zip & 26.41 & 37.75 & 48.90 & \underline{51.54} & 35.26 & 25.74 & \underline{30.97} & 24.25 & \underline{26.07} & \textbf{76.00} & 88.70 & \underline{46.25} & \underline{6.06} & 95.50 & 59.18 & 60.71 & 46.21 \\
\textbf{RDKV}   & \underline{27.00} & \underline{37.96} & 49.88 & 50.56 & \underline{37.63} & \textbf{28.20} & \textbf{32.33} & \textbf{25.05} & \textbf{26.57} & \textbf{76.00} & 88.48 & \textbf{47.88} & 5.00 & \underline{97.50} & 59.35 & 60.40 & \textbf{46.86} \\
\bottomrule
\end{tabular}
\end{table}

\begin{table}[!tb]
\centering
\caption{Performance on 16 LongBench datasets for Qwen3-4B-Instruct-2507 across cache
budgets $B_\text{total} \in \{64L, 128L, 256L, 512L, 1024L\}$.
Snap+Zip denotes SnapKV+ZipCache.
The best result in each row is in \textbf{bold}; the second-best is \underline{underlined}.}
\label{tab:longbench-qwen3}
\setlength{\tabcolsep}{0.6pt}
\fontsize{6.4}{7.6}\selectfont
\begin{tabular}{lccccccccccccccccc}
\toprule
\multirow{2}{*}{Method}
 & \multicolumn{3}{c}{Single-Doc QA}
 & \multicolumn{3}{c}{Multi-Doc QA}
 & \multicolumn{3}{c}{Summarization}
 & \multicolumn{3}{c}{Few-shot Learning}
 & \multicolumn{2}{c}{Synthetic}
 & \multicolumn{2}{c}{Code}
 & \multirow{2}{*}{Avg.} \\
\cmidrule(lr){2-4}\cmidrule(lr){5-7}\cmidrule(lr){8-10}\cmidrule(lr){11-13}\cmidrule(lr){14-15}\cmidrule(lr){16-17}
 & \rotatebox{35}{NrtvQA} & \rotatebox{35}{Qasper} & \rotatebox{35}{MF-en}
 & \rotatebox{35}{HotpotQA} & \rotatebox{35}{2WikiMQA} & \rotatebox{35}{Musique}
 & \rotatebox{35}{GovRep} & \rotatebox{35}{QMSum} & \rotatebox{35}{MultiNews}
 & \rotatebox{35}{TREC} & \rotatebox{35}{TriviaQA} & \rotatebox{35}{SAMSum}
 & \rotatebox{35}{PCount} & \rotatebox{35}{PRe}
 & \rotatebox{35}{Lcc} & \rotatebox{35}{RB-P}
 & \\
\midrule
\multicolumn{18}{c}{\textit{Qwen3-4B-Instruct-2507,} $B_\text{total} = \text{Full}$} \\
\midrule
FullKV & 27.96 & 44.61 & 49.84 & 58.84 & 43.27 & 25.54 & 30.64 & 22.56 & 24.06 & 74.50 & 87.15 & 45.36 & 2.15 & 100.00 & 64.69 & 57.56 & 47.42 \\
\midrule
\multicolumn{18}{c}{\textit{Qwen3-4B-Instruct-2507,} $B_\text{total} = 64L$} \\
\midrule
SnapKV          & 16.95 & 26.10 & 35.04 & 43.94 & 34.68 & 15.70 & 12.69 & 19.63 & 12.69 & 39.50 & 74.68 & 35.99 & 1.51 & 45.50 & 53.03 & 45.78 & 32.09 \\
AdaKV           & 16.83 & 25.08 & 36.52 & 45.79 & 35.94 & 15.06 & 12.97 & 19.40 & 12.97 & 42.50 & 76.67 & 36.46 & \textbf{2.73} & 47.00 & 53.54 & 45.54 & 32.81 \\
ThinK    & \underline{18.65} & 28.98 & 39.41 & 51.50 & \underline{39.33} & 19.25 & 14.72 & 20.81 & 14.53 & 44.50 & 80.45 & 38.43 & 0.93 & \underline{85.50} & 55.67 & 48.79 & 37.59 \\
Snap+Zip & 18.22 & \underline{31.79} & \underline{46.14} & \underline{52.48} & 38.27 & \underline{20.22} & \textbf{20.43} & \underline{21.16} & \underline{19.33} & \underline{61.00} & \underline{84.02} & \underline{40.83} & \underline{1.67} & 60.29 & \textbf{60.95} & \underline{51.23} & \underline{39.25} \\
\textbf{RDKV}   & \textbf{24.94} & \textbf{39.23} & \textbf{49.24} & \textbf{59.71} & \textbf{43.46} & \textbf{25.50} & \underline{20.01} & \textbf{23.15} & \textbf{20.18} & \textbf{70.50} & \textbf{84.64} & \textbf{41.88} & 0.75 & \textbf{100.00} & \underline{60.43} & \textbf{52.29} & \textbf{44.74} \\
\midrule
\multicolumn{18}{c}{\textit{Qwen3-4B-Instruct-2507,} $B_\text{total} = 128L$} \\
\midrule
SnapKV          & 21.46 & 31.15 & 43.96 & 53.83 & 40.79 & 22.12 & 15.77 & 21.90 & 16.16 & 51.50 & 83.44 & 39.72 & 1.30 & 93.03 & 59.07 & 50.75 & 40.37 \\
AdaKV           & 23.78 & 31.71 & 45.03 & \underline{57.81} & 40.54 & 23.84 & 16.40 & 22.70 & 16.45 & 56.50 & 81.10 & 40.47 & \underline{1.60} & \underline{98.00} & 60.64 & 51.76 & 41.77 \\
ThinK    & 22.88 & 35.83 & 44.32 & 56.70 & 40.32 & 23.61 & 18.68 & \underline{22.77} & 17.87 & 59.00 & 84.64 & 40.93 & 1.32 & 97.00 & 61.76 & 52.80 & 42.53 \\
Snap+Zip & \underline{23.84} & \underline{38.35} & \underline{48.62} & 57.59 & \underline{41.88} & \underline{24.04} & \underline{23.43} & 22.18 & \underline{21.62} & \underline{68.50} & \underline{85.73} & \underline{42.24} & \textbf{1.91} & \underline{98.00} & \underline{62.43} & \underline{53.41} & \underline{44.61} \\
\textbf{RDKV}   & \textbf{26.79} & \textbf{41.89} & \textbf{49.87} & \textbf{59.14} & \textbf{42.95} & \textbf{26.02} & \textbf{24.01} & \textbf{23.39} & \textbf{22.46} & \textbf{74.00} & \textbf{86.53} & \textbf{43.00} & 1.25 & \textbf{100.00} & \textbf{63.10} & \textbf{53.48} & \textbf{46.12} \\
\midrule
\multicolumn{18}{c}{\textit{Qwen3-4B-Instruct-2507,} $B_\text{total} = 256L$} \\
\midrule
SnapKV          & \underline{25.92} & 36.55 & \underline{48.56} & 59.14 & 41.07 & 25.18 & 19.13 & 22.75 & 19.34 & 65.00 & 84.21 & 40.26 & 1.72 & \underline{99.75} & 62.46 & 53.74 & 44.05 \\
AdaKV           & 25.46 & 38.20 & 48.10 & \underline{59.34} & 42.49 & \textbf{26.44} & 19.46 & 23.01 & 19.60 & 69.00 & 86.32 & 40.80 & 1.75 & \textbf{100.00} & 63.18 & \underline{55.07} & 44.89 \\
ThinK    & 25.30 & 39.37 & 46.06 & 59.10 & 42.73 & 25.10 & 21.58 & \underline{23.14} & 20.80 & 68.50 & 86.58 & 41.22 & \underline{1.88} & \textbf{100.00} & \underline{63.50} & \textbf{55.32} & 45.01 \\
Snap+Zip & 24.21 & \underline{40.99} & 48.26 & 58.93 & \underline{42.75} & 24.22 & \underline{25.96} & 22.62 & \underline{23.13} & \underline{72.50} & \textbf{87.13} & \underline{43.42} & \textbf{1.90} & 99.50 & \textbf{63.77} & 54.68 & \underline{45.87} \\
\textbf{RDKV}   & \textbf{26.45} & \textbf{43.49} & \textbf{50.30} & \textbf{59.60} & \textbf{42.91} & \underline{25.52} & \textbf{27.30} & \textbf{23.80} & \textbf{23.30} & \textbf{74.00} & \underline{87.03} & \textbf{44.27} & 1.25 & \textbf{100.00} & 63.39 & 53.91 & \textbf{46.66} \\
\midrule
\multicolumn{18}{c}{\textit{Qwen3-4B-Instruct-2507,} $B_\text{total} = 512L$} \\
\midrule
SnapKV          & 26.57 & 39.86 & \underline{48.52} & \underline{59.27} & \underline{43.23} & \underline{26.10} & 23.00 & 22.74 & 21.67 & 70.00 & \textbf{87.95} & 41.81 & \underline{1.99} & \textbf{100.00} & 64.72 & 56.35 & 45.86 \\
AdaKV           & \underline{26.72} & 40.87 & 48.17 & 59.10 & 42.93 & \textbf{26.61} & 23.28 & 22.58 & 21.54 & 70.50 & 87.15 & 42.20 & 1.78 & \textbf{100.00} & 64.69 & \textbf{56.58} & 45.92 \\
ThinK    & 26.48 & 41.08 & 46.44 & 59.15 & 42.90 & 25.06 & 24.79 & \underline{23.38} & 22.38 & 72.50 & \underline{87.64} & 42.51 & \textbf{2.06} & \textbf{100.00} & \textbf{65.07} & 56.34 & 46.11 \\
Snap+Zip & 25.47 & \underline{42.27} & 47.80 & 57.58 & 43.17 & 23.24 & \underline{28.44} & 22.60 & \underline{23.79} & \textbf{75.00} & 86.82 & \underline{43.79} & 1.94 & \underline{99.50} & \underline{65.03} & \underline{56.52} & \underline{46.44} \\
\textbf{RDKV}   & \textbf{27.75} & \textbf{43.89} & \textbf{49.81} & \textbf{59.97} & \textbf{43.87} & 25.37 & \textbf{30.07} & \textbf{23.67} & \textbf{24.00} & \underline{74.50} & 86.73 & \textbf{44.60} & 0.50 & \textbf{100.00} & 63.29 & 54.42 & \textbf{47.03} \\
\midrule
\multicolumn{18}{c}{\textit{Qwen3-4B-Instruct-2507,} $B_\text{total} = 1024L$} \\
\midrule
SnapKV          & \textbf{27.80} & 42.45 & 48.56 & 59.01 & 43.04 & \underline{25.83} & 26.55 & 23.26 & 22.91 & 72.50 & 87.44 & 42.80 & 1.83 & \textbf{100.00} & 64.83 & 56.65 & 46.59 \\
AdaKV           & \underline{27.64} & 42.54 & \underline{48.65} & \textbf{59.57} & 43.31 & \textbf{26.35} & 26.75 & 23.02 & 23.16 & 74.00 & \textbf{88.08} & 43.49 & 1.60 & \textbf{100.00} & \underline{64.86} & \textbf{57.44} & \underline{46.90} \\
ThinK    & 26.73 & 42.98 & 47.88 & 58.62 & 43.18 & 24.97 & 27.90 & \textbf{23.55} & 23.43 & \textbf{75.00} & \underline{87.85} & 42.95 & \underline{1.97} & \textbf{100.00} & \textbf{65.48} & \underline{57.04} & 46.85 \\
Snap+Zip & 26.60 & \underline{43.46} & 47.80 & \textbf{59.57} & \textbf{43.76} & 22.11 & \underline{29.78} & 22.67 & \underline{23.84} & 73.50 & 87.36 & \underline{44.65} & \textbf{2.16} & \underline{99.50} & 64.09 & 56.08 & 46.68 \\
\textbf{RDKV}   & 27.35 & \textbf{44.36} & \textbf{50.40} & \underline{59.35} & \underline{43.63} & 24.87 & \textbf{30.91} & \underline{23.29} & \textbf{24.45} & \underline{74.50} & 87.26 & \textbf{45.24} & 0.75 & \textbf{100.00} & 62.86 & 54.61 & \textbf{47.12} \\
\bottomrule
\end{tabular}
\end{table}

\begin{table}[!tb]
\centering
\caption{Performance on 16 LongBench datasets for LLaMA-2-13B-Chat and
Qwen2.5-72B-Instruct (2-GPU tensor parallelism) at $B_\text{total} \in \{128L, 512L\}$.
Snap+Zip denotes SnapKV+ZipCache.
The best result in each row is in \textbf{bold}; the second-best is \underline{underlined}.}
\label{tab:longbench-13b-72b}
\setlength{\tabcolsep}{0.6pt}
\fontsize{6.4}{7.6}\selectfont
\begin{tabular}{lccccccccccccccccc}
\toprule
\multirow{2}{*}{Method}
 & \multicolumn{3}{c}{Single-Doc QA}
 & \multicolumn{3}{c}{Multi-Doc QA}
 & \multicolumn{3}{c}{Summarization}
 & \multicolumn{3}{c}{Few-shot Learning}
 & \multicolumn{2}{c}{Synthetic}
 & \multicolumn{2}{c}{Code}
 & \multirow{2}{*}{Avg.} \\
\cmidrule(lr){2-4}\cmidrule(lr){5-7}\cmidrule(lr){8-10}\cmidrule(lr){11-13}\cmidrule(lr){14-15}\cmidrule(lr){16-17}
 & \rotatebox{35}{NrtvQA} & \rotatebox{35}{Qasper} & \rotatebox{35}{MF-en}
 & \rotatebox{35}{HotpotQA} & \rotatebox{35}{2WikiMQA} & \rotatebox{35}{Musique}
 & \rotatebox{35}{GovRep} & \rotatebox{35}{QMSum} & \rotatebox{35}{MultiNews}
 & \rotatebox{35}{TREC} & \rotatebox{35}{TriviaQA} & \rotatebox{35}{SAMSum}
 & \rotatebox{35}{PCount} & \rotatebox{35}{PRe}
 & \rotatebox{35}{Lcc} & \rotatebox{35}{RB-P}
 & \\
\midrule
\multicolumn{18}{c}{\textit{LLaMA-2-13B-Chat,} $B_\text{total} = \text{Full}$} \\
\midrule
FullKV & 16.25 & 17.54 & 27.88 & 14.81 & 14.97 & 6.31 & 27.47 & 20.58 & 26.12 & 69.00 & 87.83 & 38.12 & 2.91 & 12.38 & 51.17 & 51.91 & 30.33 \\
\midrule
\multicolumn{18}{c}{\textit{LLaMA-2-13B-Chat,} $B_\text{total} = 128L$} \\
\midrule
SnapKV          & 13.96 & 16.25 & 25.98 & 15.90 & 14.40 & \underline{5.78} & 19.40 & 20.07 & 20.52 & 41.00 & 86.12 & 34.97 & 2.78 & 12.13 & 46.23 & 44.95 & 26.28 \\
AdaKV           & 14.71 & \textbf{17.81} & \textbf{28.54} & 16.27 & \textbf{15.12} & 5.52 & 19.71 & 20.19 & 21.51 & \underline{62.50} & 87.02 & 34.93 & \underline{3.06} & 10.25 & 49.59 & \underline{48.58} & 28.46 \\
ThinK    & 14.35 & 14.81 & 26.36 & \underline{16.29} & 14.35 & 5.54 & 19.95 & \underline{20.24} & 21.53 & 48.00 & 87.01 & 35.45 & 2.67 & 12.25 & 48.02 & 46.59 & 27.09 \\
Snap+Zip & \textbf{16.75} & 15.93 & \underline{27.73} & \textbf{16.40} & 13.94 & 5.61 & \underline{21.66} & 18.93 & \underline{23.18} & 58.00 & \textbf{87.49} & \underline{35.79} & \textbf{5.90} & \textbf{17.38} & \textbf{50.33} & 47.14 & \underline{28.88} \\
\textbf{RDKV}   & \underline{16.03} & \underline{16.40} & 27.09 & 15.17 & \underline{15.01} & \textbf{6.89} & \textbf{22.42} & \textbf{20.29} & \textbf{24.10} & \textbf{68.50} & \underline{87.05} & \textbf{36.86} & 3.01 & \underline{13.88} & \underline{49.83} & \textbf{50.73} & \textbf{29.58} \\
\midrule
\multicolumn{18}{c}{\textit{LLaMA-2-13B-Chat,} $B_\text{total} = 512L$} \\
\midrule
SnapKV          & \underline{15.57} & 16.55 & 28.23 & 15.49 & \textbf{15.15} & 6.24 & 22.52 & 20.13 & 24.04 & 67.50 & 86.53 & \textbf{38.47} & \underline{3.09} & \textbf{13.88} & \textbf{50.49} & 50.36 & 29.64 \\
AdaKV           & 15.53 & 17.10 & 28.45 & 15.84 & \underline{15.14} & \textbf{6.38} & 22.98 & \textbf{20.68} & 24.22 & \textbf{69.50} & \underline{87.11} & 37.52 & 2.42 & \underline{13.12} & 49.50 & \underline{51.25} & \underline{29.80} \\
ThinK    & 15.52 & \underline{17.64} & \underline{28.49} & \underline{16.12} & 14.88 & \underline{6.29} & 23.16 & 19.95 & 24.52 & 68.00 & 86.75 & 37.97 & 2.84 & \textbf{13.88} & \underline{50.10} & 50.60 & 29.79 \\
Snap+Zip & 13.18 & 14.74 & 28.41 & \textbf{17.03} & 13.03 & 5.84 & \underline{24.33} & \underline{20.29} & \underline{25.32} & \underline{68.50} & 86.76 & \underline{38.19} & \textbf{5.50} & 9.50 & 49.65 & 50.99 & 29.45 \\
\textbf{RDKV}   & \textbf{16.35} & \textbf{18.08} & \textbf{28.78} & 15.23 & 15.10 & 6.19 & \textbf{25.54} & 20.27 & \textbf{25.84} & \underline{68.50} & \textbf{87.39} & 37.41 & 2.68 & \underline{13.12} & 50.00 & \textbf{52.67} & \textbf{30.20} \\
\midrule
\multicolumn{18}{c}{\textit{Qwen2.5-72B-Instruct (2-GPU TP),} $B_\text{total} = \text{Full}$} \\
\midrule
FullKV & 32.21 & 49.59 & 51.88 & 65.26 & 65.51 & 41.37 & 34.00 & 24.35 & 24.74 & 78.00 & 91.76 & 47.87 & 21.50 & 99.50 & 66.36 & 71.42 & 54.08 \\
\midrule
\multicolumn{18}{c}{\textit{Qwen2.5-72B-Instruct (2-GPU TP),} $B_\text{total} = 128L$} \\
\midrule
SnapKV          & 28.76 & 32.00 & 46.33 & 62.37 & 62.21 & 36.60 & 20.94 & 21.38 & 17.90 & 53.00 & 90.26 & 40.81 & \textbf{23.00} & 97.42 & 59.07 & 63.89 & 47.25 \\
AdaKV           & \underline{30.68} & 33.07 & 49.54 & \underline{63.92} & 62.93 & 38.35 & 21.80 & 22.02 & 18.46 & 61.50 & 90.91 & 42.15 & 21.50 & \underline{99.25} & 59.94 & 65.71 & 48.86 \\
ThinK    & 28.02 & 29.54 & 46.30 & 62.17 & 61.47 & 37.73 & 20.72 & 21.27 & 17.70 & 47.00 & 90.28 & 40.22 & 21.50 & 96.17 & 58.42 & 63.51 & 46.38 \\
Snap+Zip & 30.60 & \textbf{46.38} & \textbf{50.50} & \textbf{65.47} & \textbf{64.97} & \textbf{40.35} & \underline{26.64} & \underline{22.30} & \textbf{22.14} & \underline{73.00} & \underline{91.86} & \textbf{45.29} & 19.50 & 99.00 & \underline{64.03} & \underline{68.52} & \underline{51.91} \\
\textbf{RDKV}   & \textbf{31.76} & \underline{44.79} & \underline{49.57} & 63.86 & \underline{64.42} & \underline{39.72} & \textbf{26.93} & \textbf{23.49} & \underline{22.10} & \textbf{76.00} & \textbf{92.16} & \underline{45.15} & \underline{22.00} & \textbf{99.50} & \textbf{65.48} & \textbf{69.80} & \textbf{52.30} \\
\midrule
\multicolumn{18}{c}{\textit{Qwen2.5-72B-Instruct (2-GPU TP),} $B_\text{total} = 512L$} \\
\midrule
SnapKV          & 32.50 & 45.81 & 51.23 & 64.56 & 64.39 & 40.55 & 26.89 & 23.24 & 22.07 & 75.50 & 92.16 & 45.03 & \underline{21.50} & \textbf{99.50} & 65.04 & 69.95 & 52.49 \\
AdaKV           & \textbf{33.78} & 45.16 & 50.82 & 63.96 & \underline{64.80} & \underline{40.83} & 27.41 & 23.02 & 22.29 & \underline{76.50} & 91.89 & 45.60 & \textbf{22.00} & \textbf{99.50} & \underline{65.48} & \underline{70.48} & 52.72 \\
ThinK    & 31.03 & 43.89 & 50.81 & \underline{64.59} & 64.58 & 39.35 & 26.33 & 23.42 & 21.92 & 73.50 & \underline{92.19} & 45.76 & 21.00 & \underline{99.17} & 65.34 & 69.89 & 52.05 \\
Snap+Zip & 32.50 & \textbf{50.24} & \textbf{51.80} & 64.21 & 64.33 & \textbf{41.00} & \underline{31.63} & \underline{23.87} & \textbf{24.23} & \textbf{78.00} & \textbf{92.61} & \underline{46.28} & \underline{21.50} & \textbf{99.50} & 65.03 & 69.85 & \underline{53.54} \\
\textbf{RDKV}   & \underline{32.65} & \underline{49.51} & \underline{51.68} & \textbf{64.77} & \textbf{64.95} & 40.41 & \textbf{31.78} & \textbf{24.21} & \underline{23.94} & \textbf{78.00} & 92.06 & \textbf{47.91} & \underline{21.50} & \textbf{99.50} & \textbf{66.34} & \textbf{71.90} & \textbf{53.82} \\
\bottomrule
\end{tabular}
\end{table}


\section{Additional Experiments on RULER}
\label{app:ruler-detailed}

In this section, we present a detailed evaluation of RDKV on the
various subtasks of the RULER benchmark~\citep{hsieh2024ruler}. RULER is
specifically designed to assess core long-context capabilities through a
diverse suite of tasks. The retrieval suite includes four variants of
needle-in-a-haystack tests: Single-Needle (S-NIAH-1/2/3) at increasing
difficulty, Multi-Key (MK-NIAH-1/2/3), Multi-Query (MQ-NIAH), and
Multi-Value (MV-NIAH), which evaluate recall accuracy under diverse
distractor settings and query formulations. Beyond retrieval, Variable
Tracking (VT) measures multi-hop reasoning by requiring models to resolve
transitive variable references scattered throughout the input. Lastly,
the aggregation tasks Common Word Extraction (CWE) and Frequent Word
Extraction (FWE) test the ability to compress and synthesize
high-density signal distributed across long contexts. These tasks
collectively pose distinct challenges for context retention, salience
estimation, and compositional reasoning.

We evaluate RDKV on LLaMA-3.1-8B-Instruct with
$B_\text{total} = 2048L$, across input lengths from $4$k to $128$k
tokens. The evaluation compares RDKV with SnapKV, AdaKV, ThinK,
SnapKV+ZipCache, and FullKV (oracle).

As reported in \cref{tab:ruler-detailed}, RDKV consistently achieves
the highest average accuracy across all context lengths. At $4$k tokens,
RDKV matches FullKV ($98.59$ vs.\ $98.58$), while the best eviction
baseline (ThinK) reaches $98.11$. As context length increases, the gap
widens: at $64$k, RDKV leads with $83.20$ vs.\ $78.73$ for the next
best method (AdaKV), a $4.5$-point margin. A breakdown by task category
clarifies the source of the advantage. On S-NIAH-3---the hardest
single-needle retrieval variant---RDKV is the only method that stays
above $99\%$ from $4$k through $64$k; SnapKV and AdaKV drop to
$89.80$/$92.60$ at $64$k and to $50.80$/$47.80$ at $128$k, whereas
RDKV still reaches $99.80$ and $93.80$, respectively. The CWE
aggregation task at $16$k illustrates a complementary advantage:
eviction discards low-attention tokens that carry the frequency signal,
yielding $46$--$55$ for the baselines; RDKV retains these tokens at low
bit-width and scores $77.26$, a $22$-point improvement over the next
best method.

\begin{table}[!tb]
\centering
\caption{Performance on the 11 RULER tasks across context lengths
$\{4\text{k}, 8\text{k}, 16\text{k}, 32\text{k}, 64\text{k}, 128\text{k}\}$ for
LLaMA-3.1-8B-Instruct at $B_\text{total} = 2048L$. The best result in each row is in
\textbf{bold}; the second-best is \underline{underlined}. FullKV (no compression) is
reported as a reference upper bound and excluded from the ranking.}
\label{tab:ruler-detailed}
\setlength{\tabcolsep}{2.2pt}
\fontsize{6.6}{7.8}\selectfont
\begin{tabular}{lccccccccccccc}
\toprule
Method & S-NIAH-1 & S-NIAH-2 & S-NIAH-3 & MK-NIAH-1 & MK-NIAH-2 & MK-NIAH-3 & MQ-NIAH & MV-NIAH & VT & CWE & FWE & Avg. \\
\midrule
\multicolumn{13}{c}{\textit{Sequence length} $= 4$k} \\
\midrule
FullKV         & 100.00 & 100.00 & 100.00 & 100.00 & 100.00 & 99.80 & 99.80 & 99.70 & 99.96 & 98.92 & 86.20 & 98.58 \\
\midrule
SnapKV          & \textbf{100.00} & \textbf{100.00} & \textbf{100.00} & \textbf{100.00} & 99.60 & 95.80 & 99.70 & \underline{99.70} & 99.92 & 97.42 & 76.40 & 97.14 \\
AdaKV           & \textbf{100.00} & \textbf{100.00} & \textbf{100.00} & \textbf{100.00} & \underline{99.80} & \underline{99.00} & \underline{99.75} & \textbf{99.75} & \textbf{99.96} & 98.20 & 81.40 & 97.99 \\
ThinK       & \textbf{100.00} & \textbf{100.00} & \textbf{100.00} & \textbf{100.00} & \underline{99.80} & 96.80 & 99.70 & 99.50 & 99.80 & \textbf{99.12} & \underline{84.47} & \underline{98.11} \\
Snap+Zip    & 46.60 & 98.60 & 92.80 & 99.80 & 63.00 & 0.00 & 99.20 & 94.65 & 88.40 & 83.72 & 74.60 & 76.49 \\
\textbf{RDKV}   & \textbf{100.00} & \textbf{100.00} & \textbf{100.00} & \textbf{100.00} & \textbf{100.00} & \textbf{99.80} & \textbf{99.80} & \underline{99.70} & \textbf{99.96} & \underline{98.94} & \textbf{86.27} & \textbf{98.59} \\
\midrule
\multicolumn{13}{c}{\textit{Sequence length} $= 8$k} \\
\midrule
FullKV         & 100.00 & 100.00 & 100.00 & 99.80 & 99.80 & 99.80 & 99.70 & 99.50 & 99.76 & 93.76 & 95.60 & 98.88 \\
\midrule
SnapKV          & \textbf{100.00} & \textbf{100.00} & 98.20 & \textbf{99.80} & 99.60 & 69.40 & \textbf{99.90} & 99.20 & \textbf{99.80} & 77.68 & 86.60 & 93.65 \\
AdaKV           & \textbf{100.00} & \textbf{100.00} & \underline{99.60} & \textbf{99.80} & \textbf{99.80} & \underline{86.20} & \underline{99.80} & \underline{99.35} & \underline{99.76} & 77.22 & 88.93 & \underline{95.50} \\
ThinK       & \textbf{100.00} & \textbf{100.00} & 92.80 & \textbf{99.80} & \textbf{99.80} & 67.00 & 99.65 & 99.20 & \underline{99.76} & \underline{90.48} & \underline{89.07} & 94.32 \\
Snap+Zip    & 66.00 & 98.80 & 88.60 & 99.00 & 69.80 & 0.00 & 98.15 & 92.70 & 73.56 & 65.96 & 81.73 & 75.85 \\
\textbf{RDKV}   & \textbf{100.00} & \textbf{100.00} & \textbf{100.00} & \textbf{99.80} & \textbf{99.80} & \textbf{99.60} & 99.70 & \textbf{99.50} & \underline{99.76} & \textbf{93.54} & \textbf{95.53} & \textbf{98.84} \\
\midrule
\multicolumn{13}{c}{\textit{Sequence length} $= 16$k} \\
\midrule
FullKV         & 100.00 & 100.00 & 100.00 & 100.00 & 100.00 & 99.80 & 99.70 & 98.70 & 99.68 & 73.38 & 95.53 & 96.98 \\
\midrule
SnapKV          & \textbf{100.00} & 99.80 & 94.60 & 99.60 & \textbf{100.00} & 45.60 & 99.75 & \underline{98.90} & 99.32 & \underline{54.84} & \underline{94.33} & 89.70 \\
AdaKV           & \textbf{100.00} & \textbf{100.00} & \underline{96.00} & 99.40 & \textbf{100.00} & \underline{67.60} & \textbf{99.80} & 98.70 & \textbf{99.52} & 46.32 & 93.00 & \underline{90.94} \\
ThinK       & \textbf{100.00} & \textbf{100.00} & 80.20 & \textbf{100.00} & 99.00 & 38.20 & \textbf{99.80} & \textbf{99.00} & 99.20 & 54.28 & 90.93 & 87.33 \\
Snap+Zip    & 47.60 & 98.00 & 82.80 & 99.20 & 55.80 & 0.00 & 95.05 & 92.35 & 73.32 & 36.42 & 76.47 & 68.82 \\
\textbf{RDKV}   & \textbf{100.00} & \textbf{100.00} & \textbf{100.00} & \textbf{100.00} & 99.80 & \textbf{99.00} & 99.60 & 98.65 & \textbf{99.52} & \textbf{77.26} & \textbf{95.07} & \textbf{97.17} \\
\midrule
\multicolumn{13}{c}{\textit{Sequence length} $= 32$k} \\
\midrule
FullKV         & 100.00 & 100.00 & 100.00 & 99.40 & 99.60 & 99.20 & 99.60 & 99.20 & 99.04 & 10.18 & 87.40 & 90.33 \\
\midrule
SnapKV          & \textbf{100.00} & 98.80 & 94.60 & 98.80 & 99.20 & 32.60 & 99.40 & \underline{98.50} & 96.60 & 21.12 & 72.93 & 82.96 \\
AdaKV           & \textbf{100.00} & 98.60 & \underline{95.00} & 99.20 & \textbf{99.60} & \underline{49.80} & \textbf{99.55} & 98.40 & \underline{97.68} & 20.28 & 68.33 & \underline{84.22} \\
ThinK       & \textbf{100.00} & \textbf{100.00} & 76.80 & \textbf{99.40} & 95.20 & 21.20 & 99.35 & 97.95 & 97.12 & 18.16 & 68.33 & 79.41 \\
Snap+Zip    & 94.60 & \textbf{100.00} & 90.80 & 99.00 & 46.40 & 1.80 & \underline{99.50} & 95.95 & 83.28 & \underline{21.26} & \underline{76.13} & 73.52 \\
\textbf{RDKV}   & \textbf{100.00} & \textbf{100.00} & \textbf{99.80} & \textbf{99.40} & \underline{99.40} & \textbf{90.60} & 99.30 & \textbf{98.55} & \textbf{98.52} & \textbf{23.98} & \textbf{79.67} & \textbf{89.93} \\
\midrule
\multicolumn{13}{c}{\textit{Sequence length} $= 64$k} \\
\midrule
FullKV         & 100.00 & 100.00 & 100.00 & 99.40 & 98.40 & 97.00 & 99.05 & 94.50 & 96.64 & 1.22 & 87.13 & 88.49 \\
\midrule
SnapKV          & \textbf{100.00} & 94.80 & 89.80 & 98.60 & \textbf{96.20} & 13.00 & 98.15 & \underline{94.15} & 89.80 & 2.64 & \underline{80.00} & 77.92 \\
AdaKV           & \textbf{100.00} & 95.40 & 92.60 & 98.60 & 94.40 & \underline{21.80} & \underline{98.50} & 93.85 & \underline{91.84} & 2.28 & 76.73 & \underline{78.73} \\
ThinK       & \textbf{100.00} & 97.00 & 70.40 & \underline{99.00} & 83.40 & 5.20 & 98.45 & 94.00 & 90.80 & 2.14 & 67.67 & 73.46 \\
Snap+Zip    & 99.60 & \textbf{100.00} & \underline{94.60} & \underline{99.00} & 42.60 & 0.80 & 98.30 & 92.10 & 84.00 & \textbf{4.66} & \textbf{85.33} & 72.82 \\
\textbf{RDKV}   & \textbf{100.00} & \underline{99.20} & \textbf{99.80} & \textbf{99.40} & \underline{95.40} & \textbf{56.40} & \textbf{99.45} & \textbf{94.45} & \textbf{94.44} & \underline{3.58} & 73.13 & \textbf{83.20} \\
\midrule
\multicolumn{13}{c}{\textit{Sequence length} $= 128$k} \\
\midrule
FullKV         & 100.00 & 99.40 & 99.80 & 97.60 & 88.60 & 66.20 & 98.45 & 90.65 & 66.56 & 0.00 & 71.80 & 79.91 \\
\midrule
SnapKV          & \textbf{100.00} & \textbf{99.20} & 50.80 & \textbf{97.20} & \textbf{82.20} & 3.20 & \underline{97.10} & \textbf{90.25} & \underline{66.60} & \underline{0.12} & \underline{52.20} & \underline{67.17} \\
AdaKV           & \textbf{100.00} & \textbf{99.20} & 47.80 & \textbf{97.20} & 74.40 & \underline{5.40} & \textbf{97.15} & \underline{89.15} & 63.16 & 0.10 & 35.00 & 64.42 \\
ThinK       & \textbf{100.00} & \textbf{99.20} & 61.80 & 97.00 & \underline{81.60} & 2.00 & 96.10 & 88.70 & 62.96 & \textbf{0.18} & 49.00 & 67.14 \\
Snap+Zip    & \textbf{100.00} & 96.20 & \underline{85.40} & 95.00 & 49.00 & 1.80 & 92.05 & 81.30 & \textbf{67.32} & 0.06 & \textbf{63.87} & 66.55 \\
\textbf{RDKV}   & \textbf{100.00} & 99.00 & \textbf{93.80} & 96.80 & \underline{81.60} & \textbf{9.00} & 96.45 & 88.50 & 63.12 & 0.06 & 44.73 & \textbf{70.28} \\
\bottomrule
\end{tabular}
\end{table}

\paragraph{Comparison with HqeKV.}
HqeKV~\citep{anonymous2026hqekv} is a concurrent method that also
combines eviction and quantization, but fixes the tier ratios rather
than deriving them from a single allocation curve.
\Cref{tab:rdkv-vs-hqekv} compares RDKV and HqeKV on the 11 RULER
tasks at $B_\text{total} = 1024L$ across three context lengths
($16$k, $32$k, $64$k) on LLaMA-3.1-8B-Instruct.
RDKV leads at every context length, with the gap widening from
$+1.8$ points at $32$k to $+5.3$ at $16$k and $+5.0$ at $64$k.
The advantage is concentrated on tasks where the allocation
granularity matters most:
on CWE, RDKV leads by $+49.9$ / $+22.0$ / $+6.5$ points across the
three lengths; on MV-NIAH, by $+3.4$ / $+7.4$ / $+12.1$; on MQ-NIAH,
by $+2.7$ / $+5.3$ / $+12.2$.
HqeKV outperforms RDKV on FWE at $32$k and $64$k and on MK-NIAH-3 at
$32$k, where its fixed high-precision tier retains more tokens at
full bit-width; RDKV's water-filling allocator trades off these
tokens for finer-grained precision elsewhere.

\begin{table}[!ht]
\centering
\caption{RDKV vs.\ HqeKV~\citep{anonymous2026hqekv} on the 11 RULER tasks
across context lengths $\{16\text{k}, 32\text{k}, 64\text{k}\}$ for
LLaMA-3.1-8B-Instruct at $B_\text{total} = 1024L$.
The best result in each row is in \textbf{bold}.}
\label{tab:rdkv-vs-hqekv}
\setlength{\tabcolsep}{2.2pt}
\fontsize{6.6}{7.8}\selectfont
\begin{tabular}{lccccccccccccc}
\toprule
Method & S-NIAH-1 & S-NIAH-2 & S-NIAH-3 & MK-NIAH-1 & MK-NIAH-2 & MK-NIAH-3 & MQ-NIAH & MV-NIAH & VT & CWE & FWE & Avg. \\
\midrule
\multicolumn{13}{c}{\textit{Sequence length} $= 16$k} \\
\midrule
HqeKV       & \textbf{100.00} &  99.40 &  99.00 & \textbf{99.60} & \textbf{99.80} & \textbf{86.40} &  96.70 &  95.10 &  96.20 &  26.80 & \textbf{95.00} &  90.40 \\
\textbf{RDKV} & \textbf{100.00} & \textbf{99.80} & \textbf{100.00} & \textbf{99.60} &  99.60 &  86.20 & \textbf{99.40} & \textbf{98.50} & \textbf{99.20} & \textbf{76.70} &  93.20 & \textbf{95.70} \\
\midrule
\multicolumn{13}{c}{\textit{Sequence length} $= 32$k} \\
\midrule
HqeKV       &  99.80 & \textbf{99.20} &  88.80 &  96.60 &  98.60 & \textbf{85.20} &  93.80 &  91.00 &  95.30 &  15.30 & \textbf{87.20} &  86.40 \\
\textbf{RDKV} & \textbf{100.00} &  98.40 & \textbf{99.60} & \textbf{98.40} & \textbf{98.80} &  71.20 & \textbf{99.20} & \textbf{98.30} & \textbf{98.70} & \textbf{37.30} &  71.20 & \textbf{88.30} \\
\midrule
\multicolumn{13}{c}{\textit{Sequence length} $= 64$k} \\
\midrule
HqeKV       &  99.60 &  96.00 &  72.40 &  97.20 & \textbf{94.80} &  28.00 &  86.40 &  81.00 &  86.20 &   0.10 & \textbf{84.10} &  75.10 \\
\textbf{RDKV} & \textbf{100.00} & \textbf{98.40} & \textbf{99.00} & \textbf{99.00} &  93.00 & \textbf{30.60} & \textbf{98.60} & \textbf{93.00} & \textbf{92.90} & \textbf{6.60} &  69.70 & \textbf{80.10} \\
\bottomrule
\end{tabular}
\end{table}


\section{Additional Experiments on Needle-in-a-Haystack}
\label{app:niah}

\Cref{fig:niah-appendix} extends the Needle-in-a-Haystack evaluation
of \cref{sec:exp-niah} to a larger cache budget
$B_\text{total} = 128L$.  Even with the relaxed budget, SnapKV and
AdaKV still exhibit failure bands at intermediate depths ($11$--$89\%$)
for longer contexts: a needle that does not rank among the top-$k$
tokens is evicted regardless of the budget headroom.  RDKV preserves a
near-uniform success pattern comparable to FullKV; the residual misses
concentrate at depths and lengths where FullKV itself begins to lose
recall.

\begin{figure}[!tb]
  \centering
  \begin{subfigure}[t]{0.49\linewidth}
    \includegraphics[width=\linewidth]{figs/niah_fullkv.pdf}
    \caption{FullKV (avg.\ $1.00$).}
  \end{subfigure}
  \begin{subfigure}[t]{0.49\linewidth}
    \includegraphics[width=\linewidth]{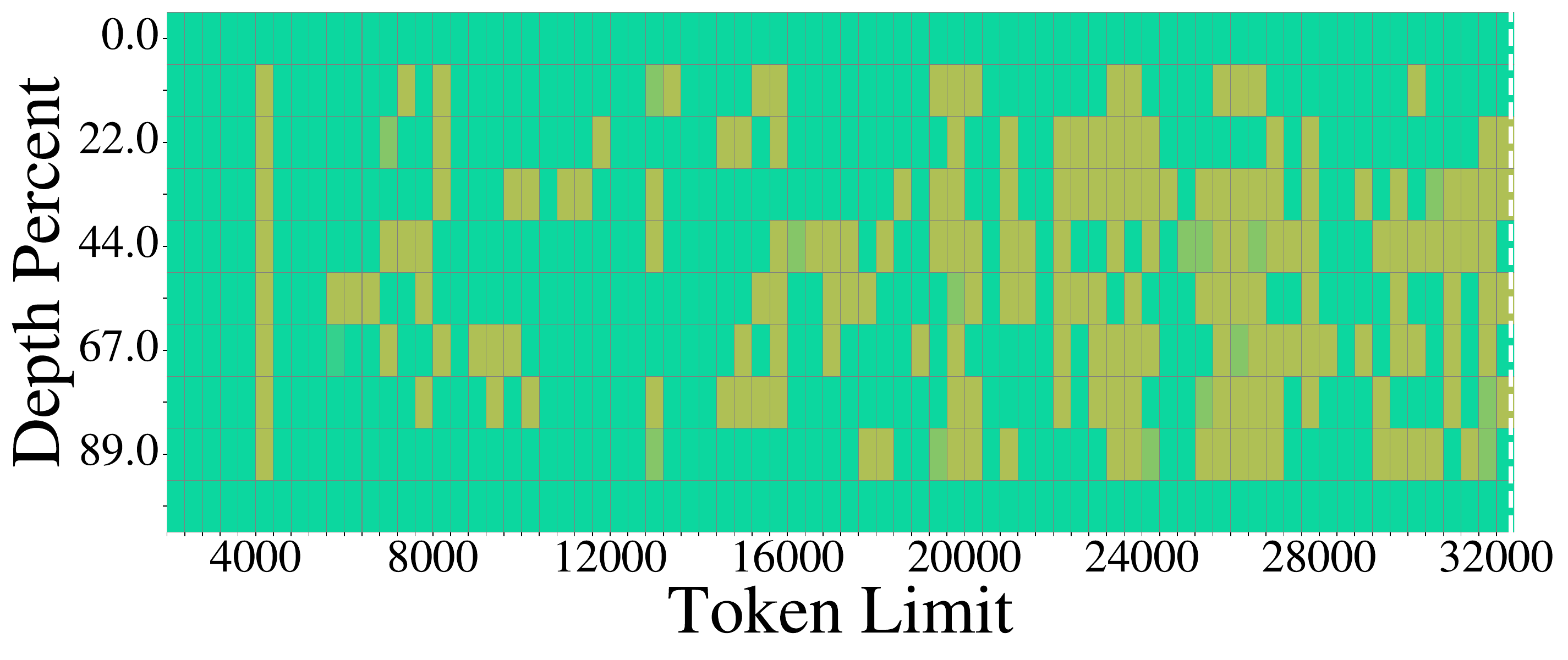}
    \caption{SnapKV (avg.\ $0.90$).}
  \end{subfigure}\\[2pt]
  \begin{subfigure}[t]{0.49\linewidth}
    \includegraphics[width=\linewidth]{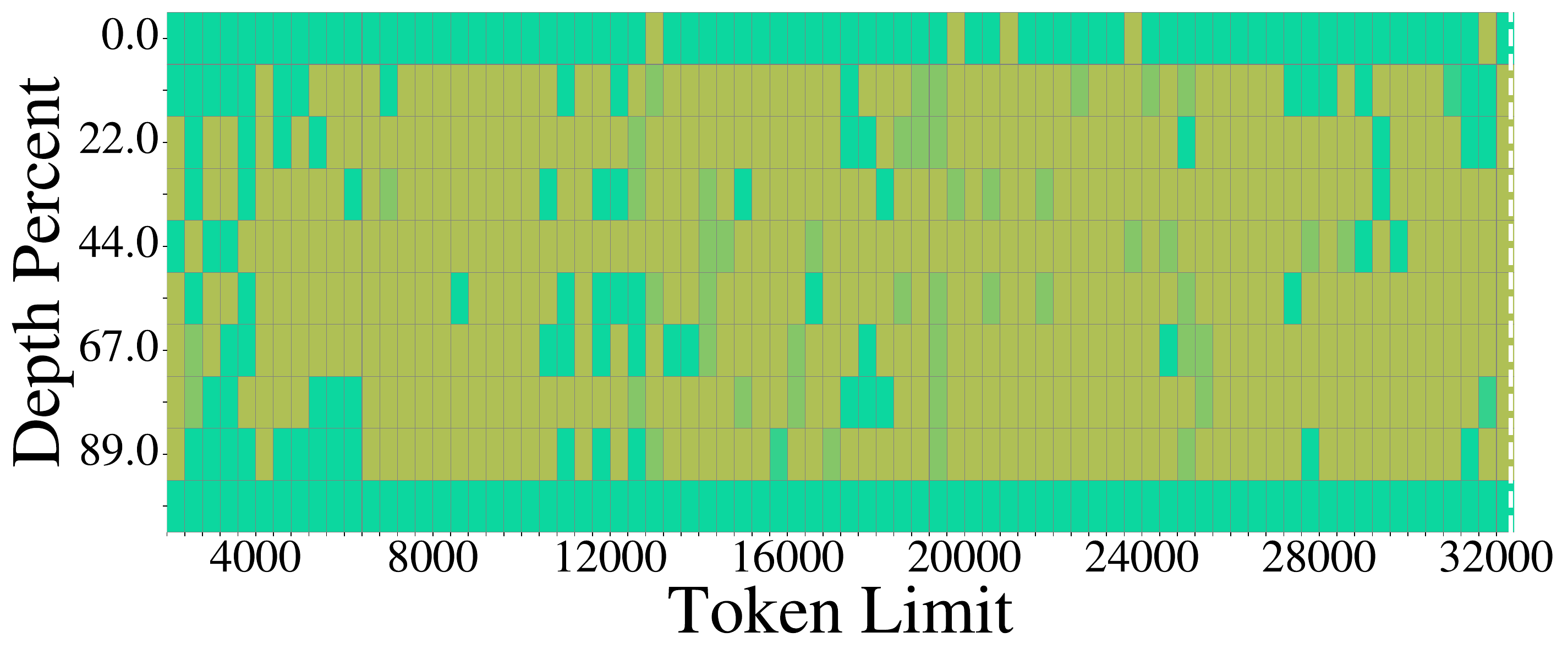}
    \caption{AdaKV (avg.\ $0.75$).}
  \end{subfigure}
  \begin{subfigure}[t]{0.49\linewidth}
    \includegraphics[width=\linewidth]{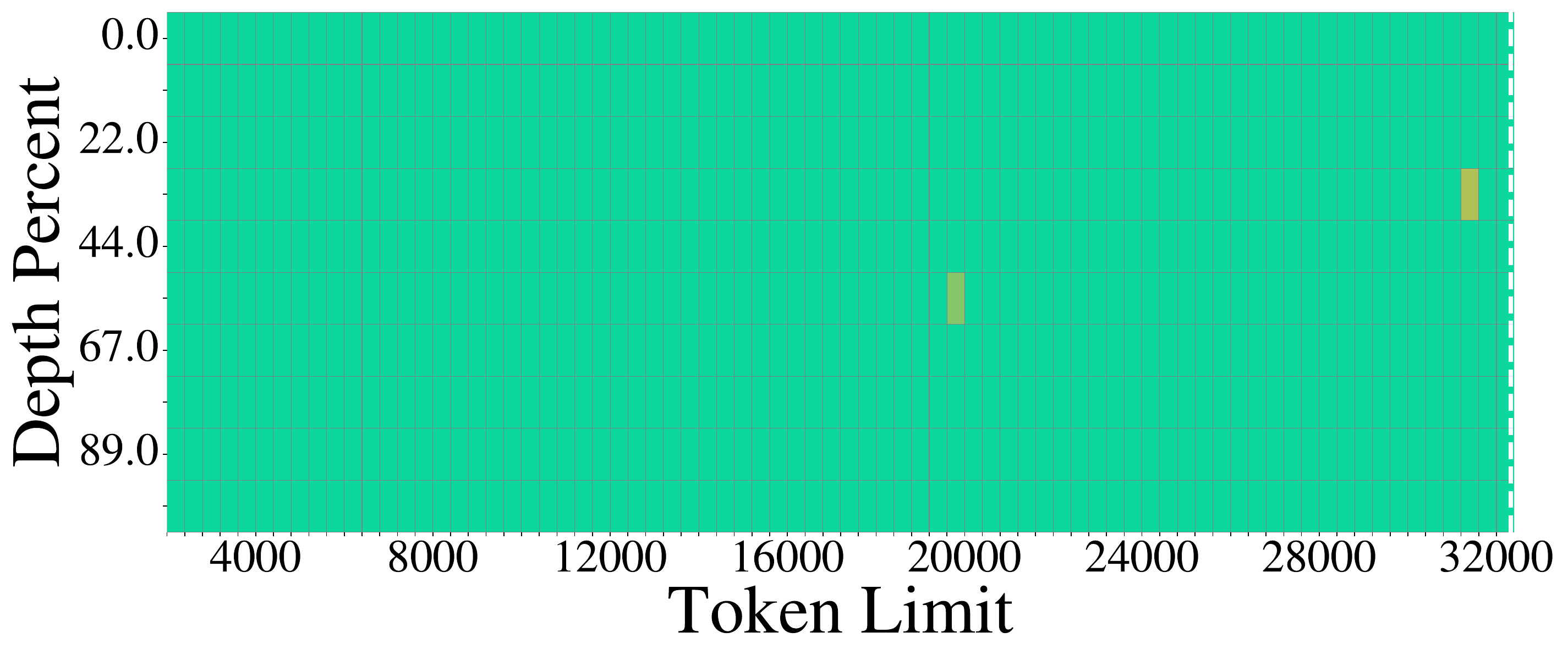}
    \caption{RDKV (avg.\ $0.99$).}
  \end{subfigure}
  \caption{Needle-in-a-Haystack on LLaMA-3.1-8B-Instruct at
  $B_\text{total} = 128L$ and context length up to $32$k.  Green denotes
  successful retrieval; warmer colours denote partial or complete
  failure.}
  \label{fig:niah-appendix}
\end{figure}


\section{Additional Experiments on InfiniteBench}
\label{app:infbench-detailed}

In this section, we present a detailed evaluation of RDKV on the
InfiniteBench benchmark~\citep{zhang2024infbench}, which extends
long-context evaluation to sequences exceeding $100$k tokens. The 10
tasks span three categories: (i)~retrieval---Passkey Retrieval
(Retr.Pass), Number Retrieval (Retr.Num), and KV Retrieval (Retr.KV)---
which test the ability to locate and extract specific information
embedded in long contexts; (ii)~language understanding---English
Dialogue (En.Dia), English Summarization (En.Sum), English
Multiple-Choice (En.MC), English QA (En.QA), and Chinese QA
(Zh.QA)---which evaluate comprehension and reasoning over novel-length
inputs; and (iii)~structured reasoning---Math Find (Math.Find)
and Code Debugging (Debug)---which require precise numerical or logical
retrieval. These tasks collectively pose challenges for both context
retention and fine-grained information extraction at extreme sequence
lengths.

We evaluate RDKV on LLaMA-3.1-8B-Instruct with
$B_\text{total} = 1024L$. The evaluation compares RDKV with SnapKV,
AdaKV, ThinK, SnapKV+ZipCache, and FullKV (oracle).

As reported in \cref{tab:infbench-detailed}, RDKV achieves the highest
average ($39.46$) among all compression methods, leading AdaKV ($38.06$)
by $1.40$ points. A breakdown by task category reveals where the
advantage concentrates. On the retrieval tasks, RDKV scores $7.20$ on
KV Retrieval vs.\ $\leq 2.20$ for all baselines---a $3.3\times$
improvement---because the needle-style retrieval pattern aligns directly
with the rate-distortion allocation: high-attention tokens are preserved
at high precision while low-attention tokens are aggressively quantized
or removed. On language understanding tasks, RDKV leads on En.Dia
($13.00$ vs.\ $\leq 11.50$), En.Sum ($25.13$ vs.\ $\leq 23.66$), and
En.QA ($14.49$ vs.\ $\leq 13.43$), matching FullKV accuracy on
En.MC ($68.56$). On structured reasoning tasks (Math.Find, Debug), all
compression methods cluster near each other, indicating that cache
compression is not the bottleneck once the relevant information is
preserved.

\begin{table}[!tb]
\centering
\caption{Performance on the 10 InfiniteBench tasks for LLaMA-3.1-8B-Instruct at
$B_\text{total} = 1024L$. The best result among compression methods is in \textbf{bold}; the
second-best is \underline{underlined}. FullKV is reported as a reference upper bound.}
\label{tab:infbench-detailed}
\fontsize{9.5}{11.4}\selectfont
\setlength{\tabcolsep}{1pt}
\begin{tabular}{lcccccccccc c}
\toprule
Method & Retr.Pass & Retr.Num & Retr.KV & En.Dia & En.Sum & En.MC & En.QA & Zh.QA & Math.Find & Debug & Avg. \\
\midrule
FullKV    & 100.0 & 99.32 & 56.20 & 18.00 & 27.52 & 68.56 & 14.63 & 13.28 & 34.00 & 22.34 & 45.38 \\
\midrule
SnapKV    & \textbf{100.00} & \underline{96.61} & 1.40 & 8.50 & 23.35 & \underline{68.12} & 12.58 & \underline{12.44} & \textbf{34.00} & 22.08 & 37.91 \\
AdaKV     & \textbf{100.00} & 94.41 & 1.80 & 10.50 & \underline{23.66} & \underline{68.12} & 12.94 & 12.30 & \textbf{34.00} & \textbf{22.84} & \underline{38.06} \\
ThinK      & \textbf{100.00} & 87.97 & 1.80 & 7.00 & 23.03 & \underline{68.12} & 11.21 & 12.35 & \underline{33.71} & 22.08 & 36.73 \\
Snap+Zip   & \textbf{100.00} & 91.86 & \underline{2.20} & \underline{11.50} & 22.16 & 67.69 & \underline{13.43} & 12.40 & \textbf{34.00} & 22.34 & 37.76 \\
\textbf{RDKV} & \textbf{100.00} & \textbf{96.95} & \textbf{7.20} & \textbf{13.00} & \textbf{25.13} & \textbf{68.56} & \textbf{14.49} & \textbf{12.65} & \textbf{34.00} & \underline{22.59} & \textbf{39.46} \\
\bottomrule
\end{tabular}
\end{table}

\section{Additional Experiments on Ablation Study}
\label{app:ablation-extension}

\subsection{Prefill Overhead}
\label{app:ablation-overhead}

\cref{tab:prefill-overhead} breaks down the time-to-first-token (TTFT)
for RDKV on LLaMA-3.1-8B-Instruct with a $128$K-token prefill on a
single A100 64\,GB. The baseline is a standard FullKV prefill with
FlashAttention-2. RDKV bypasses the HuggingFace \texttt{generate}
wrapper and calls transformer layers directly, saving $593$\,ms of
framework overhead. The four RDKV-specific stages---weight computation
($w_t$, $w_c$), MCKP bisection, and TriZone packing---add a combined
$2{,}331$\,ms. The net overhead is $1{,}740$\,ms, or $6.0\%$ of the
FullKV prefill, yielding a measured TTFT of $30{,}583$\,ms.
This one-time cost is amortised over the entire decode sequence; with
$128$K context and the $4.5\times$ decode speedup reported in
\cref{sec:experiments}, RDKV breaks even within the first ${\sim}\,50$
generated tokens.

\begin{table}[!tb]
\centering
\caption{Prefill overhead breakdown for RDKV on LLaMA-3.1-8B-Instruct
at $128$K context length (A100 64\,GB). Percentages are relative to
the FullKV prefill baseline.}
\label{tab:prefill-overhead}
\footnotesize
\begin{tabular}{lrr}
\toprule
Component & Time (ms) & \% of FullKV Prefill \\
\midrule
FullKV prefill (FA2) & 28\,843 & --- (baseline) \\
\midrule
Forward path saving & $-593$ & $-2.1\%$ \\
$w_t$ computation & $+308$ & $+1.1\%$ \\
$w_c$ computation & $+550$ & $+1.9\%$ \\
MCKP bisection & $+609$ & $+2.1\%$ \\
TriZone packing & $+863$ & $+3.0\%$ \\
\midrule
Net RDKV overhead & $+1{,}740$ & $+6.0\%$ \\
\midrule
RDKV TTFT & 30\,583 & $106.0\%$ \\
\bottomrule
\end{tabular}
\end{table}

\subsection{K/V Budget Split Ratio}
\label{app:ablation-kv-split}

The main experiments fix an equal K/V split
($B^V = B^K = \tfrac{1}{2}B_\text{head}$) as a design convention
(\cref{app:impl-details}).
To validate this choice, we sweep the K budget ratio
$r_K \in \{0.4,\,0.5,\,0.6\}$
(so $B^K = r_K B_\text{head}$ and $B^V = (1-r_K) B_\text{head}$)
on LLaMA-3.1-8B-Instruct at three cache budgets
$B_\text{total} \in \{128L,\,256L,\,512L\}$.
All other settings (observation window $S_w = 32$, pooling kernel $w = 5$,
$\mathcal{B} = \{0,2,4,8,16\}$) are kept identical.

\begin{table}[!tb]
\centering
\caption{K/V budget split ablation: per-task LongBench scores for LLaMA-3.1-8B-Instruct
  at $B_\text{total} \in \{128L, 256L, 512L\}$.
  $r_K$ denotes the fraction of $B_\text{head}$ allocated to K.
  The best result in each row is in \textbf{bold}; the second-best is \underline{underlined}.}
\label{tab:kv-split}
\setlength{\tabcolsep}{0.8pt}
\fontsize{6.4}{7.6}\selectfont
\begin{tabular}{lccccccccccccccccc}
\toprule
\multirow{2}{*}{$r_K$}
 & \multicolumn{3}{c}{Single-Doc QA}
 & \multicolumn{3}{c}{Multi-Doc QA}
 & \multicolumn{3}{c}{Summarization}
 & \multicolumn{3}{c}{Few-shot Learning}
 & \multicolumn{2}{c}{Synthetic}
 & \multicolumn{2}{c}{Code}
 & \multirow{2}{*}{Avg.} \\
\cmidrule(lr){2-4}\cmidrule(lr){5-7}\cmidrule(lr){8-10}\cmidrule(lr){11-13}\cmidrule(lr){14-15}\cmidrule(lr){16-17}
 & \rotatebox{35}{NrtvQA} & \rotatebox{35}{Qasper} & \rotatebox{35}{MF-en}
 & \rotatebox{35}{HotpotQA} & \rotatebox{35}{2WikiMQA} & \rotatebox{35}{Musique}
 & \rotatebox{35}{GovRep} & \rotatebox{35}{QMSum} & \rotatebox{35}{MultiNews}
 & \rotatebox{35}{TREC} & \rotatebox{35}{TriviaQA} & \rotatebox{35}{SAMSum}
 & \rotatebox{35}{PCount} & \rotatebox{35}{PRe}
 & \rotatebox{35}{Lcc} & \rotatebox{35}{RB-P}
 & \\
\midrule
\multicolumn{18}{c}{\textit{LLaMA-3.1-8B-Instruct,} $B_\text{total} = 128L$} \\
\midrule
$r_K{=}0.4$ & \textbf{29.65} & \textbf{41.93} & \textbf{56.21} & \underline{56.51} & 48.05 & \textbf{30.93} & \textbf{25.58} & \textbf{24.98} & \textbf{24.37} & \underline{63.50} & \textbf{92.14} & 41.11 & 8.58 & 99.50 & 61.09 & 53.55 & \underline{47.36} \\
$r_K{=}0.5$ & \underline{29.45} & \underline{41.34} & \underline{55.55} & \textbf{57.39} & \textbf{49.38} & \underline{30.89} & \underline{25.55} & \underline{24.59} & \underline{24.29} & \textbf{65.00} & 91.92 & \textbf{41.80} & \textbf{8.75} & \textbf{100.00} & \textbf{62.68} & \textbf{55.41} & \textbf{47.75} \\
$r_K{=}0.6$ & 29.22 & 38.60 & 55.40 & 56.47 & \underline{48.69} & 30.74 & 24.69 & 24.12 & 23.44 & \underline{63.50} & \underline{91.94} & \underline{41.61} & \textbf{8.75} & \textbf{100.00} & \underline{62.41} & \underline{54.92} & 47.15 \\
\midrule
\multicolumn{18}{c}{\textit{LLaMA-3.1-8B-Instruct,} $B_\text{total} = 256L$} \\
\midrule
$r_K{=}0.4$ & \underline{29.22} & \underline{43.75} & \underline{56.13} & 56.89 & 49.16 & 30.85 & \textbf{28.26} & \underline{24.74} & \textbf{25.71} & 67.00 & \underline{91.61} & 41.19 & 8.52 & \underline{99.50} & \underline{63.11} & 53.89 & 48.10 \\
$r_K{=}0.5$ & \textbf{29.63} & \textbf{44.85} & \textbf{56.33} & \textbf{57.85} & \textbf{49.64} & \textbf{31.73} & \underline{28.01} & 24.66 & \underline{25.66} & \textbf{70.00} & \textbf{91.87} & \underline{41.79} & \textbf{8.83} & \underline{99.50} & \textbf{63.27} & \underline{56.01} & \textbf{48.73} \\
$r_K{=}0.6$ & 29.15 & 43.62 & 55.49 & \underline{57.05} & \underline{49.21} & \underline{31.28} & 27.09 & \textbf{24.89} & 25.28 & \underline{68.50} & 91.20 & \textbf{42.77} & \underline{8.68} & \textbf{100.00} & 62.54 & \textbf{56.51} & \underline{48.33} \\
\midrule
\multicolumn{18}{c}{\textit{LLaMA-3.1-8B-Instruct,} $B_\text{total} = 512L$} \\
\midrule
$r_K{=}0.4$ & 29.85 & \underline{45.35} & \textbf{56.77} & 57.04 & \textbf{48.99} & 31.11 & \textbf{30.72} & \textbf{25.24} & \underline{26.37} & 70.50 & \textbf{91.91} & 42.26 & 8.75 & \textbf{100.00} & \textbf{63.41} & 55.73 & 49.00 \\
$r_K{=}0.5$ & \textbf{30.24} & \textbf{45.67} & 55.58 & \underline{57.06} & \underline{48.95} & \underline{31.67} & \underline{30.68} & \underline{25.06} & \textbf{26.46} & \textbf{72.00} & \underline{91.83} & \textbf{43.15} & \underline{8.83} & \textbf{100.00} & \underline{63.36} & \underline{57.02} & \textbf{49.22} \\
$r_K{=}0.6$ & \underline{29.92} & 44.96 & \underline{55.61} & \textbf{57.31} & 48.78 & \textbf{31.97} & 29.78 & 24.93 & 26.29 & \underline{71.00} & 91.58 & \underline{43.02} & \textbf{8.93} & 99.50 & 63.38 & \textbf{57.17} & \underline{49.01} \\
\bottomrule
\end{tabular}
\end{table}

\cref{tab:kv-split} reports the 16-task LongBench average.
At every budget level the equal split $r_K = 0.5$ achieves the highest
score. The margins are small ($\leq 0.63$ points), indicating that
the allocation is robust to moderate shifts in the K/V ratio.
Shifting more budget toward K ($r_K = 0.6$) consistently hurts:
the V cache carries per-token information that directly scales the
attention output (\cref{thm:v-score}), so starving it degrades
quality faster than under-provisioning per-channel K precision.
Conversely, a mild reduction of K budget ($r_K = 0.4$) is largely
absorbed because the per-channel distortion weight $w_c$
(\cref{thm:k-score}) concentrates on a small subset of outlier
channels; most channels tolerate lower bit-widths with negligible
error.
The symmetric optimum $r_K = 0.5$ balances these two effects,
justifying the equal split adopted throughout the paper.

\section{Visualization of Bit Allocation}
\label{app:bit-vis}

To illustrate how the RDKV allocator distributes bit-widths across
tokens and channels, we visualize the per-unit distortion weight
$w_t$ (V cache, \cref{thm:v-score}) and $w_c$ (K cache,
\cref{thm:k-score}) together with the resulting bit-width assignment
$b \in \{0, 2, 4, 8, 16\}$ on a representative LongBench sample
(LLaMA-3.1-8B-Instruct, $B_\text{total} = 128L$).
Each dot is coloured by its assigned bit-width:
{\color{blue}$\bullet$}~16-bit,
{\color{orange}$\bullet$}~8-bit,
{\color{red}$\bullet$}~4-bit,
$\bullet$~2-bit;
tokens assigned 0~bits (evicted) are omitted.
We show layers 15 (middle) and 31 (final) with two KV heads each.

\paragraph{Token-level V allocation
(\cref{fig:tok-l15h0}--\cref{fig:tok-l31h1}).}
The score distribution spans roughly six orders of magnitude.
Sink tokens (position~0) and recent tokens near the sequence tail
consistently receive the highest scores and correspondingly 16-bit
or 8-bit retention---the former carry accumulated attention mass, and
the latter fall within the observation window.
The bulk of mid-sequence tokens reside at 2--4~bits.
Head~0 and head~1 share the global profile but differ in which
mid-sequence positions spike, reflecting head-specific attention
patterns.

\paragraph{Channel-level K allocation
(\cref{fig:ch-l15h0}--\cref{fig:ch-l31h1}).}
The 128 channels of each head display a heavy-tailed score
distribution.
A small number of outlier channels---typically 2--4 per head---receive
16-bit, with scores 2--3 orders of magnitude above the median.
These correspond to the persistent outlier channels documented in
prior work on KV cache quantization~\citep{kivi,xu2024think}: the
product $\|Q_{:,c}\|_2 \,\|K_{:,c}\|_2$ is dominated by a handful of
dimensions where both query and key norms are large.
The allocator automatically assigns these channels full precision while
compressing the remaining channels to 4--8~bits, avoiding the
uniform-precision penalty that affects fixed-bit-width quantization
methods.


\begin{figure}[!ht]
\centering
\includegraphics[width=\textwidth]{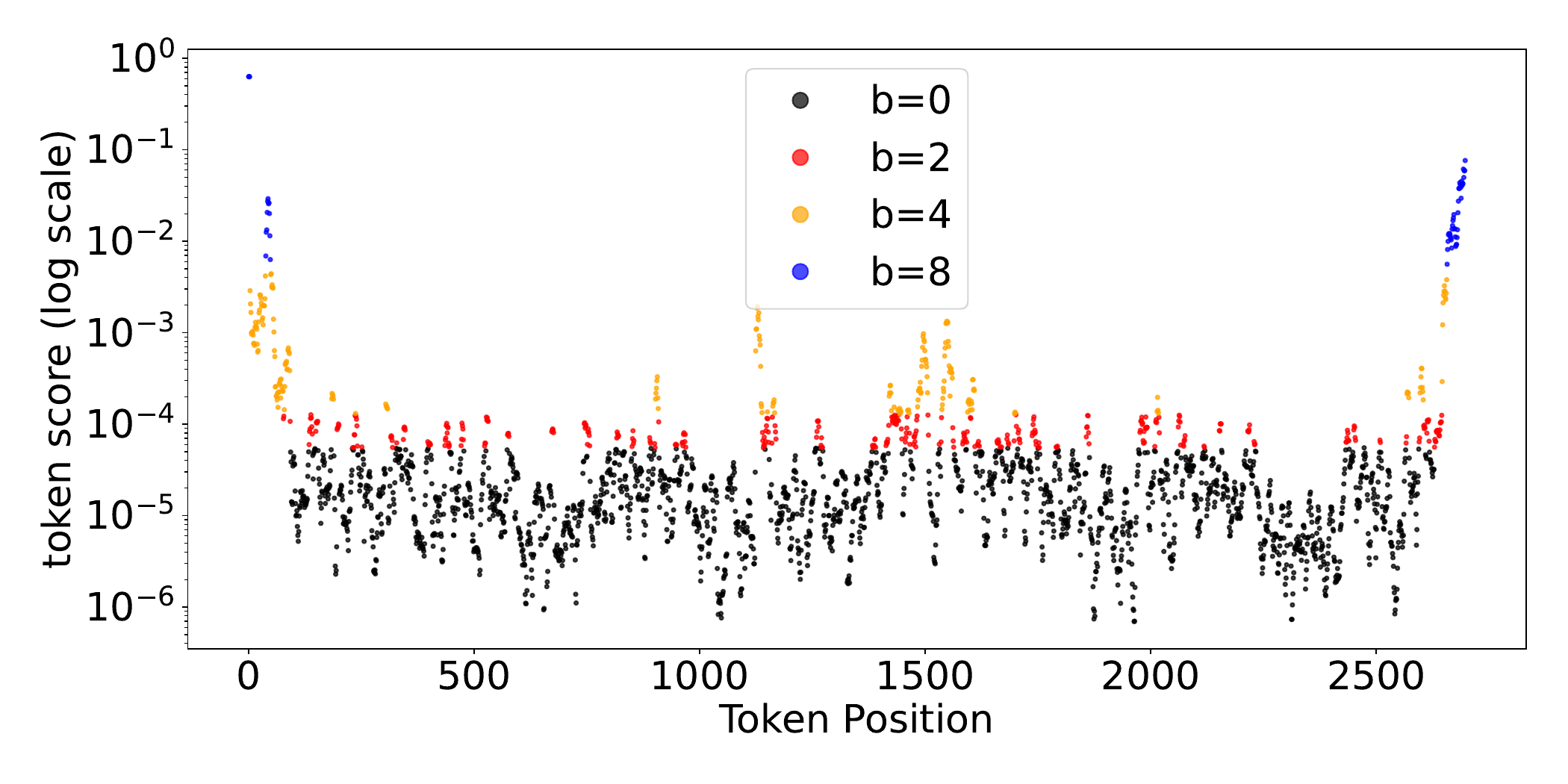}
\caption{Per-token V-cache bit allocation: layer 15, head 0.}
\label{fig:tok-l15h0}
\end{figure}
\vspace{-2mm}

\begin{figure}[!ht]
\centering
\includegraphics[width=\textwidth]{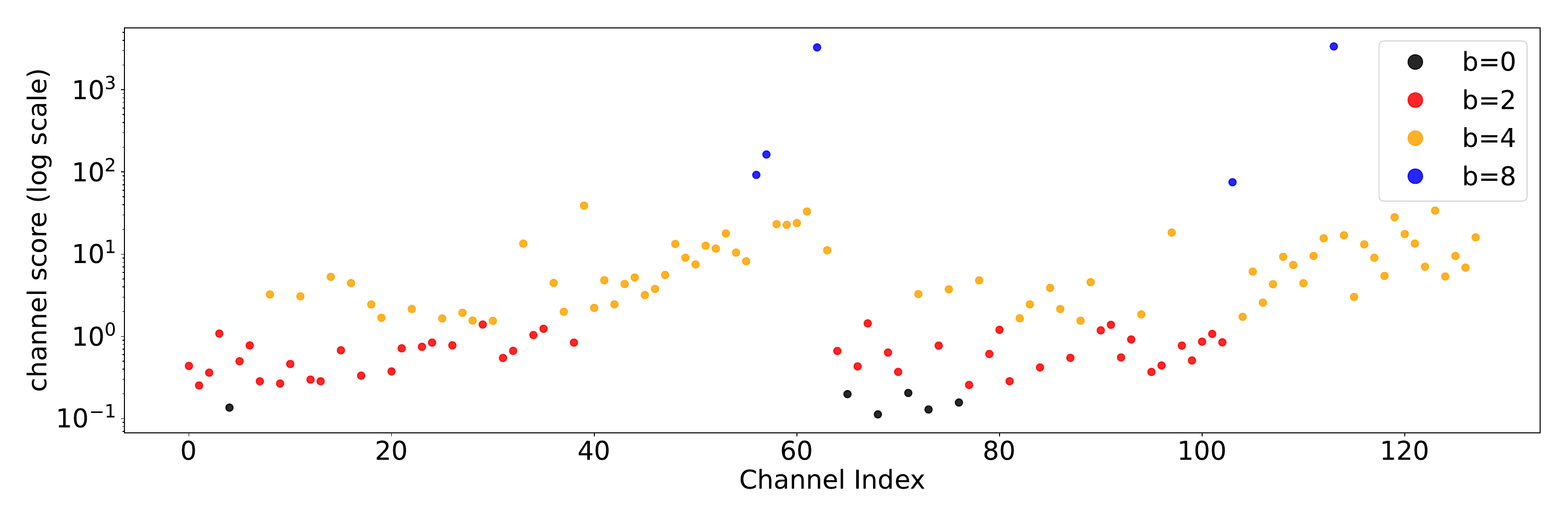}
\caption{Per-channel K-cache bit allocation: layer 15, head 0.}
\label{fig:ch-l15h0}
\end{figure}


\begin{figure}[!ht]
\centering
\includegraphics[width=\textwidth]{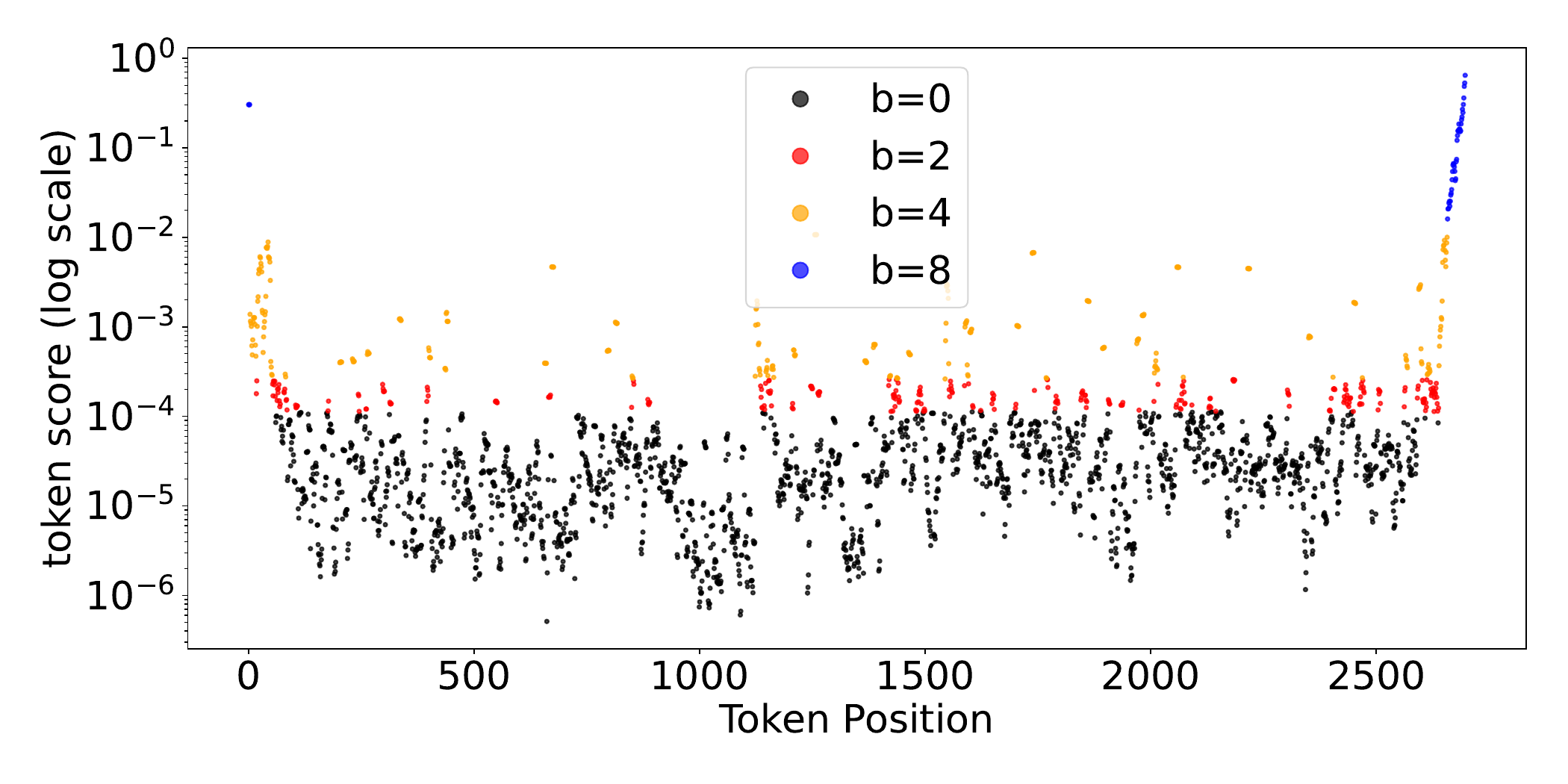}
\caption{Per-token V-cache bit allocation: layer 15, head 1.}
\label{fig:tok-l15h1}
\end{figure}

\begin{figure}[!ht]
\centering
\includegraphics[width=\textwidth]{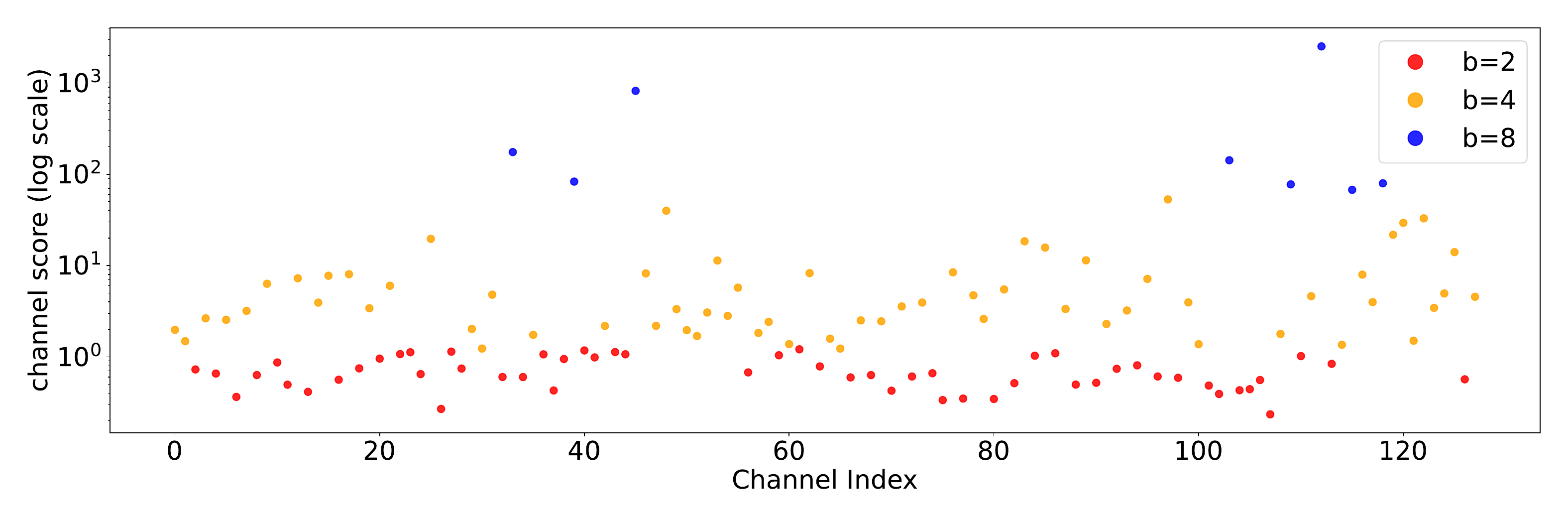}
\caption{Per-channel K-cache bit allocation: layer 15, head 1.}
\label{fig:ch-l15h1}
\end{figure}


\begin{figure}[!ht]
\centering
\includegraphics[width=\textwidth]{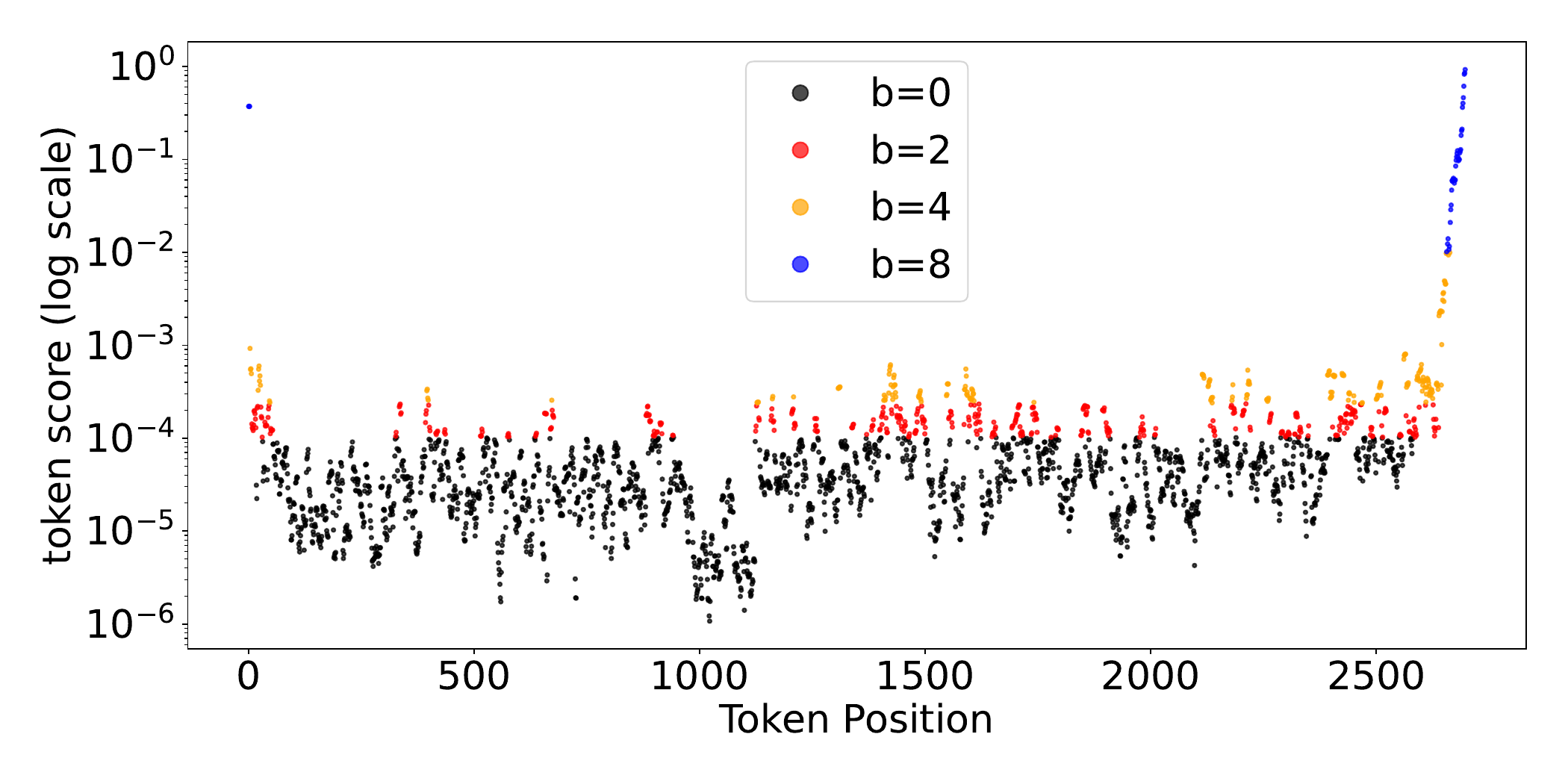}
\caption{Per-token V-cache bit allocation: layer 31, head 0.}
\label{fig:tok-l31h0}
\end{figure}

\begin{figure}[!ht]
\centering
\includegraphics[width=\textwidth]{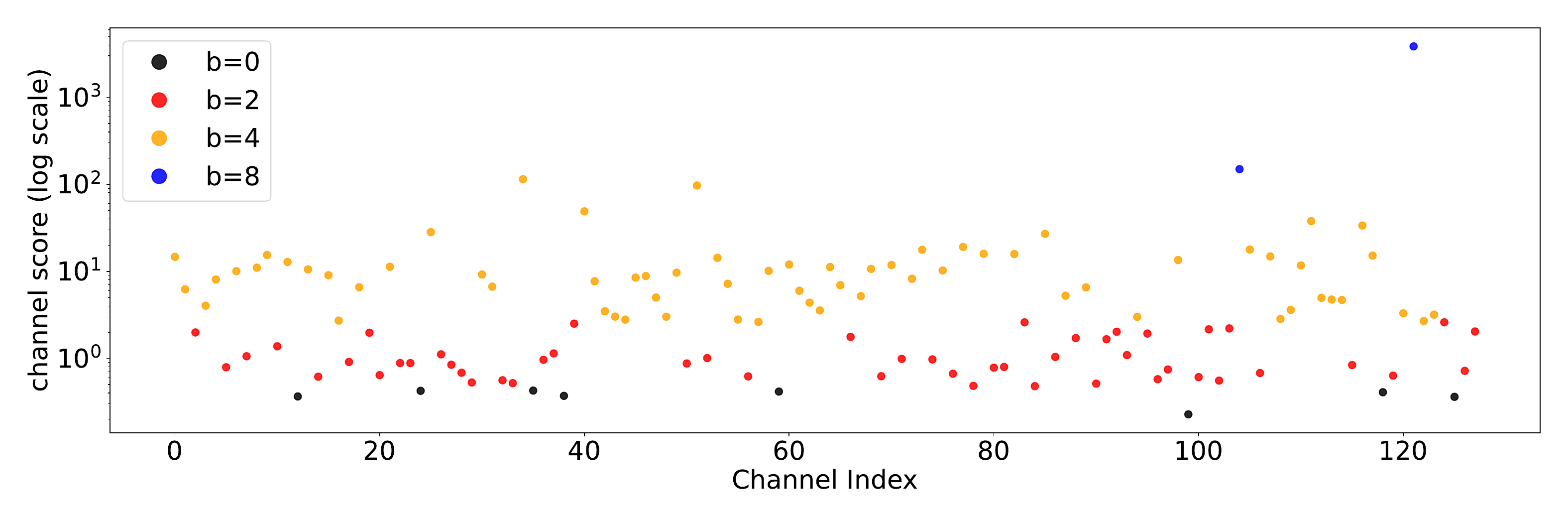}
\caption{Per-channel K-cache bit allocation: layer 31, head 0.}
\label{fig:ch-l31h0}
\end{figure}


\begin{figure}[!ht]
\centering
\includegraphics[width=\textwidth]{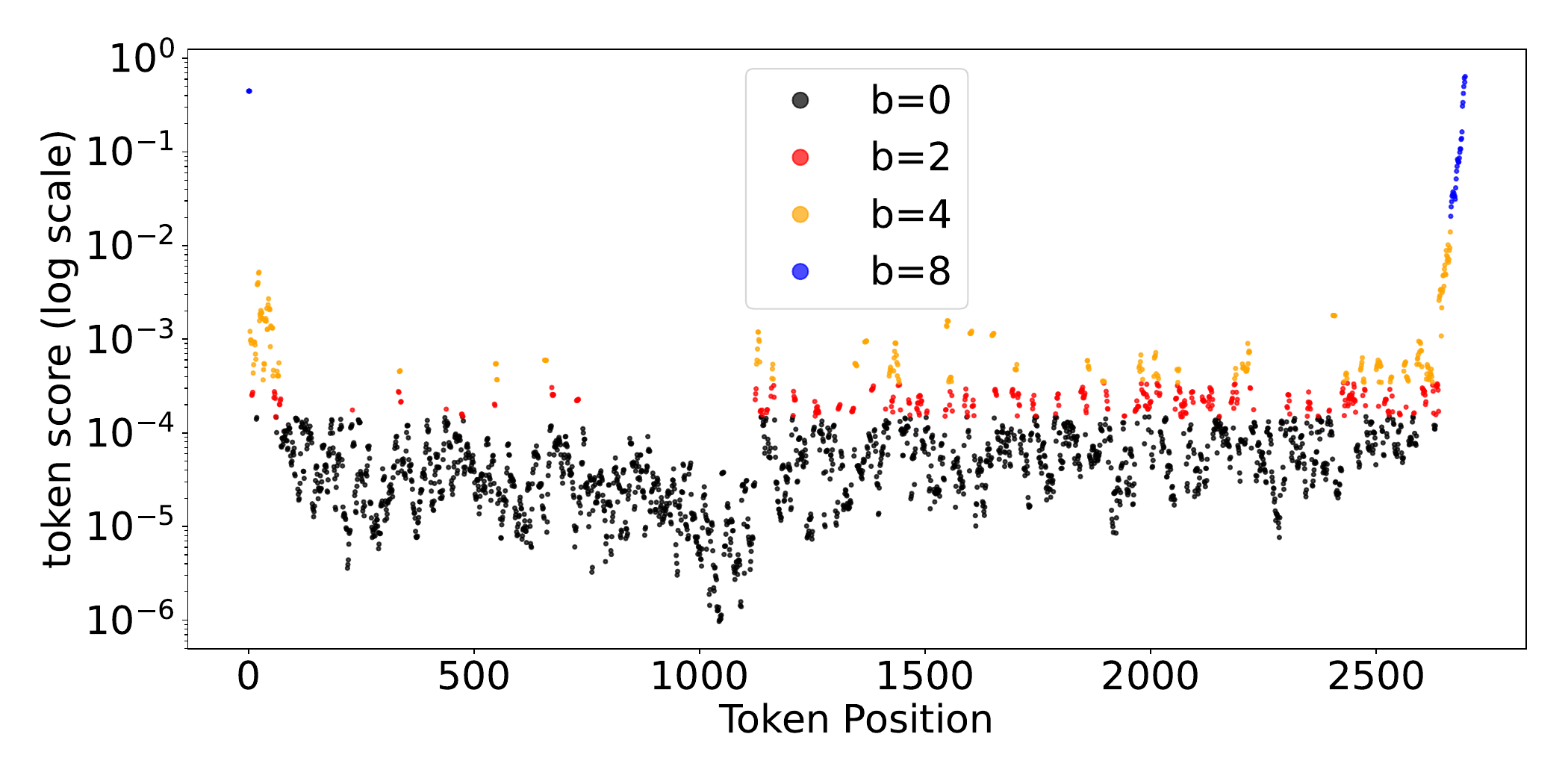}
\caption{Per-token V-cache bit allocation: layer 31, head 1.}
\label{fig:tok-l31h1}
\end{figure}

\begin{figure}[!ht]
\centering
\includegraphics[width=\textwidth]{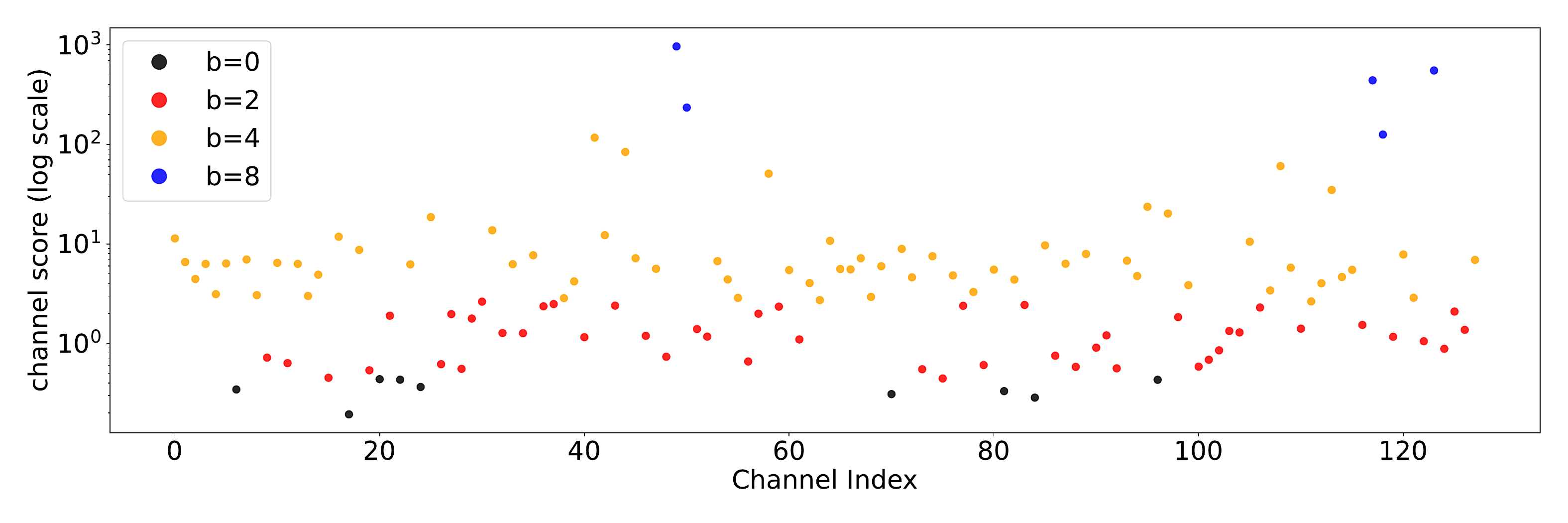}
\caption{Per-channel K-cache bit allocation: layer 31, head 1.}
\label{fig:ch-l31h1}
\end{figure}

\clearpage
\section{Calibrated Distortion Tables}
\label{app:epsilon-calibration}

\Cref{tab:epsilon-calibration} reports the calibrated per-coordinate
distortion $\varepsilon^K(b)$ (K~cache, per-channel NMSE) and
$\varepsilon^V(b)$ (V~cache, per-token NMSE) for each evaluated model
and bit-width $b\in\{2,4,8\}$.
Each entry is averaged across all layers, heads, and 32 calibration
sequences (4\,k tokens, LongBench prefill prefixes;
see \cref{app:impl-details} for calibration details).
These layer-averaged values are the distortion tables consumed by
the MCKP solver (\cref{prop:weak-duality}).
All models share $\varepsilon(0)=1$ (eviction) and
$\varepsilon(16)=0$ (full precision).

V~cache distortion is notably consistent across architectures at every
bit-width, whereas K~cache distortion shows greater cross-model
variance---chiefly because Qwen2.5-72B's embedding-adjacent layers
exhibit elevated per-channel 8-bit quantization error.

\begin{table}[h]
\centering
\caption{Calibrated normalized distortion $\varepsilon(b)$ per model.
  K~cache: per-channel NMSE; V~cache: per-token NMSE.}
\label{tab:epsilon-calibration}
\small
\begin{tabular}{l ccc ccc}
\toprule
& \multicolumn{3}{c}{$\varepsilon^K(b)$ (per-channel)}
& \multicolumn{3}{c}{$\varepsilon^V(b)$ (per-token)} \\
\cmidrule(lr){2-4} \cmidrule(lr){5-7}
Model & $b{=}2$ & $b{=}4$ & $b{=}8$
      & $b{=}2$ & $b{=}4$ & $b{=}8$ \\
\midrule
LLaMA-3.1-8B  & 0.149 & 0.0062 & $2.2\!\times\!10^{-5}$ & 0.313 & 0.0140 & $4.9\!\times\!10^{-5}$ \\
LLaMA-2-13B   & 0.288 & 0.0124 & $5.5\!\times\!10^{-5}$ & 0.272 & 0.0122 & $4.4\!\times\!10^{-5}$ \\
Mistral-7B    & 0.280 & 0.0116 & $6.9\!\times\!10^{-5}$ & 0.296 & 0.0130 & $4.5\!\times\!10^{-5}$ \\
Qwen2.5-72B   & 0.296 & 0.0164 & $4.4\!\times\!10^{-3}$ & 0.281 & 0.0126 & $4.4\!\times\!10^{-5}$ \\
Qwen3-4B      & 0.347 & 0.0147 & $1.5\!\times\!10^{-4}$ & 0.313 & 0.0139 & $4.8\!\times\!10^{-5}$ \\
\bottomrule
\end{tabular}
\end{table}

\section{Limitations and Future Work}
\label{app:limitations}

\textbf{Limitations.}
Like SnapKV, AdaKV, and PyramidKV, RDKV compresses the KV cache once
after prefill and does not re-evaluate during decoding.  This is a
common design choice in the prefill-dominated regime (long prompt, short
generation) targeted by these methods.  The allocation is therefore
frozen with respect to attention-pattern shifts that may occur during
generation.

\textbf{Future work.}
The rate-distortion framework can be extended to the decode phase.
A natural approach accumulates attention weights over a small buffer of
recent decode tokens and periodically applies the same MCKP allocation
to compress them, keeping decode-phase memory sub-linear in the number
of generated tokens relative to the uncompressed baseline.

\section{Impact Statement}
\label{app:impact}

RDKV reduces the memory footprint and decoding latency of long-context
LLM inference without modifying model weights or training procedures.
On the positive side, lower hardware requirements enable longer context
windows on smaller GPUs, reducing both the financial cost and the energy
consumption of serving long-context applications such as document
summarization, multi-document QA, and retrieval-augmented generation.
This may broaden access to long-context LLM capabilities for
resource-constrained practitioners and organizations.
More broadly, by reducing the number of GPU hours required per query,
RDKV lowers the energy consumption and carbon footprint of LLM
serving, contributing to more environmentally sustainable deployment
of long-context inference.
On the negative side, as with any inference acceleration technique,
reduced serving cost could lower the barrier to deploying LLMs at
scale, potentially amplifying existing concerns associated with
large-scale LLM deployment such as misinformation generation.
However, RDKV does not introduce new model capabilities---it preserves
the output distribution of the original model within the compression
budget---and therefore does not create risks beyond those already
present in the underlying LLM.

\end{document}